\documentclass[10pt]{article}
\usepackage[utf8]{inputenc}
\usepackage{geometry}
\usepackage{amsmath,amssymb,amsthm,mathrsfs,nicefrac}
\usepackage{mathtools}
\usepackage{latexsym,amscd,amsbsy,dsfont,amsfonts}
\usepackage{graphicx}
\usepackage{authblk,textcomp}
\usepackage[dvipsnames]{xcolor}
\usepackage{cancel}

\usepackage{bm}
\usepackage[cal=euler]{mathalfa}
\usepackage{libertine}
\usepackage[libertine,smallerops]{newtxmath}
\usepackage[T1]{fontenc}
\usepackage[font=small]{caption}
\usepackage{tikz}
\usetikzlibrary{arrows}
\usetikzlibrary{decorations}
\usepackage{subcaption}
\usepackage{booktabs}
\usepackage{doi}
\usepackage{float}
\usepackage{enumitem}

\usepackage{algorithm}
\usepackage{algpseudocode}

\usepackage{hyperref}
\hypersetup{pdfauthor={IdePHICS},pdftitle={OnlineLearning},%
            colorlinks, linktocpage=true, pdfstartpage=1, pdfstartview=FitV,%
    breaklinks=true, pdfpagemode=UseNone, pageanchor=true, pdfpagemode=UseOutlines,%
    plainpages=false, bookmarksnumbered, bookmarksopen=true, bookmarksopenlevel=1,%
    hypertexnames=true, pdfhighlight=/O,%
    urlcolor=orange, linkcolor=RoyalBlue, citecolor=orange}

\theoremstyle{plain}
\newtheorem{theorem}{Theorem}[section]

\newtheorem{corollary}[theorem]{Corollary}
\theoremstyle{definition}
\newtheorem{definition}[theorem]{Definition}

\theoremstyle{remark}

\def\Tr{{\rm Tr}}

\def\w{\underline{w}}
\newcommand{\E}{\mathbb{E}}

\newcommand{\pr}{\mathbb{P}}
\newcommand{\bR}{\mathbb{R}}

\newcommand{\cN}{\mathcal{N}}

\newtheorem{lemma}{Lemma}
\def\dde{{\rm d}}
\DeclareMathOperator{\Prox}{Prox}

\DeclareMathOperator{\R}{\mathbb{R}}

\DeclareMathOperator{\Cov}{\text{Cov}}

\DeclarePairedDelimiter\norm{\lVert}{\rVert}
\DeclareMathOperator*{\argmin}{argmin} 
\DeclareMathOperator*{\argmax}{argmax} 

\DeclareMathOperator{\erfc}{\text{erfc}}

\newcommand{\bS}{{\boldsymbol{S}}}
\newcommand{\bbR}{{\boldsymbol{R}}}
\newcommand{\bhS}{{\hat{\boldsymbol{S}}}}
\newcommand{\bU}{{\boldsymbol{U}}}

\newcommand{\bhQ}{{\hat{\boldsymbol{Q}}}}
\newcommand{\bhV}{{\hat{\boldsymbol{V}}}}
\newcommand{\bQ}{{\boldsymbol{Q}}}
\newcommand{\bV}{{\boldsymbol{V}}}
\newcommand{\bh}{{\boldsymbol{h}}}

\newcommand{\be}{{\boldsymbol{e}}}
\newcommand{\bu}{{\boldsymbol{u}}}
\newcommand{\bv}{{\boldsymbol{v}}}
\newcommand{\bZ}{{\boldsymbol{Z}}}
\newcommand{\bX}{{\boldsymbol{X}}}
\newcommand{\bY}{{\boldsymbol{Y}}}
\newcommand{\bT}{{\boldsymbol{T}}}
\newcommand{\bphi}{{\boldsymbol{\phi}}}
\newcommand{\bx}{{\boldsymbol{x}}}
\newcommand{\by}{{\boldsymbol{y}}}
\newcommand{\bA}{{\boldsymbol{A}}}
\newcommand{\bI}{{\boldsymbol{I}}}
\def\inp{d}

\def\nsamp{n}

\def\C{\bm{C}}
\def\i{\nu}

\def\w{\bm{w}}

\def\X{\bm{X}}
\def\y{\bm{y}}
\def\Y{\bm{Y}}
\def\V{\bm{V}}

\def\Vi{{\cal V}}

\def\Vib{\bar{{\cal V}}}

\def\nc{k}

\def\dd{\text{d}}
\def\prox{\text{prox}}

\def\gamm{\bm{\gamma}}
\def\Lamb{\bm{\Lambda}}

\def\Zout{{\cal Z}_{\text{out}}}
\def\Z0{{\cal Z}_0}

\def\fout{\bm{f}_{\text{out}}}
\def\fouti{\bm{f}_{\text{out},\i}}

\def\om{\bm{\omega}}

\usepackage[toc,page]{appendix} 
\appto\appendix{\counterwithin{equation}{section}}

\title{Learning curves for \\ the multi-class teacher-student perceptron}
\author[1,$\ast$]{Elisabetta Cornacchia}
\author[2,$\ast$]{Francesca Mignacco}
\author[3,4,$\ast$]{Rodrigo Veiga}
\author[5]{C\'edric Gerbelot}
\author[3]{Bruno Loureiro}
\author[6]{Lenka Zdeborov\'a}

\affil[1]{\small Ecole Polytechnique F\'{e}d\'{e}rale de Lausanne (EPFL). Mathematical Data Science (MDS) lab. \newline CH-1015 Lausanne, Switzerland.}
\affil[2]{\small Universit\'{e} Paris-Saclay, CNRS, CEA, Institut de physique th\'{e}orique, 91191, Gif-sur-Yvette, France.}
\affil[3]{\small Ecole Polytechnique F\'{e}d\'{e}rale de Lausanne (EPFL). 
Information, Learning and Physics (IdePHICS) lab. \newline CH-1015 Lausanne, Switzerland.}
\affil[4]{\small Universidade de S\~{a}o Paulo. Instituto de F\'{i}sica. S\~{a}o Paulo, Brazil.}
\affil[5]{\small Laboratoire de Physique de l’Ecole Normale Sup\'{e}rieure, Universit\'{e} PSL, CNRS, Paris,
France.}
\affil[6]{\small Ecole Polytechnique F\'{e}d\'{e}rale de Lausanne (EPFL).
Statistical Physics of Computation (SPOC) lab. \newline CH-1015 Lausanne, Switzerland.}
\affil[$\ast$]{Equal contribution.}

\date{}

\begin{document}

\maketitle

\begin{abstract}
   One of the most classical results in high-dimensional learning theory provides a closed-form expression for the generalisation error of binary classification with the single-layer teacher-student perceptron on i.i.d. Gaussian inputs.
   Both Bayes-optimal estimation and empirical risk minimisation (ERM) were extensively analysed for this setting. At the same time, a considerable part of modern machine learning practice concerns multi-class classification. Yet, an analogous analysis for the corresponding multi-class teacher-student perceptron was missing.
   In this manuscript we fill this gap by deriving and evaluating asymptotic expressions for both the Bayes-optimal and ERM generalisation errors in the high-dimensional regime.
   For Gaussian teacher weights, we investigate the performance of ERM with both cross-entropy and square losses, and explore the role of ridge regularisation in approaching Bayes-optimality.
   In particular, we observe that regularised cross-entropy minimisation yields close-to-optimal accuracy. 
   Instead, for a binary teacher we show that a first-order phase transition arises in the Bayes-optimal performance.
\end{abstract}

\tableofcontents

\section{Introduction}
\label{sec:intro}
Starting with the seminal work of Gardner and Derrida \cite{gardner1989three} the teacher-student perceptron is a broadly adopted and studied model for high-dimensional supervised binary (i.e., two classes) classification. In this model the input data are Gaussian independent identically distributed (i.i.d.) and a single-layer teacher neural network with randomly chosen i.i.d. weights from some distribution generates the labels. A student neural network then uses the input data and labels to \emph{learn} the teacher function. The corresponding generalisation error as a function of the number of samples per dimension $\alpha=n/d$ was first derived using the replica method from statistical physics in the limit $n,d \to \infty$ for a range of teacher weights distributions (Gaussian and Rademacher being the most commonly considered) and for a range of estimators, e.g., Bayes-optimal or empirical risk minimisation (ERM) with common losses, see reviews \cite{seung1992statistical,watkin1993statistical,engel2001statistical} and references therein. Notably, the phase transition in the optimal generalisation error of the teacher-student perceptron with Rademacher teacher weights \cite{gyorgyi1990first,sompolinsky1990learning} is possibly one of the earliest examples of the so-called \emph{statistical-to-computational trade-offs} that are currently broadly studied in high-dimensional statistics and inference. More recently, these works on the teacher-student perceptron have been put on rigorous ground in \cite{barbier2019optimal} for the Bayes-optimal estimation, and in \cite{aubin2020generalization} for ERM with convex losses. 

Modern machine learning classification tasks most often involve more than two classes, e.g., 10 for classification on MNIST or CIFAR10, or even 1000 in ImageNet. Multi-class classification is hence more commonly considered in practice. To the best of our knowledge, the analysis of the high-dimensional teacher-student perceptron has not been generalised to the multi-class setting yet. Closely related settings, such as high-dimensional multi-class classification with Gaussian mixture data \cite{loureiro2021learning,wang2021benign,9414099,thrampoulidis2020theoretical} were recently reported, while -- as far as we know -- the teacher-student setting is still missing. In this paper we fill this gap. We define the multi-class teacher-student perceptron model and provide the following {\bf main contributions}:

\begin{description}
\item[C1]  We derive, prove and evaluate an asymptotic closed-form expression for the generalisation error of the Bayes-optimal estimator in high-dimensions. In the case of Rademacher teacher weights we unveil a first-order phase transition in the learning curve, in analogy with the two-classes case. 
\item[C2] Similarly, we derive, prove and evaluate an asymptotic closed-form expression for the generalisation performance of ridge-regularised ERM with convex losses. In particular, we discuss and compare two widely used loss functions: the square and cross-entropy losses.
\item[C3] We compare optimally regularised cross-entropy classification to the Bayesian classifier, and conclude that for three classes the two are extremely close, in analogy with what was observed for two classes \cite{aubin2020generalization}. 
\end{description}

The main technical difficulty of analysing the teacher-student perceptron with $k>2$ classes is that the corresponding closed-form formulas are given in terms of a set of coupled self-consistent equations on $(k-1) \times (k-1)$ dimensional \emph{matrix variables} (a.k.a. \emph{order parameters}), involving $(k-1)$-dimensional integrals. This poses some challenges in both the mathematical proof and the numerical evaluation of their solution. In this work, we overcome these difficulties by leveraging on recent works with similar matrix structure, notably the committee machine \cite{aubin2019committee,barbier2021overlap} and the supervised $k$-cluster Gaussian mixture classification \cite{loureiro2021learning}. The heuristic replica method allows to derive a generic set of equations covering both the Bayes-optimal case and the ERM cases. The rigorous proof for the Bayes-optimal case is given in \cite{aubin2019committee,barbier2021overlap} based on an interpolation argument. Here we provide the proof for the asymptotic performance of the estimator obtained with convex ERM, employing a similar proof strategy as in \cite{loureiro2021learning}, which leverages on the rigorous analysis of matrix-valued approximate message passing iterations~\cite{javanmard2013state}.

\paragraph{Reproducibility ---} The code used for producing the figures in this manuscript are publicly available on \href{https://github.com/rodsveiga/mc_perceptron}{GitHub}.
 
\paragraph{The data model ---}
We consider a multi-class classification problem where the training data $\bm{X}=\left(\bm{x}_1,\ldots,\bm{x}_n\right)^\top\in\mathbb{R}^{n\times d}$ are composed of $n$ $d-$dimensional i.i.d. standard Gaussian samples, where $x_{\mu i}\sim\mathcal{N}\left( x_{\mu i } \rvert 0,1 \right)$, $\forall i\in\{1,\ldots, d\}$, $\forall\mu\in\{1,\ldots ,n\}$. The corresponding labels are $\bm{Y}=(\bm{y}_1,\ldots,\bm{y}_n)^\top\in\{0,1\}^{n\times k}$, each representing the one-hot encoding of one of $k$ possible classes. In particular, we assume the labels are generated by a \emph{teacher} matrix $\bm{W}^*=(\bm{w}_1^*,\ldots\bm{w}_k^*)\in\mathbb{R}^{d\times k}$ as 
\begin{equation}
y_{\mu l}=\begin{cases} 1 & {\rm if\;}
l =\argmax_{h \in \{ 1,.., k \} } \left( {\bm{w}^*_h }^\top \bm{x}_{\mu }\right)\\
0 & {\rm otherwise}
\end{cases},\label{eq:model}
\end{equation}
$\forall\mu\in\{1,\ldots n\}$. Figure \ref{fig:ml_perc_k} displays schematically the system. 
\begin{figure}[ht!]
\centering
\centerline{\includegraphics[width=0.6\textwidth]{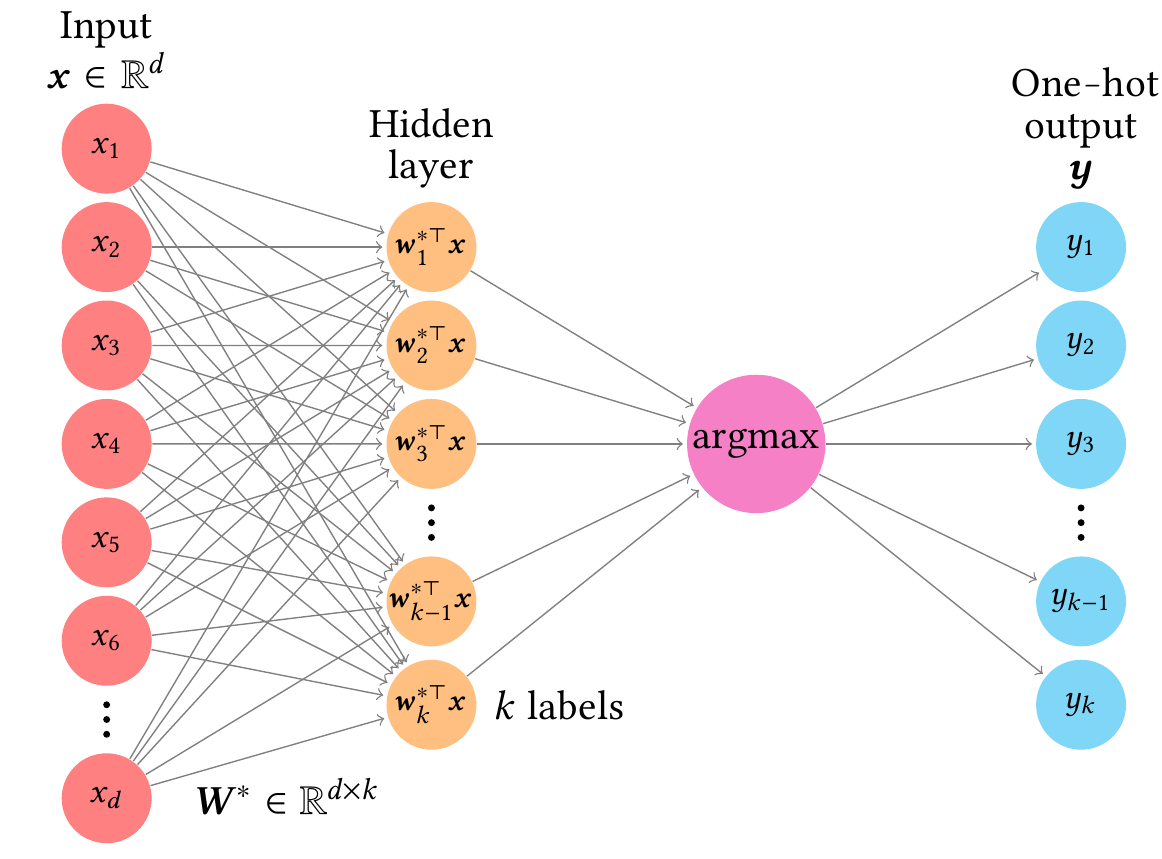}}
\caption{{\bf Multi-class perceptron:} Schematic representation of the multi-class classification problem defined in Eq.~\eqref{eq:model}.}
\label{fig:ml_perc_k}
\end{figure}
In the following, we will denote the output channel as $\phi_{\rm out} (\bm{v}) := \rm{e}_{\argmax_l(\{v_l\}_{l\in[k]})}\in\{0,1\}^{k}$, where $\rm{e}_h$ is the standard one-hot vector with $h$-th site equal to $1$ and all other entries equal to zero. The teacher matrix $\bm{W}^*$ is drawn with i.i.d. entries either from a standard Gaussian $w^*_{il} \sim \mathcal{N}( w^*_{il} \rvert 0,1 )$ or a Rademacher distribution $w^*_{il}=\pm 1$ with equal probability. Note that for $k = 2$ this problem corresponds to the well-studied perceptron problem with binary labels \cite{gardner1989three,engel2001statistical}. 

In what follows, we will be interested in the problem of \emph{learning} the teacher target function in the high-dimensional setting, where $n,d\rightarrow\infty$ at a fixed rate, or \emph{sample complexity}, $\alpha=n/d$, under two estimation procedures: empirical risk minimisation (ERM) and Bayes-optimal estimation.

\paragraph{Empirical risk minimisation (ERM) ---} In the first case, the statistician (or student) is given only the training data $(\bm{X}, \bm{Y})$, and her goal is to learn the teacher weights $\bm{W}^*$ with a multi-class perceptron model $\hat{\bm{y}}(\bm{x}) = \phi_{\rm{out}}(\bm{W}^{\top}\bm{x})$ by minimising a regularised empirical risk over the training set:
\begin{subequations}
 \label{eq:optimization_ERM}
\begin{equation} 
    \bm{\hat W}=\argmin_{\bm{W}\in\mathbb{R}^{d\times k}}\left[\mathcal{L}\left(\bm{W};\bm{X},\bm{Y}\right)+r_\lambda\left(\bm{W}\right)\right]\;,
\end{equation} 
\begin{equation}
    \mathcal{L}\left(\bm{W};\bm{X},\bm{Y}\right)=\sum_{\mu=1}^n\ell\left(\bm{W}^\top\bm{x}_\mu, \bm{y}_\mu\right) \;.
\end{equation}
\end{subequations}

The loss function $\ell$ accounts for the performance of the weight vector $\bm{W}$ over a single training point. Two widely used loss functions for multi-class classification are the cross-entropy loss $\ell(\bm{z},\bm{y})=-\sum_{l=1}^k y_l\cdot $ $\ln\left(e^{z_l}/\sum_{l=1}^k e^{z_l}\right)$ and the square loss $\ell(\bm{z},\bm{y})=(\bm{z}-\bm{y})^\top(\bm{z}-\bm{y})/2$. We will focus on the ridge regularisation function $r_\lambda(\bm{w})=\lambda\,\Vert\bm{W}\Vert_{F}^2/2$, where $\Vert \cdot \Vert_F$ indicates the Frobenius norm.

 \paragraph{Bayes-optimal estimator ---} In the second case, known as the \emph{Bayes-optimal} setting, the student has access not only to the training data but also to prior knowledge on the teacher weights distribution $P_{w}^*$ and on the model generating the inputs and the labels \eqref{eq:model}. In the teacher-student setting under consideration, where labels are generated by a noiseless channel, the Bayes-optimal estimator for the label $\bm{y}_{\rm new}$ of a previously unseen point $\bm{x}_{\rm new}$ can be computed directly from the Bayes-optimal estimator $\bm{\hat W}_{\rm BO}$ of the teacher weights as $\bm{\hat y}_{\rm new}=\phi_{\rm out}(\bm{\hat W}_{\rm BO}^\top\,\bm{x}_{\rm new})$. The matrix $\bm{\hat W}_{\rm BO}$ is the minimiser of the mean-squared error with respect to the ground-truth $\bm{W}^*$, i.e.,
\begin{equation}
\begin{split}
\label{eq:BO_estimator}
\bm{\hat W}_{\rm BO}&=\argmin_{\bm{W}}\E_{\bm{S},\bm{W}^*}\|\bm{W}-\bm{W}^* \|^2_F\\&=\E_{\bm{S},\bm{W}^*,\bm{W}|(\bm{S},\bm{W}^*)}\left[\bm{W}\right].
\end{split}
\end{equation}
Note that computing explicitly the Bayesian estimator requires computing the posterior distribution, which in general is unfeasible in high-dimensions. However, as we shall see, its performance can be characterised exactly in such limit. A key quantity in our derivation is the \emph{free entropy density}:
\begin{equation}
      \Phi=\lim_{d\rightarrow\infty}\frac 1d\E_{\bm{X},\bm{W}^*}\ln Z_d,\label{eq:free_entropy_def}
  \end{equation}
  where $Z_d$ is the \emph{partition function}, i.e., the normalisation factor of the posterior distribution over the weights
  \begin{equation}
  \label{eq:posterior}
      P\left(\bm{W}|\bm{X},\bm{Y}\right)=\frac{1}{Z_d}\prod_{l=1}^k P_w^*\left(\bm{w}_l\right) \prod_{\mu=1}^n \delta\left(\bm{y}_\mu -\phi_{\rm out}( \bm{W}^\top\bm{x}_\mu)\right).
  \end{equation}
  In the Bayes-optimal setting, the free entropy density is closely related to the mutual information density between the labels and the weights, see \cite{barbier2019optimal} for an explicit discussion of this connection.
  
\paragraph{Generalisation error ---} 
The performance of different optimisation strategies is measured through the \emph{generalisation error}, i.e. the expected error on a fresh sample. As it is commonly done for classification, in this work we will be interested in the misclassification rate (a.k.a. $0/1$ error):
\begin{equation}
    \varepsilon_{\rm gen}(\alpha)=\mathbb{E}_{\bm{x}_{\rm new},\bm{X},\bm{W}^*}\mathds{1}\left[\bm{\hat y}\left(\bm{\hat W}(\alpha)\right)\neq \bm{y}_{\rm new}\right],\label{eq:generalization_error_general}
\end{equation}
where $\bm{x}_{\rm new}$ is a previously unseen data point and $\bm{y}_{\rm new}$ is the corresponding label, generated by the teacher as in Eq.~\eqref{eq:model}. Similarly, the estimator $\bm{\hat y}$ is generated by the weight vector $\bm{\hat W}$, which in turn depends on the training set. We compare the performance obtained with ERM to the one achieved by the Bayes-optimal estimator from Eq.~\eqref{eq:BO_estimator}. Note that Eq.~\eqref{eq:generalization_error_general} for the Bayes-optimal error can be rewritten as
\begin{equation}
\begin{split}
    \varepsilon_{\rm gen}^{\text{Bayes}}&= \frac 12  \E_{\bm{S},\bm{x},\bm{W}^*}  \| \phi_{\rm out}({\bm{W}^*}^\top \bm{x}) - \phi_{\rm out}(\langle {\bm{W}}^\top \bm{x} \rangle) \|_2^2 \\
    &=1-\E_{\bm{S},\bm{x},\bm{W}^*} \left[ \phi_{\rm out}({\bm{W}^*}^\top \bm{x})^{\top}\phi_{\rm out}( {\bm{\hat W}_{\rm BO}}^\top \,\bm{x} )\right],
    \end{split}\label{eq:BO_error}
    \end{equation}
  where we have renamed for brevity $\bm{x}=\bm{x}_{\rm new}$, and $\langle \cdot\rangle = \E_{\bm{W}|(\bm{S},\bm{W}^*)}$, and we have used that $\|\phi_{\rm out}(\cdot)\|^2_2\equiv 1$. Since the distribution of $\bm{x}_{\rm new}$ is rotationally invariant, the averaged quantity $\E_{\bm{S},\bm{x}_{\rm new},\bm{W}^*} \Big[ \phi_{\rm out}({\bm{W}^*}^\top \bm{x}_{\rm new})^{\top} \cdot $ $ \phi_{\rm out}( {\bm{\hat W}_{\rm BO}}^\top \,\bm{x}_{\rm new} )\Big]$ only depends on the correlation between $\bm {W}^*$ and $\bm{\hat W}_{\rm BO}$, which as we will discuss later concentrates to the maximiser of the free entropy in Eq.~\eqref{eq:free_entropy_def} in the high-dimensional limit. 

\section{Main theoretical results}
From the definition of the generalisation error in \eqref{eq:generalization_error_general} and of the teacher model \eqref{eq:model}, it is easy to see that crucially the generalisation error only depends on the statistics of the $k$-dimensional quantities $({\bm{W}^{*}}^\top\bm{x}_{\rm new}, \hat{\bm{W}}^\top\bm{x}_{\rm new})\in\mathbb{R}^{k}\times \mathbb{R}^{k}$ (a.k.a. \emph{local fields}) -- both for Bayes-optimal estimation and ERM. Therefore, characterising the sufficient statistics of the local fields is equivalent to characterising the generalisation error. Our key theoretical result is that in the high-dimensional limit considered here the local fields are jointly Gaussian, and therefore the generalisation error only depends on the correlation $\bm{\bar m}_d$ between the teacher $\bm{ W}^*$ and the estimator $\bm{\hat W}$, and the covariances $\bm{Q}^*_d$ and $\bm{\bar q}_d$ of the teacher and the estimator respectively (a.k.a. the \emph{overlaps}):
\begin{subequations}
\label{eq:mqv}
 \begin{equation}
     \bm{\bar m}_d \equiv \frac{1}{d}\bm{\hat W}^\top \bm{W}^* \;,
 \end{equation}
  \begin{equation}
     \bm{\bar q}_d \equiv \frac{1}{d}\bm{\hat W}^\top \bm{\hat W}  \;,
 \end{equation}
  \begin{equation}
     \bm{Q}^*_d  \equiv \frac{1}{d} {\bm{W}^*}^\top \bm{W}^* \;.
 \end{equation}
\end{subequations}
Note that we keep the subscript $d$ to emphasize that these definitions are still in finite dimension and to distinguish them from the corresponding overlaps in the high-dimensional limit. As we will show next, these low-dimensional sufficient statistics can be computed explicitly by solving a set of coupled $(k-1)\times(k-1)$ self-consistent equations. 
\subsection{Performance of empirical risk minimization}
Our result holds under the following assumptions, in addition to the Gaussian hypothesis on the design matrix $\bm{X}$.
\vspace{1em}\\
\noindent 
\textbf{Assumptions:}
\begin{itemize}
\item[\textbf{(A1)}] the functions $\mathcal{L},r_{\lambda}$ are proper, closed, lower-semicontinuous, convex functions. The loss function $\mathcal{L}$ is differentiable and pseudo-Lipschitz of order 2 in both its arguments. We assume additionally that the regularisation $r_\lambda$ is strongly convex, differentiable and pseudo-Lipschitz of order 2;

\item[\textbf{(A2)}] the dimensions $n,d$ grow linearly according to the finite ratio $\alpha = n/d$;

\item[\textbf{(A3)}] the lines of the ground truth matrix $\mathbf{W^*} \in \mathbb{R}^{d \times k}$ are sampled i.i.d. from a sub-Gaussian probability distribution in $\mathbb{R}^{k}$.
\end{itemize}
\begin{theorem} \label{thm:ERM_main}
Let $\bm{\xi}\sim \cN_k(\bm{0},\bm{I}_k)$. Under (A1)-(A3), for any pair of pseudo-Lipschitz functions $\psi_{1} : \mathbb{R}^{d \times k} \to \mathbb{R}, \psi_{2} : \mathbb{R}^{n \times k} \to \mathbb{R}$ of order 2, the estimator $\hat{\bm{W}}$ and $\bm{\hat Z} = \frac{1}{\sqrt{d}} \mathbf{X}\bm{\hat W}$ satisfy the following:
\begin{subequations}
\begin{equation}
    \psi_{1}(\hat{\bm{W}}) \overset{P}{\longrightarrow} \mathbb{E}_{\bm \xi}\left[\mathcal{Z}^*_w(\bm{\hat m}\bm{\hat q}^{-1/2}\bm{\xi},\bm{\hat m}^T\bm{\hat q}^{-1}\bm{\hat m}) \psi_{1}\left(\bm{f}_w(\bm{\hat q}^{1/2}\bm{\xi},\bm{\hat V})\right)\right] \;,
\end{equation}
\begin{equation}
    \psi_2 (\bm{\hat Z}) \overset{P}{\longrightarrow} \mathbb{E}_{\bm{\xi}}\left[\int_{\mathbb{R}^{k}}\dd\bm{y}~\mathcal{Z}_{\rm out}^* (\bm{y},\bm{m}\bm{q}^{-1/2}\bm{\xi},\bm{Q}^*-\bm{m}^\top \bm{q}^{-1}\bm{m}) \psi_{2}\left(\bm{f}_{\rm out}(\bm{y},\bm{q}^{1/2}\bm{\xi},\bm{V})\right)\right] \;,
\end{equation}
\end{subequations}
where $\overset{P}{\to}$ denotes convergence in probability as $n,d \to \infty$
and the parameters $({\bm m},{\bm q},{\bm V})$ are the solution (assumed to be unique) of the following set of self-consistent equations  (where we introduced the auxiliary parameters $(\bm{ \hat m},\bm{ \hat q},\bm{ \hat V})$ :
\begin{align}
\label{eq:SE}
&\begin{cases}
    \bm{m}=&\;\mathbb{E}_{\bm \xi}\left[\mathcal{Z}^*_w (\bm{\hat m}\bm{\hat q}^{-1/2}\bm{\xi},\bm{\hat m}^T\bm{\hat q}^{-1}\bm{\hat m})\bm{f}^*_w(\bm{\hat m}\bm{\hat q}^{-1/2}\bm{\xi},\bm{\hat m}^T\bm{\hat q}^{-1}\bm{\hat m})\,\bm{f}_w(\bm{\hat q}^{1/2}\bm{\xi},\bm{\hat V})^\top\right],\\
    \bm{q}=&\;\mathbb{E}_{\bm \xi}\left[\mathcal{Z}^*_w(\bm{\hat m}\bm{\hat q}^{-1/2}\bm{\xi},\bm{\hat m}^T\bm{\hat q}^{-1}\bm{\hat m})\bm{f}_w(\bm{\hat q}^{1/2}\bm{\xi},\bm{\hat V})\bm{f}_w(\bm{\hat q}^{1/2}\bm{\xi},\bm{\hat V})^\top\right],\\
    \bm{V}=&\;\mathbb{E}_{\bm \xi}\left[\mathcal{Z}^*_w(\bm{\hat m}\bm{\hat q}^{-1/2}\bm{\xi},\bm{\hat m}^T\bm{\hat q}^{-1}\bm{\hat m})\partial_{\bm{\gamma}}\bm{f}_w(\bm{\hat q}^{1/2}\bm{\xi},\bm{\hat V})\right],
\end{cases}\notag\\
&\begin{cases}
\bm{\hat m} =&\;\alpha\,\mathbb{E}_{\bm{\xi}}\left[\int_{\mathbb{R}^{k}}\dd\bm{y}~\mathcal{Z}_{\rm out}^* (\bm{y},\bm{m}\bm{q}^{-1/2}\bm{\xi},\bm{Q}^*-\bm{m}^\top \bm{q}^{-1}\bm{m})\bm{f}_{\rm out}^*(\bm{y},\bm{m}\bm{q}^{-1/2}\bm{\xi},\bm{Q}^*-\bm{m}^\top \bm{q}^{-1}\bm{m})\,\bm{f}_{\rm out}(\bm{y},\bm{q}^{1/2}\bm{\xi},\bm{V})^\top \right],\\
\bm{ \hat q}=&\;\alpha\,\mathbb{E}_{\bm{\xi}}\left[\int_{\mathbb{R}^{k}}\dd\bm{y}~\mathcal{Z}_{\rm out}^* (\bm{y},\bm{m}\bm{q}^{-1/2}\bm{\xi},\bm{Q}^*-\bm{m}^\top \bm{q}^{-1}\bm{m})\bm{f}_{\rm out}(\bm{y},\bm{q}^{1/2}\bm{\xi},\bm{V})\,\bm{f}_{\rm out}(\bm{y},\bm{q}^{1/2}\bm{\xi},\bm{V})^\top \right],\\
\bm{\hat V} =&\;-\alpha\,\mathbb{E}_{\bm{\xi}}\left[\int_{\mathbb{R}^{k}}\dd\bm{y}~\mathcal{Z}_{\rm out}^* (\bm{y},\bm{m}\bm{q}^{-1/2}\bm{\xi},\bm{Q}^*-\bm{m}^\top \bm{q}^{-1}\bm{m})\partial_{\bm{\omega}} \bm{f}_{\rm out}(\bm{y},\bm{q}^{1/2}\bm{\xi},\bm{V})\right].
\end{cases}
\end{align}
We have made use of the following auxiliary functions:
\begin{subequations}
\label{eq:auxiliary_defs}
\begin{equation}
    \mathcal{Z}^{*}_w\left(\bm{\gamma},\bm{\Lambda}\right) =\E_{\bm{w}^*\sim P^*_w}\left[e^{-\frac 12 {\bm{w}^*}^\top\bm{\Lambda}{\bm{w}^*}+\bm{\gamma}^\top \bm{w}^*}\right]  \;,
\end{equation}
\begin{equation}
     \bm{f}^{*}_w\left(\bm{\gamma},\bm{\Lambda}\right) =\partial_{\bm{\gamma}}\ln \mathcal{Z}^{*}_w\left(\bm{\gamma},\bm{\Lambda}\right) \;,
\end{equation}
\begin{equation}
    \mathcal{Z}^{*}_{\rm out}\left(\bm{y},\bm{\omega},\bm{V}\right)=\E_{\bm{z}^*\sim \mathcal{N}(\bm{0},\bm{I}_{k})}\left[\delta\left(\bm{y}-\phi_{\text{out}}\left(\bm{V}^{1/2}\bm{z}^*+\bm{\omega}\right)\right)\right]  \;,
\end{equation}
\begin{equation}
    \bm{f}^{*}_{\rm out}\left(\bm{y},\bm{\omega},\bm{V}\right)= \partial_{\bm{\omega}}\ln \mathcal{Z}^{*}_{\rm out}\left(\bm{y},\bm{\omega},\bm{V}\right) \;,
\end{equation}
\begin{equation}
    f_{w}\left(\bm{\gamma},\bm{\Lambda}\right) = \prox_{\bm{\Lambda}^{-1}r_{\lambda}}\left(\bm{\Lambda}^{-1}\bm{\gamma}\right)  \;.
\end{equation}
\begin{equation}
    f_{\rm out}\left(\bm{y},\bm{\omega},\bm{V}\right)=\prox_{\bm{V}\ell(\cdot, \bm{y})}(\bm{\omega})  \;.
\end{equation}
\end{subequations}
where $P^*_{w}$ is the distribution (in $\mathbb{R}^{k}$) of the teacher weights and, for any function $f$
\begin{align}
    \prox_{\bm{\tau}f}(\bm{x}) = \underset{\bm{z}\in\mathbb{R}^{k}}{\argmin}\left[\frac{1}{2}(\bm{z}-\bm{x})^{\top}\bm{\tau}^{-1}(\bm{z}-\bm{x})+f(\bm{z})\right]
\end{align}
is a proximal operator (here defined with matrix parameters, see Appendix \ref{sec:req_back} for more detail). The simplified expressions of the auxiliary functions are provided in Appendix~\ref{app:denoising_fcts} and depend on the choices of the teacher weights distribution, the regularisation and the loss function. 
\end{theorem}
\paragraph{Proof outline.} We now provide a short outline of the proof for the asymptotic performance of the estimator obtained with convex ERM. The idea, pioneered in \cite{bayati2011lasso}, is to express the estimator $\hat{\mathbf{W}}$ as the limit of a carefully chosen sequence whose iterates have an exact, rigorous asymptotic characterisation. Such a sequence can be designed using iterations of an approximate message passing (AMP) algorithm \cite{donoho2009message,zdeborova2016statistical} for which the statistics of the iterates are asymptotically characterised by a closed form, low dimensional iteration, the \emph{state evolution equations}. In our case we need a matrix-valued AMP iteration with separable update functions, rigorously covered in \cite{javanmard2013state,gerbelot2021graph}. One then needs to show that converging trajectories of the iteration can be systematically found, in order to properly define the limit of the sequence. Here this is made possible by the convexity of the problem. In the end, the performance of the estimator is shown to be characterised by the fixed point of the state evolution equations.
\subsection{Bayes-optimal performance}
\label{sec:bo}
The sufficient statistics describing the performance of the Bayes-optimal estimator \eqref{eq:BO_estimator} can also be derived in the high-dimensional limit, and are closely related to the free entropy density. Indeed, in Appendix \ref{sec:app:replica} we show that the Bayes-optimal estimator can be fully characterised by only one overlap matrix $\bm{q}\in\mathbb{R}^{k\times k}$ which is given by the solution of following extremisation problem:
\begin{subequations}
\label{eq:free_entropy}
\begin{equation}
     \Phi ={\rm extr}_{\bm{q},\bm{\hat q}}\left\{-\frac{1}{2}\Tr{\left[\bm{q}\bm{\hat q}\right]} + \Psi_w^*(\bm{\hat q})+\alpha \Psi^*_{\rm out}(\bm{q})\right\}  \;,
\end{equation}
\begin{equation}
    \Psi_w^*(\bm{\hat q})=\mathbb{E}_{\bm{\xi}}\left[\mathcal{Z}_{w}^*\left(\bm{\hat q}^{1/2}\bm{\xi},\bm{\hat q}\right)\ln\mathcal{Z}_{w}^*\left(\bm{\hat q}^{1/2}\bm{\xi},\bm{\hat q}\right)\right]  \;,
\end{equation}
\begin{equation}
    \Psi_{\rm out}^*(\bm{q})=\mathbb{E}_{\bm{\xi}}\left[\int_{\mathbb{R}^{k}}\dd\bm{y}~\mathcal{Z}_{\rm out}^*\left(\bm{y};\bm{q}^{1/2}\bm{\xi},\bm{Q}^*-\bm{q}\right)\ln\mathcal{Z}_{\rm out}^*\left(\bm{y};\bm{q}^{1/2}\bm{\xi},\bm{Q}^*-\bm{q}\right)\right]  \;,
\end{equation}
\end{subequations}
\noindent where $(\mathcal{Z}_{w}^{*}, \mathcal{Z}_{\text{out}}^{*})$ are the same auxiliary functions defined in Eqs.~\eqref{eq:auxiliary_defs}, and $\Phi$ is the free entropy density defined in Eq.~\eqref{eq:free_entropy_def}.
Looking for extremisers of the equation above leads to the following set of self-consistent equations:
\begin{align} 
\label{eq:SE_BO}
\begin{cases}
    \bm{q}=&\;\mathbb{E}_{\bm \xi}\left[\mathcal{Z}^{*}_{w}(\bm{\hat q}^{1/2}\bm{\xi},\bm{\hat q})\bm{f}^{*}_{w}(\bm{\hat q}^{1/2}\bm{\xi},\bm{\hat q})\bm{f}^{*}_{w}(\bm{\hat q}^{1/2}\bm{\xi},\bm{\hat q})^\top\right],\\
   \bm{ \hat q}=&\;\alpha\,\mathbb{E}_{\bm{\xi}}\left[\int_{\mathbb{R}^{k}}\dd\bm{y}~\mathcal{Z}^{*}_{\text{out}} (\bm{y};\bm{q}^{1/2}\bm{\xi},\bm{Q}^*-\bm{q})\bm{f}^{*}_{\text{out}}(\bm{y};\bm{q}^{1/2}\bm{\xi},\bm{Q}^*-\bm{q})\,\bm{f}_{\rm out}^*(\bm{y};\bm{q}^{1/2}\bm{\xi},\bm{Q}^*-\bm{q})^\top \right],
   \end{cases}
\end{align}
\noindent where, again, $(\bm{f}_{w}^{*}, \bm{f}_{\text{out}}^{*})$ are the same as defined in Eqs.~\eqref{eq:auxiliary_defs}. Note the similarity between the equations \eqref{eq:SE_BO} above and the ones characterising the sufficient statistics of ERM, Eq.~\eqref{eq:SE}. Indeed, the equations above can be obtained from the ones of ERM through the following mapping, known in the context of statistical physics as \emph{Nishimori conditions} \cite{Nishimori2001StatisticalPO}:
\begin{align}
\label{eq:nishimori}
    &\bm{f}_{w} \to \bm{f}_{w}^{*}, && \bm{f}_{\rm out} \to \bm{f}_{\rm out}^*,\\
    &\bm{m} \to \bm{q}, && \bm{\hat m}\to \bm{\hat q},\\
    &\bm{V} \to \bm{Q}^* -\bm{q}, &&  \bm{\hat V} \to \bm{\hat Q}^* +\bm{\hat q}.
\end{align}
Intuitively, the student's additional knowledge of the data generating process is translated by choosing the same set of auxiliary functions as the teacher. These conditions imply, on average, no statistical difference between the ground truth configuration and a configuration sampled uniformly at random from the posterior distribution. Therefore, in the Bayes-optimal setting there is no distinction between the teacher-student overlap and the student self-overlap. This connection is further discussed in Appendix~\ref{sec:app:replica}, where we show that both set of equations can be derived from a common framework.

Despite the close similarity between the two sets of self-consistent equations, note one key difference: the set of extrema in Eq.~\eqref{eq:SE_BO} is not necessarily a single point. This means that, differently from Eqs.~\eqref{eq:SE}, the fixed-point of the self-consistent equations \eqref{eq:SE_BO} might not be unique. In this case, it is important to stress that the overlap $\bm{q}$ corresponding to the Bayes-optimal estimator \eqref{eq:BO_estimator} is, by definition, the one with highest free entropy density. Therefore, the Bayes-optimal generalisation error is then evaluated by finding the fixed point of equations~\eqref{eq:SE_BO} that maximises the free entropy \eqref{eq:free_entropy}. 

A proof of this claim and of equations~\eqref{eq:SE_BO} and \eqref{eq:free_entropy} for the Bayes-optimal performance was done by~(\cite{aubin2019committee}, Theorem 3.1, and \cite{barbier2021overlap}) in the setting of the committee machine, by an interpolation method that shows the correctness of the replica prediction for the free-entropy of the system.  Their proof applies to teacher-student committee machines with bounded output channel, prior distribution with finite second moment and Gaussian i.i.d. inputs. Therefore, it applies to the multi-class perceptron of our setting, both with Gaussian and Rademacher prior. 

\subsection{Generalisation error}
The characterisation of the generalisation error in the high-dimensional limit is a direct consequence of Theorem~\ref{thm:ERM_main}.
\begin{corollary} \label{cor:gen_error_ERM}
In the high-dimensional limit the asymptotic generalisation error associated to the ERM estimator \eqref{eq:optimization_ERM} can be expressed only as a function of the parameters $(\bm{m},\bm{q})$ obtained by solving the self-consistent equations \eqref{eq:SE}:
\begin{equation} \label{eq:generalization_ERM}
    \varepsilon_{\rm gen}= \pr_{(\bm{\nu},\bm{\mu})\sim\mathcal{N}(\bm{0}, \bm{\Sigma})} \left( \phi_{\rm out}(\bm{\mu}) \neq \phi_{\rm out}(\bm{\nu}) \right),
\end{equation}
where $\bm{\Sigma} = \begin{bmatrix}\bm{Q}^* & \bm{m} \\ \bm{m} & \bm{q}\end{bmatrix}$.
\end{corollary}
The proof of Corollary~\ref{cor:gen_error_ERM} is straightforward and follows by noticing that for any $\bm v$, the function $\psi(\bm{z}) = \mathds{1}(\phi_{\rm out}(\bm{z}) \neq \phi_{\rm out}(\bm{v}))$ is pseudo-Lipschitz.

As one can expect from the discussion in Section \ref{sec:bo}, the Bayes-Optimal generalisation error is obtained by a similar, but simpler expression depending only on the overlap $\bm{q}$, obtained by extremising~\eqref{eq:free_entropy}:
\begin{equation} \label{eq:generalization_Bayes}
    \varepsilon_{\rm gen}= \pr_{\bm{\xi}\sim \cN(\bm{0}, \bm{I}_k)} \left( \phi_{\rm out}(\bm{q}^{1/2}\bm{\xi} ) \neq \phi_{\rm out} ({\bm{Q}^*}^{1/2}\bm{\xi}) \right).
\end{equation}
\section{Approximate message-passing algorithm}
\label{sec:amp}
\begin{algorithm}[ht!]
   \caption{Approximate message passing}
   \label{alg:amp}
\begin{algorithmic}
   \State {\bfseries Input:} data matrix $ \X \in \R^{\nsamp\times\inp}$ and label matrix $ \Y \in \R^{\nsamp\times k} $
   \State  Initialise $ \hat{\w}^{0}_j  \in \R^{k} $ and $\hat{\C}_{j}^0 \; , \; \fouti^0 \in \R^{k \times k}  $
    \State  for $j = 1,..., \inp$  and $\i = 1, ..., \nsamp$ at $t=0$
   \Repeat
   \State \underline{Channel updates}
    \vspace{0.05cm}
   \State Mean $\bm{\om}_\i \in \R^{k} $ and variance $\V_\i \in \R^{k\times k}$ 
   \vspace{0.15cm}
   \State $ \V^{t}_\i  = \frac{1}{\inp} \sum_{j=1}^\inp x_{j\i}^2  \hat{\C}_{j}^t   $
    \vspace{0.05cm}
   \State $ \om_{\i}^t = \frac{1}{\sqrt{\inp}}   \sum_{l=1}^\inp \hat{\w}_{j}^t 
   - (  \V^{t}_\i )^\top  \fouti^{t-1}   $
   \vspace{0.15cm}
   \State Denoisers $  \fouti \in \R^{k} $ and  $ \partial_{\om}  \fouti \in \R^{k \times k} $  
   \vspace{0.15cm}
   \State $\fouti^t \leftarrow \fout \left( \y_\i , \om_{\i}^t , \V^{t}_\i  \right)      $
    \vspace{0.05cm}
   \State $  \partial_{\om}  \fouti^t \leftarrow  \partial_{\om} \fout  \left( \y_\i , \om_{\i}^t , \V^{t}_\i  \right)     $
   \vspace{0.25cm}
    \State \underline{Prior updates}
     \vspace{0.05cm}
   \State Mean $ \gamm_j \in \R^{k} $ and variance $\Lamb_j \in \R^{k\times k}$ 
   \vspace{0.15cm}
     \State $ \Lamb^{t}_j  = -  \frac{1}{\inp} \sum_{\i=1}^\nsamp  x_{j \i }^2 \partial_{\om} 
     \fouti^t  $
       \vspace{0.05cm}
   \State $ \gamm^{t}_j  =   \frac{1}{\sqrt{\inp}} \sum_{\i=1}^\nsamp  x_{j \i } \fouti^t  
    + \Lamb_{j}^t \hat{\w}^{t}_j   $
      \vspace{0.15cm}
      \State Posterior estimators $ \hat{\w}_j \in \R^{k} $ and $ \hat{\C}_j \in \R^{k \times k} $
    \vspace{0.15cm}
   \State $\hat{\w}_{j}^t  = f_{\bm w}\left(  \gamm^{t}_j , \Lamb_{j}^t   \right)    $ 
   \vspace{0.05cm}
   \State $\hat{\C}_{j}^t  = \partial_{\gamm} f_{\bm{w}}  \left( \gamm^{t}_j , \Lamb_{j}^t    \right)    $ 
     \vspace{0.25cm}
   \State $t \leftarrow t + 1$
    \vspace{0.15cm}
   \Until Convergence on $\hat{\w}_j$ and $\hat{\C}_j$
    \State {\bfseries Output:}  $ \{ \hat{\w}_j \}_{j=1}^{\inp} $ and $\{ \hat{\C}_j \}_{j=1}^{\inp}$.
\end{algorithmic}
\end{algorithm}
In order to illustrate our theoretical results for the performance of the Bayesian \eqref{eq:BO_estimator} and ERM \eqref{eq:optimization_ERM} estimators, we would like to compare our asymptotic expressions for the generalisation error with finite instance simulations. On one hand, the regularised empirical risk defined in \eqref{eq:optimization_ERM} is strongly convex, and therefore it can be readily minimised with any descent-based algorithm such as gradient descent or stochastic gradient descent. Indeed, in the ERM simulations that follow we employ out-of-the-box multi-class solvers from Scikit learn \cite{scikit} to assess our theoretical result from Theorem \ref{thm:ERM_main}. On the other hand, explicitly computing the Bayesian estimator \eqref{eq:BO_estimator} requires sampling from the posterior, an operation which is prohibitively costly in high-dimensions. Instead, in this manuscript we employ an \emph{Approximate Message Passing} (AMP) algorithm to efficiently approximate the posterior marginals. AMP has several interesting properties which make it a popular tool in the study of random problems. First, it is proven to be optimal among a class of random estimation problems \cite{celentano2020estimation}, and for this reason it is widely used as a benchmark to assess algorithmic complexity. Second, it admits a set of scalar \emph{state evolution equations} allowing to track its performance in high-dimensions \cite{javanmard2013state}. 

For the Bayes-optimal estimation problem considered here, AMP is summarised by the pseudo-code in Algorithm \ref{alg:amp}. It follows the well-known AMP algorithm for generalised linear estimation \cite{donoho2009message, rangan2011generalized}, which takes advantage of the high-dimensional limit $d \to\infty$ by approximating the posterior distribution \eqref{eq:posterior} by a multivariate Gaussian distribution through a belief propagation procedure expanded in powers of $ d^{-1}$. The difference is that the estimators $ \hat{\bm{w}}_j  $ are $k$-dimensional vectors and their variances  $ \hat{\bm{C}}_j  $ are $k\times k$ dimensional matrices, with $j=1,\ldots, d$. Both channel and prior updated functions, $\bm{f}_\text{out}$ and $\bm{f}_{\bm{w}}$, respectively, are defined in Appendix \ref{app:denoising_fcts}. For a detailed derivation of the algorithm, see \cite{aubin2020}. 

Several versions of this $k$-fold AMP and the associated state evolution appeared in previous works, e.g. \cite{aubin2019committee}. It can be shown that the state evolution equations associated to Algorithm \ref{alg:amp} for Bayes-optimal estimation coincide exactly with the self-consistent equations \eqref{eq:SE_BO} presented in Section \ref{sec:bo} starting from an \emph{uninformed initialisation} $\bm{q}_0  \approx \bm{0}$ \cite{aubin2019committee}. This interesting property implies that when the extremisation problem in Eq.~\eqref{eq:free_entropy} has only one extremiser, AMP provides an exact approximation to the Bayes-optimal estimator in the high-dimensional limit. Instead, when there are more than one maxima in Eq.~\eqref{eq:free_entropy}, AMP will converge to an estimator with overlap $\bm{q}$ closest to the uninformed initial condition. If this is not the global maxima, this corresponds to a situation where AMP differs from the Bayes-optimal estimator. Since AMP provides a bound on the performance of first-order algorithms, this situation is an example of an \emph{algorithmic hard phase}, where it is conjectured that the statistical optimal performance cannot be achieved by algorithms running in time $\sim O(d^2)$.

We have implemented Algorithm \ref{alg:amp} for $k=3$ using the mapping presented in Appendix \ref{appendix:mapping}, which makes the estimators $(k-1)$-dimensional vectors and their variances $(k-1)\times (k-1)$ dimensional matrices. The detailed expressions for the computation of the denoising functions, as well as the integrals to be numerically evaluated are presented in Appendix \ref{app:amp}. 
\begin{figure*}[ht!]
\centering
\begin{subfigure}[t]{.45\textwidth}
  \centering
  \includegraphics[width=\linewidth]{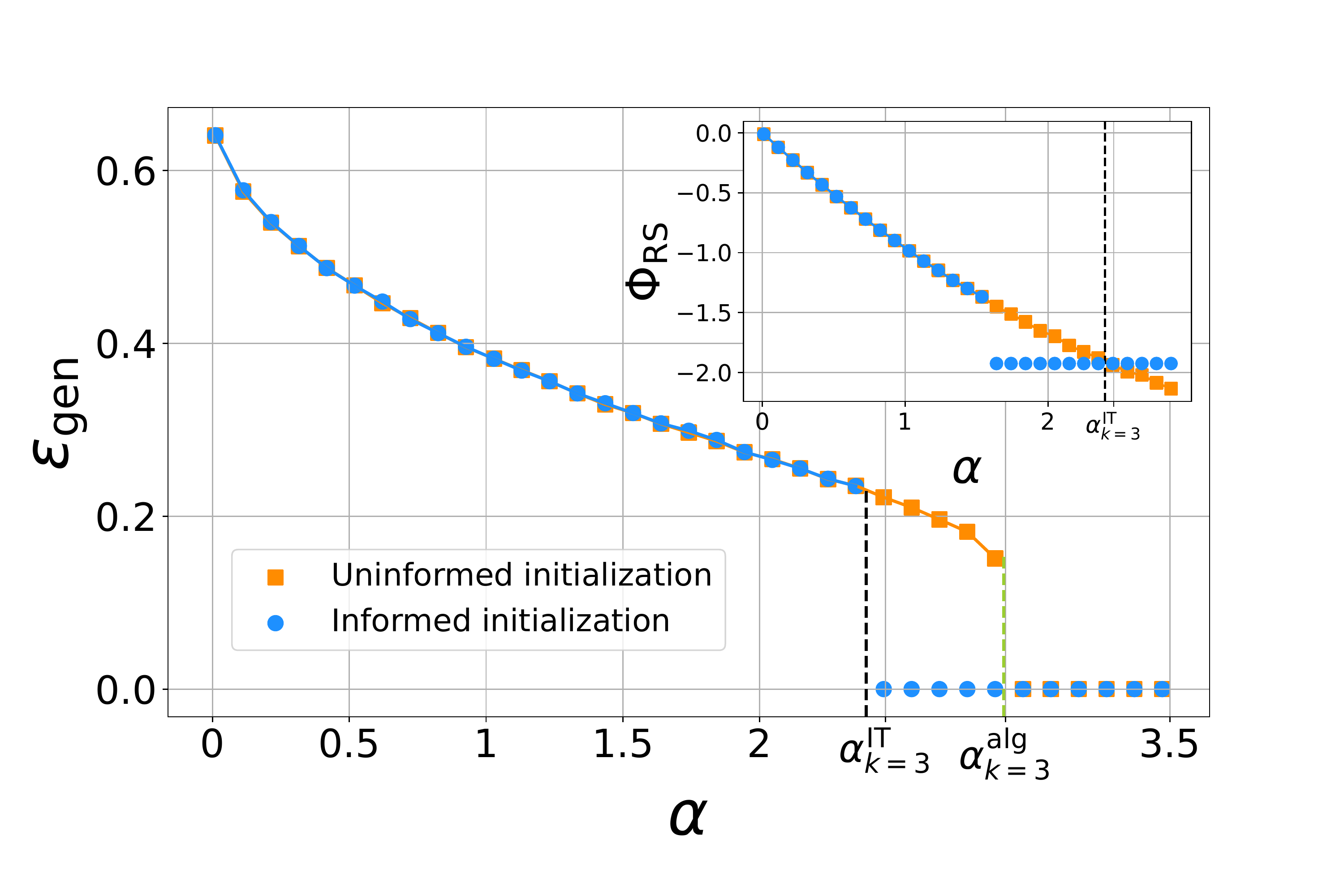}
  \caption{Generalisation error $\varepsilon_{\rm gen}$ as a function of the sample complexity $\alpha$ evaluated using Eqs.~(\ref{eq:generalization_Bayes}). The orange points correspond to the error that would be asymptotically reached by the randomly initialised AMP algorithm. The blue points mark the Bayes-optimal generalisation error. The inset depicts the corresponding free entropies as a function of $\alpha$, their crossing locating the information-theoretic transition to perfect generalisation at $\alpha^{\rm IT}_{k=3} = 2.45$. AMP reaches perfect generalisation starting from $\alpha^{\rm alg}_{k=3} = 2.89$. }
  \label{fig:k=3binary_error}
\end{subfigure}%
\hspace{0.5cm}
\begin{subfigure}[t]{.45\textwidth}
  \centering
  \includegraphics[width=\linewidth]{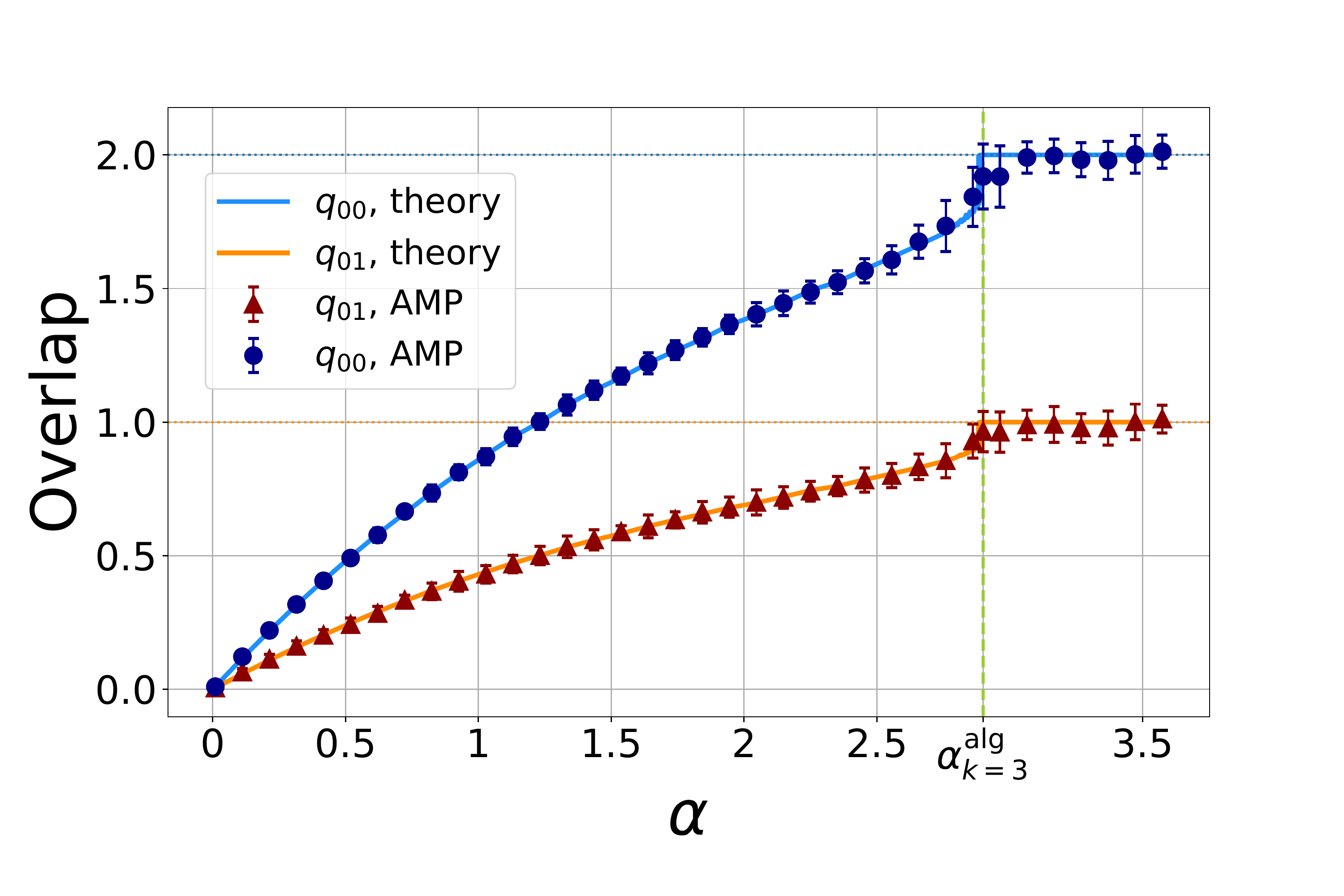}
  \caption{Diagonal ($q_{00}$) and anti-diagonal ($q_{01}$) entries of the self-overlap $2\times 2$ matrix $\bm{q}$ as a function of the sample complexity $\alpha$ in the Bayes-optimal setting. The full lines mark the fixed points of the saddle point equations of Eq.~\eqref{eq:SE_BO}, while the symbols represent
  the result obtained from the AMP algorithm described in Section~\ref{sec:amp} averaged over 20 runs. The error bars represent the respective standard deviations. }
  \label{fig:k=3binary_overlap}
\end{subfigure}
\caption{Bayes-optimal performance in the case of Rademacher teacher weights with $k=3$ classes.}
\label{fig:k=3binary}
\end{figure*}
\begin{figure}[t!]
\centering
\includegraphics[scale=.35]{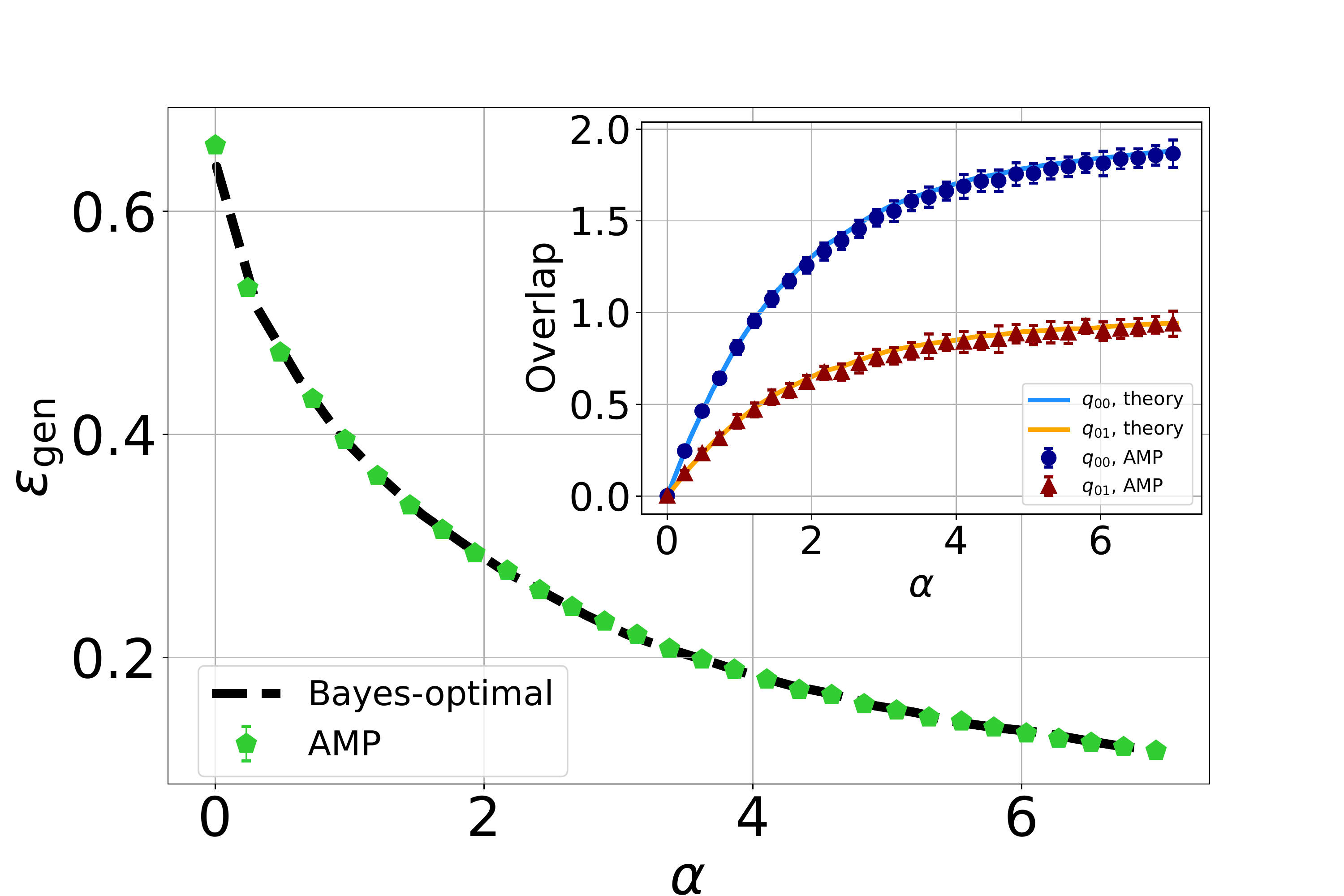}
\caption{{\bf AMP for Gaussian teacher prior:} Generalisation error $\varepsilon_{\rm gen}$ as a function of the sample complexity $\alpha$. The performance of AMP (averaged over 20 runs) computed using the formula in Eq.~\eqref{eq:generalization_Bayes}, is marked by the green symbols (error bars are smaller than the symbols) . The dashed black line indicates the Bayes optimal error. The inset displays the diagonal ($q_{00}$) and anti-diagonal ($q_{01}$) entries of the self-overlap $2\times 2$ matrix as a function of the sample complexity $\alpha$ in the Bayes-optimal setting. The full lines mark the fixed points of the saddle point equations of Eq.~\eqref{eq:SE_BO}, while the symbols represent the result obtained from the AMP algorithm described in Section~\ref{sec:amp} averaged over 20 runs. The error bars represent the respective standard deviations.}
\label{fig:gaussAMP}
\end{figure}
\begin{figure*}[t!]
\centering
\begin{subfigure}[t]{.45\textwidth}
\includegraphics[width=\linewidth]{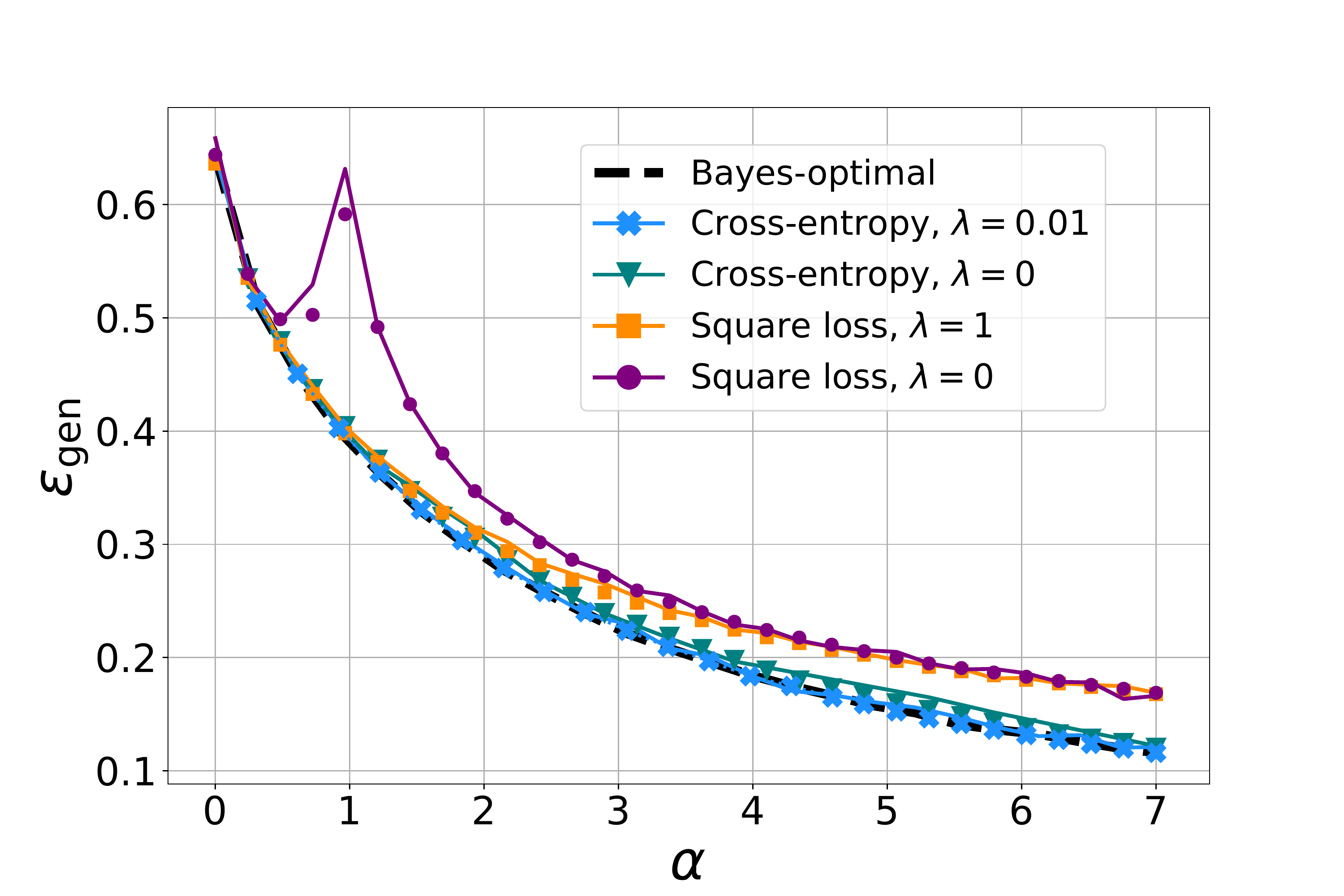}
\caption{Generalisation error $\varepsilon_{\rm gen}$ as a function of the sample complexity $\alpha$. The black dashed line depicts the Bayes-optimal performance. The full lines display the performance of ERM with cross-entropy loss (blue) and square loss (orange), each computed at optimised ridge regularisation ($\lambda=0.01$ and $\lambda=1$ respectively, see panels (c) and (d)) from the fixed points of Eq.~\eqref{eq:SE}. The symbols mark the results from numerical simulations at dimension $d=1000$, averaged over $250$ instances. We also plot the performance of simulations at zero regularisation and theory at $\lambda\rightarrow 0^+$, for both cross-entropy (teal) and square loss (purple).
  \label{fig:ref_gauss1}}
\end{subfigure}%
\hspace{0.5cm}
\begin{subfigure}[t]{.45\textwidth}
\includegraphics[width=\linewidth]{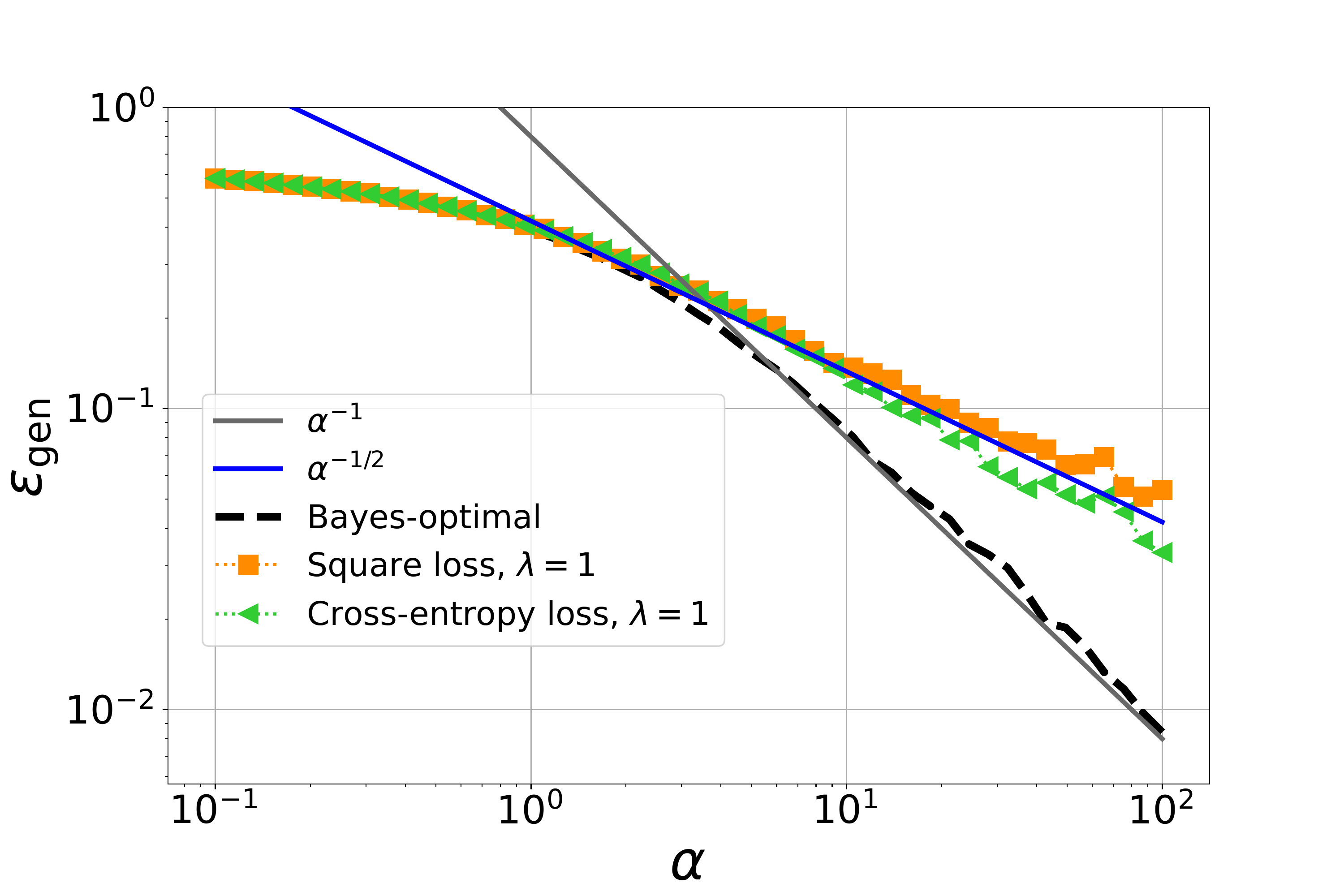}
\caption{{\bf Large$-\alpha$ behaviour:} Generalisation error $\varepsilon_{\rm gen}$ as a function of the sample complexity $\alpha$. We plot our theoretical predictions at large $\alpha$, in log-log scale for visibility purposes. The black dashed line marks the Bayes-optimal error. The symbols represent the error obtained by ERM at fixed regularisation strength $\lambda=1$. }
\label{fig:ref_gauss2}
\end{subfigure}
\begin{subfigure}[t]{.45\textwidth}
\includegraphics[width=\linewidth]{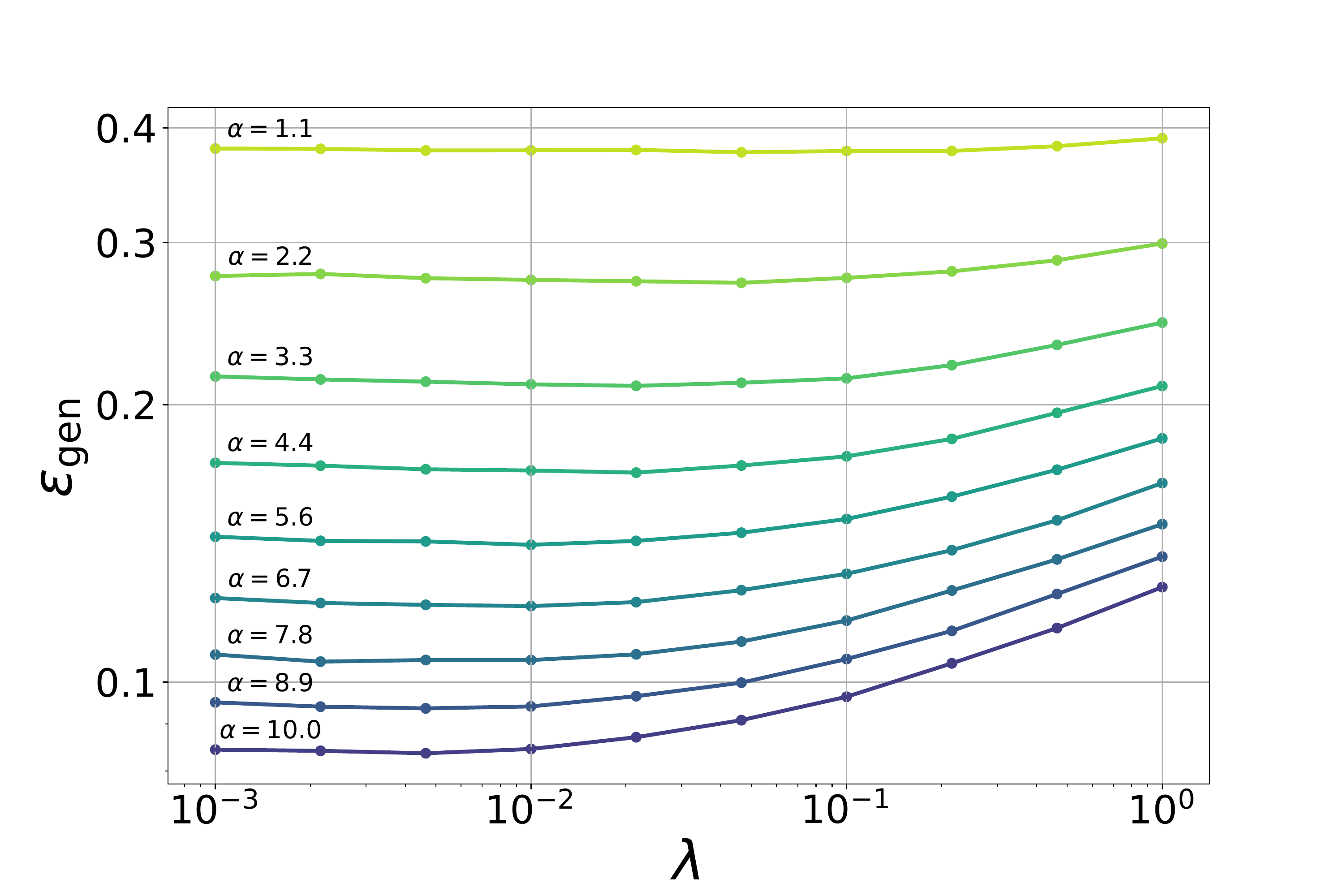}
    \caption{{\bf Cross-entropy loss:} Generalisation error $\varepsilon_{\rm gen}$ as a function of the regularisation strength $\lambda$, at fixed sample complexity $\alpha$. Different values of $\alpha$ are depicted with different colours. The curves are the result of numerical simulations performed at dimension $d=1000$, averaged over $250$ instances. We conclude that for these values of $\alpha$ the optimal $\lambda$ is close to $0.01$. 
  \label{fig:ref_gauss3}}
\end{subfigure}\hspace{0.5cm}
\begin{subfigure}[t]{.45\textwidth}
\includegraphics[width=\linewidth]{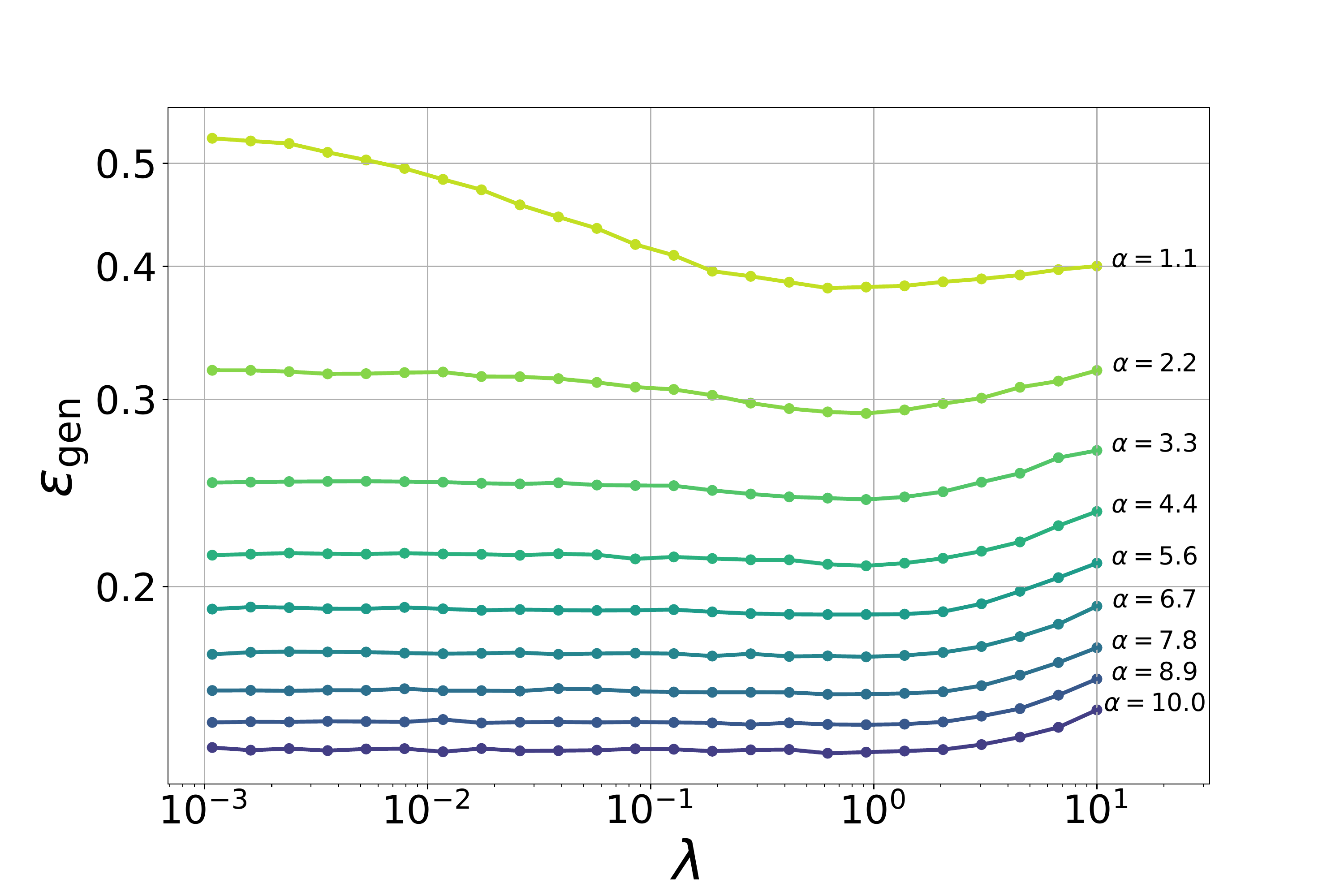}
 \caption{{\bf Square loss:} Generalisation error $\varepsilon_{\rm gen}$ as a function of the regularisation strength $\lambda$, at fixed sample complexity $\alpha$. Different values of $\alpha$ are depicted with different colours. The curves are the result of numerical simulations performed at dimension $d=1000$, averaged over $250$ instances. We conclude that for these values of $\alpha$ the optimal $\lambda$ is close to $1$. 
  \label{fig:ref_gauss4}}
\end{subfigure}
\caption{Bayes-optimal and ERM performances in the case of Gaussian teacher weights.}
\label{fig:gaussian_prior}
\end{figure*}
\section{Results for \texorpdfstring{$k=3$}{k=3} classes}
In this section we discuss the consequences of our theoretical results for the particular case of $k=3$ classes and compare them with numerical simulations. We investigate the dependence of the generalisation error on the sample complexity~$\alpha$. First, we consider the case of Rademacher teacher weights and show that a first-order phase transition arises in the Bayes-optimal performance. Then, we turn to the case of Gaussian teacher weights and explore the role of the regularisation strength $\lambda$ in approaching the Bayes-optimal performance with ERM. 
\paragraph{Bayes-optimal performance for Rademacher teacher ---} 
The main difference between Gaussian and Rademacher teacher is that in the second case perfect generalisation is achievable at finite sample complexity, in line with the results known for the two-classes case of~\cite{gyorgyi1990first,sompolinsky1990learning,seung1992statistical}. To compute the optimal information-theoretical performance, we have evaluated the global extremum of the replica free entropy. To that extent, we have run the replica saddle point iterations Eqs.~(\ref{eq:SE_BO}) with both uninformed and informed initialisations and computed the free entropy \eqref{eq:free_entropy} of the fixed points (if distinct) reached by the two initialisations. In Fig.~\ref{fig:k=3binary_error} we report the generalisation error corresponding to the fixed points reached by the two initialisations, along with their corresponding free entropy in the inset. We found that indeed, for Rademacher teacher weights, the generalisation error decreases continuously for $\alpha \leq \alpha_{\rm IT}^{(k=3)} \approx {2.45}$, and then jumps to zero for all $\alpha > \alpha_{\rm IT}^{(k=3)}$. From the statistical physics point of view, this discontinuous transition in the generalisation error corresponds to a \emph{first-order phase transition} associated to the discontinuous appearance of a second extrema in the free energy potential corresponding to perfect learning. As we have previously discussed, the state evolution of the AMP Algorithm \ref{alg:amp} is equivalent to doing gradient descent on the free energy potential \eqref{eq:free_entropy} starting from an uninformed random initialisation. Therefore, the appearance of a second extremum away from zero implies that AMP is not able to achieve the Bayes-optimal statistical performance. Since AMP is conjectured to be optimal among first-order methods \cite{celentano2020estimation}, this result is an example of a fundamental \emph{statistical-to-algorithmic gap} in this problem. For $\alpha > \alpha_{\rm algo}^{(k=3)} = 2.89$, we observe that the uninformed minima disappears, and we can check that this coincides with the sample complexity for which AMP is able to achieve zero generalisation error from random initialisation. This marks the algorithmic threshold, i.e. the sample complexity beyond which perfect generalisation is reachable algorithmically efficiently.  Our findings thus suggest the existence of an algorithmic \emph{hard phase} for $\alpha^{(k=3)}_{\rm{IT}}<\alpha<\alpha^{(k=3)}_{\rm{algo}}$, where the theoretically optimal performance is not reachable by efficient algorithms. 

We note here the comparison with the canonical perceptron with Rademacher teacher weights and two classes, where the same thresholds are well known to be $\alpha_{\rm IT}^{(k=2)}=1.249$, $\alpha_{\rm algo}^{(k=2)}=1.493$ \cite{gyorgyi1990first,sompolinsky1990learning,barbier2019optimal}. Naturally, these values are roughly twice smaller than the ones for $k=3$ since for $k$ classes the teacher has $k-1$ independent $d$-dimensional binary elements that need to be recovered in order 
to reach perfect generalisation. Comparing more precisely the values for $k=3$ and also their difference, all are slightly smaller than the double of the values for $k=2$. 

\paragraph{Bayes-optimal performance for Gaussian teacher ---} Fig.~\ref{fig:gaussAMP} and \ref{fig:gaussian_prior} summarise our results for the case of Gaussian teacher weights. The Bayes-optimal generalisation error, computed from Eq.~\eqref{eq:BO_error}, is depicted by the dashed black line in both figures and is a smooth, monotonically-decreasing function of the sample complexity. It is interesting to note that, for Gaussian teacher weights, the Bayes-optimal AMP algorithm -- described in Sec.~\ref{sec:amp} and marked by the green symbols in Fig.~\ref{fig:gaussAMP} -- achieves the Bayes-optimal performance. This is highly non-trivial: computing the Bayesian estimator usually requires sampling from the posterior distribution of the weights given the data, and therefore can be prohibitively costly in the high-dimensional regime considered here. For Gaussian weights AMP provides an exact approximation of the posterior marginals in quadratic time in the input dimension. 

\paragraph{Approaching Bayes-optimality with ERM ---} Instead, how does ERM compare to the Bayesian estimator? Note that the empirical risk in Eq.~\eqref{eq:optimization_ERM} is convex, and therefore, at variance with the posterior estimation, this problem can be readily simulated using descent-based algorithms such as stochastic gradient descent. The generalisation error obtained by ERM is plotted in Fig.~\ref{fig:ref_gauss1} as a function of the sample complexity. The full lines depict our theoretical predictions for the learning curves while the symbols mark the results from numerical simulations performed at finite dimension $d=1000$ (more details on the numerics are provided in Appendix~\ref{app:numerics}). We find excellent agreement between the two. For both cross-entropy and square losses, we show the performance achieved without regularisation ($\lambda=0$) and with naively-optimised $\lambda$, obtained by cross-validation in Fig.~\ref{fig:gaussian_prior} (c), (d).  Interestingly, we find that the optimally-regularised cross-entropy loss achieves a close-to-optimal performance, while the square loss maintains a finite gap with respect to the Bayes-optimal error even at fine-tuned regularisation strength. Similar results were obtained for the two-classes teacher student perceptron \cite{aubin2020generalization}.
The fact that regularised cross-entropy minimisation is so close to optimal is remarkable and worth further investigation in a more general setting. 

\paragraph{Large--$\alpha$ behaviour ---} Fig.~\ref{fig:ref_gauss2} considers again a Gaussian teacher prior and explores the behaviour of the generalisation error at large sample complexity. The Bayes-optimal performance is depicted in black and decays as $1/\alpha$ in the large$-\alpha$ regime. On the other hand, the performance obtained by ERM at fixed $\lambda$ displays a slower decay $\alpha^{-1/2}$. This is again compatible with the behaviour observed in the two-classes case \cite{aubin2020generalization}. It remains to be analysed whether for $k>2$ the optimally regularised ERM achieves the $1/\alpha$ rate as it does for the two classes. 

\paragraph{The role of regularisation ---} The lower panels of Fig. \ref{fig:gaussian_prior} further illustrate the role played by ridge regularisation. We plot the generalisation error as a function of the regularisation strength $\lambda$ at fixed sample complexity $\alpha$ for the cross-entropy (\ref{fig:ref_gauss3}) and the square loss (\ref{fig:ref_gauss4}). Different curves represent different values of sample complexity. We observe that the optimal regularisation depends only very mildly on the sample complexity $\alpha$ for this range of values of $\alpha$. 

\section*{Acknowledgements}
We acknowledge funding from the ERC
under the European Union’s Horizon 2020 Research and Innovation Programme Grant Agreement
714608-SMiLe. RV was partially financed by the Coordena\c{c}\~{a}o de Aperfeiçoamento de Pessoal de N\'{i}vel Superior - Brasil (CAPES) - Finance Code 001. RV is grateful to EPFL and IdePHICS lab for their generous hospitality during the realization of this project.

\newpage
\appendix
\section*{Appendix}\label{sec:app}
\addcontentsline{toc}{section}{\nameref{sec:app}}

\section{Prior reduction}
\label{appendix:mapping}
In this section we explain the mapping from $k$ to $k-1$ dimensions that we apply to evaluate our theoretical results from Eq. \eqref{eq:SE} as well as to implement Algorithm \ref{alg:amp}. The intuition is exactly the same of the binary perceptron: the knowledge of $k-1$ components of the one-hot label representation $\bm{y}$ is enough to determine the remaining component. Nevertheless for $k>2$, shifting the weights in order to reproduce this structure introduces additional correlations that must be taken into account. 

We recall that $\bm{W}^*$ is $d \times k$ matrix, and denote by $\bm{w}^*_l$ , $1\leq l \leq k$, its columns, each corresponding to a different class. Notice that the label $\bm{y} = \rm{e}_{\argmax_l(\{{\bm{w}^*_l}^\top \bm{x}\}_{l\in[k]})}$ given by Eq. \eqref{eq:model} of a data point $\bm{x}$ can be equivalently expressed by taking the $k^{\rm th}-$component, i.e. ${\bm{w}^*}^\top_k\bm{x}$, as a reference for comparison and setting
\begin{align}
\label{eq:prior_transf}
    \bm{\Tilde w}^*_{h}\leftarrow \bm{w}^*_{h}-\bm{w}^*_{k} \quad \text{ for all } 1 \leq h \leq k,
\end{align}
so that $\bm{\Tilde w}^*_{k} = \bm{0}$, and the problem is reduced to $k-1$ dimensions. We then replace $\bm{W}^*$ by $\bm{\Tilde W}^* \in \bR^{d \times (k-1)}$. Denoting $\bm{1}_k$ as the $k$-dimensional vector with all entries equal to $1$, we present schematically in Figure \ref{fig:ml_perc_red} the prior reduction.
\begin{figure}[ht!]
\centering
\centerline{\includegraphics[width=0.6\textwidth]{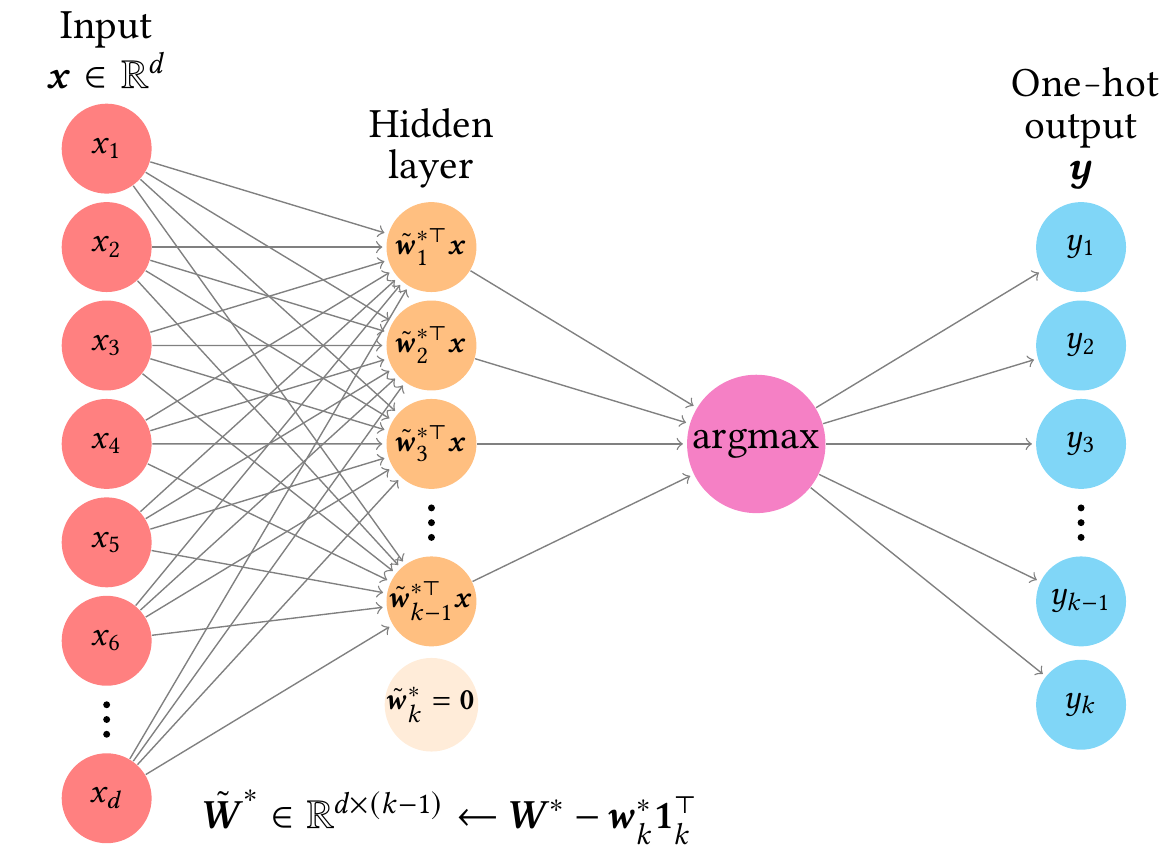}}
\caption{{\bf Prior reduction:} Schematic representation of the prior reduction mapping over the multiclass classification problem depicted in Figure \ref{fig:ml_perc_k}.}
\label{fig:ml_perc_red}
\end{figure}
Note that this mapping introduces correlations along the columns of $\bm{\Tilde W}^*$, but not along the rows, i.e. the $d$ components of each vector $\bm{\Tilde w}_l^*$ remains i.i.d. Therefore, the prior over the weights is still factorizable along the extensive dimension $d$.

\paragraph{Gaussian prior ---}

In the Gaussian prior setting,
\begin{equation}
P_w ( \bm{w}^* ) = \mathcal{N} (\bm{w}^* \rvert \bm{0} , \bm{I}_k  )  =  \frac{1}{(2 \pi )^{k/2}}   \exp\left(-\frac{1}{2} \bm{w}^{*\top} \bm{w}^*   \right)\;,
\end{equation}
where  $\bm{w}^* \in \mathbb{R}^k $ is a column of the matrix $ \bm{W}^*$, 
the transformation imposes a new covariance matrix with elements
\begin{equation}
 \bm{\Sigma}_{jl} \equiv   \Cov \left( \bm{w}_j^* -  \bm{w}_k^*    , \bm{w}_l^* -  \bm{w}_k^*    \right)  \,. 
\end{equation}
By making use of the identity
\begin{equation}
   \Cov \left( \alpha \bm{a} + \beta \bm{b} , \gamma \bm{c} +  \delta \bm{b} \right)  =
 \alpha \gamma    \Cov \left(  \bm{a} ,  \bm{c} \right) 
 + \alpha \delta    \Cov \left(  \bm{a} ,  \bm{d} \right)  
 + \beta \gamma    \Cov \left(  \bm{b} ,  \bm{c} \right)  
 +\beta \delta    \Cov \left(  \bm{b} ,  \bm{d} \right)  \,,  
\end{equation}
one can write
\begin{equation}
    \mathcal{N} (\bm{w}^* \rvert \bm{w}_k^* , \bm{\Sigma}  ) \propto \exp \left[ -\frac{1}{2}
      \left( \bm{w}^* - \bm{w}_k^*  \right)^\top  \bm{\Sigma}^{-1} \left( \bm{w}^* - \bm{w}_k^* \right) \right] 
      \leftarrow \mathcal{N} (\bm{w}^* \rvert \bm{0}_k , \bm{I}_k  )  \;,
\end{equation}
with
\begin{equation}
     \bm{\Sigma} = \bm{I}_k  - \rm{e}_k \rm{e}_k^\top +   \rm{e}_k \sum_{l \ne k }   \rm{e}_l^\top   + \left( \sum_{l \ne k }   \rm{e}_l  \right) \rm{e}_k^\top   \;,
\end{equation}
Thus, since all contributions related to the $k$-th degree of freedom become zero, the transformation \eqref{eq:prior_transf} allows us to write the mapping 
\begin{subequations}
\begin{equation}
    \bm{I}_{k-1} \leftarrow  \bm{I}_k  - \rm{e}_k \rm{e}_k^\top  \;,
\end{equation}
\begin{equation}
    \bm{1}_{k-1} \bm{1}_{k-1}^\top  \leftarrow  \rm{e}_k^\top + \rm{e}_k \sum_{l \ne k }   \rm{e}_l^\top   + \left( \sum_{l \ne k }   \rm{e}_l  \right) \rm{e}_k^\top  \;;
\end{equation}
\end{subequations}
and finally for $\tilde{\bm{w}}^* \in \mathbb{R}^{k-1}$ :
\begin{subequations}
\begin{equation}
     \mathcal{N} (\tilde{\bm{w}}^* \rvert \bm{0} , \tilde{\bm{\Sigma}} ) \propto \exp \left( -\frac{1}{2}
      \tilde{\bm{w}}^{*\top}  \tilde{\bm{\Sigma}}{}^{-1}  \tilde{\bm{w}}^* \right) \;,
\end{equation}
with covariance $\tilde{\bm{\Sigma}} \in \mathbb{R}^{(k-1) \times (k-1)}$ given by 
\begin{equation} 
\label{eq:cov}
    \tilde{\bm{\Sigma}} = \bm{I}_{k-1} + \bm{1}_{k-1} \bm{1}_{k-1}^\top    \;.
\end{equation}
\end{subequations}
Therefore each row of the reduced matrix $\bm{\Tilde W}^*$ follows a Gaussian distribution with $\bm 0$ mean and covariance matrix given by Eq.~\eqref{eq:cov}

\paragraph{Rademacher prior ---} In the Rademacher setting, $\bm{w}^* \in \mathbb{R}^k $,
\begin{equation}
    P_w (\bm{w}^*) = \frac{1}{2^k} \prod_{l=1}^k \left[ \delta(w_l^* + 1 ) + \delta(w_l^* - 1 )    \right] \;, 
\end{equation}
we can write
\begin{equation}
\begin{split}
      P_w (\bm{w}^*) =&   \frac{1}{2^k} \prod_{l=1}^k \left[ \delta(w_l^* - w_k^* + w_k^* + 1 ) + \delta(w_l^* -  w_k^* + w_k^*  - 1 )    \right] \\
  =&   \frac{1}{2^k}  \left[  \delta( w_k^* + 1 )  + \delta( w_k^* - 1 ) \right] \prod_{l=1}^{k-1} \left[ \delta\left( (w_l - w_k^* ) + w_k^* + 1 \right) + \delta( (w_l^* -  w_k^* ) + w_k^*  - 1 )    \right] \;,
\end{split}
\end{equation}
leading to the reduced prior for $\tilde{\bm{w}}^* \in \mathbb{R}^{k-1}$:
\begin{equation}
\label{eq:rad_prior_k1}
     P_{\tilde{w}} (\tilde{\bm{w}}^* \rvert  \bm{w}_k^* ) = \frac{1}{2^k}  \left[  \delta( w_k^* + 1 )  + \delta( w_k^* - 1 ) \right] \prod_{l=1}^{k-1} \left[ \delta\left( \tilde{w}_l^* + w_k^* + 1 \right) + \delta(
     \tilde{w}_l^* + w_k^*  - 1 )    \right]  \;.
\end{equation}

The dimensional reduction simplifies our analysis when it comes to the numerical evaluation both of the Gaussian integrals in Eq. \eqref{eq:SE} and of the prior and channel updates of AMP, as discussed in Appendix \ref{app:amp}. The same reformulation applies straightforwardly to the parameters $\bm{m}$ and $\bm{q}$.

\section{Replica calculation}
\label{sec:app:replica}
In this section, we carry out the (heuristic) replica computation leading to the system of equations \eqref{eq:SE} in the main text. We consider a general setting where the student has access to a prior distribution $P_w$ over the teacher weights and a model distribution $P_{\rm out}$, which can be the true ones or not. This formulation encompasses both the Bayes-optimal and non Bayes-optimal settings. As we shall see in the following, ERM can be seen as a special case of the latter. The posterior distribution of the student weights is given by
\begin{align}
\begin{split}
    P\left(\{\bm{w}_{l}\}_{l=1}^{k} | \bm{X},\bm{Y} \right) &= \frac{1}{Z_d} \prod_{l=1}^k P_w(\bm{w}_{l} ) \prod_{\mu=1}^n P_{\text{out}}(\bm{y}_\mu | \{h_{\mu l}\}_{l=1}^k )
    \end{split}
\end{align}
where we have defined $h_{\mu l} = \bm{w}_l^\top\bm{x}_\mu/\sqrt{d}$. The partition function is then
\begin{align}
   Z_d = \int_{\bR^{d\times k}} \dde \bm{w} \prod_{l=1}^k P_w(\bm{w}_{l} ) \prod_{\mu=1}^n P_{\text{out}}(\bm{y}_\mu | \{h_{\mu l}\}_{l=1}^k ).
\end{align}
By using the \emph{replica trick}, we can compute the free entropy in the high-dimensional limit as 
\begin{align}
    \Phi:=\lim_{d\rightarrow\infty}\Phi_d := \lim_{d\rightarrow\infty}\frac{1}{d} \E_{\bm{X},\bm{W}^*} \ln Z_d \approx  \lim_{d \to \infty}\lim_{p \to 0^+} \frac{1}{d}\, \frac{\partial}{\partial p} \E _{\bm{X},\bm{W}^*}Z_d^p.\label{eq:replica_trick}
\end{align}
We can then rewrite the average in Eq. \eqref{eq:replica_trick} as
\begin{align}
    \E_{\bm{X},\bm{W}^*} Z_d^p &= \E_{\bm{X},\bm{W}^*} \left[\int_{\bR^{d\times k}} \dde \bm{w} \prod_{l=1}^k P_w(\bm{w}_{l}) \prod_{\mu=1}^n P_{\text{out}}(\bm{y}_\mu | \{h_{\mu l}\}_{l=1}^k ) \right]^p \\
    &=\E_{\bm{X},\bm{W}^*} \left[ \prod_{a=1}^p  \int_{\bR^{d\times k}} \dde \bm{w}^a \prod_{l=1}^k P_w(\bm{w}_{l}^a ) \prod_{\mu =1}^n P_{\text{out}} \left( \bm{y}_\mu |\{h_{\mu l}^a\}_{l=1}^k \right) \right]\\
    &= \E_{\bm{X}} \int_{\mathbb{R}^{n\times k}} \dde \bm{Y} \prod_{a=0}^p \left[ \int_{\bR^{d\times k}} \dde \bm{w}^a \prod_{l=1}^k P^a_w(\bm{w}_{l}^a ) \prod_{\mu =1}^n P^a_{\text{out}} \left( \bm{y}_\mu |\{h_{\mu l}^a\}_{l=1}^k \right) \right],
\end{align}
where above we have renamed $\bm{w}^* =\bm{w}^0$. In order to account for both the Bayes-optimal and non Bayes-optimal cases, we keep the distinction between teacher and student distributions by adding an index $a$ to prior and model distributions. In what follows, $P_w^0=P_w^*$ and $P_{\rm out}^0=P_{\rm out}^*$ refer to the teacher, while  $P_w^{a>0}=P_w$ and $P_{\rm out}^{a>0}=P_{\rm out}$ to the student. Let us denote the covariance tensor of the $h_{\mu l}^a$ as
\begin{align}
    \E [h_{\mu l}^a h_{\nu l'}^b ] & = \delta_{\mu \nu} Q_{bl'}^{al},\\
    Q_{bl'}^{al} & = \frac{1}{d} \sum_{i=1}^d w_{il}^a w_{il'}^b,
\end{align}
with $\bm{Q}_a^b \in \bR^{k \times k}$.
We can rewrite the above as
\begin{align}
    \begin{split}
          \E_{\bm{X},\bm{W}^*} Z_d^p &= \E_{\bm{X}} \int_{\bR^{n\times k}} \dde \bm{Y} \prod_{a=0}^p \left[ \int_{\bR^{d\times k}} \dde \bm{w}^a \prod_{l=1}^k P^a_w(\bm{w}_{l}^a ) \prod_{\mu =1}^n P^a_{\text{out}} \left( \bm{y}_\mu |\{h_{\mu l}^a\}_{l=1}^k \right)\right]\\
          &=\prod_{(a,l);(b,l)}\int_{\mathbb{R}}\dde Q_{bl'}^{al}\; I_{\rm prior}(\{Q^{al}_{bl'}\})\;I_{\rm channel}(\{Q^{al}_{bl'}\}),
    \end{split}
\end{align}
where we have denoted
\begin{align}
    I_{\rm prior}(\{Q^{al}_{bl'}\})&=\prod_{a=0}^p\int_{\mathbb{R}^{d\times k}}\dde \bm{w}^a \, \left[\prod_{l=1}^kP^a_w(\bm{w}_{l}^a)\right]\prod_{(a,l);(b,l')}\delta\left(Q^{al}_{bl'}-\frac{1}{d}\sum_{i=1}^dw_{il}^aw_{il'}^b\right),\\
    I_{\rm channel}(\{Q^{al}_{bl'}\})&=\int_{\mathbb{R}^{n\times K}}\dde \bm{Y}\prod_{a=0}^p\int_{\mathbb{R}^{d\times k}}\dde h^a\left[\prod_{a=0}^p\prod_{\mu=1}^n P^a_{\rm out}(\bm{y}_\mu|h^a_\mu)\right]\\&\times\exp\left(-\frac{n}{2}\ln\det \bm{Q}-\frac{nk(p+1)}{2}\ln 2\pi-\frac{1}{2}\sum_{\mu=1}^n\sum_{a,b}\sum_{l,l'}h_{\mu l}^a (Q^{-1})_{bl'}^{al}h_{\mu l'}^b\right).
\end{align}
and we have introduced both the definitions of the overlaps $\{Q_{bl'}^{al}\}$ and the local fields $\{h^a_{\mu l}\}$. We can introduce the Fourier representation of the Dirac $\delta-$functions in the prior term $I_{\rm prior}$ and rewrite
\begin{align}
    \mathbb{E}Z^p_d=\prod_{(a,l);(b,l')}\int_{\mathbb{R}^2}\frac{\dde Q^{al}_{bl'}\,\dde \hat Q^{al}_{bl'}}{2\pi}\exp\left(d\,H(\bm{Q},\bm{\hat Q})\right),
\end{align}
where we have defined
\begin{align}
    H(\bm{Q},\bm{\hat Q}):=\frac{1}{2}\sum_{a=0}^p\sum_{l,l'}Q^{al}_{al'}\hat Q^{al}_{al'}-\frac{1}{2}\sum_{a\neq b}\sum_{l,l'}Q^{al}_{bl'}\hat Q^{al}_{bl'}+\ln I(\{\hat Q^{al}_{bl'}\})+\alpha \ln J(\{Q^{al}_{bl'}\})
\end{align}
and the auxiliary functions:
\begin{align}
   I(\{\hat Q^{al}_{bl'}\})&=\prod_{a=0}^p\int_{\mathbb{R}^k}\dde \bm{w}^a \,P^a_w(\bm{w}^a)\exp\left(-\frac{1}{2}\sum_{a=0}^p\sum_{l,l'}w_l^a \hat Q_{al'}^{al}w^a_{l'}+\frac{1}{2}\sum_{a\neq b}\sum_{l,l'}w_l^a \hat Q_{bl'}^{al}w^b_{l'}\right),\\
   J(\{Q^{al}_{bl'}\})&=\int_{\mathbb{R}^k}\dde \bm{y} \prod_{a=0}^p\int_{\mathbb{R}^k}\frac{\dde \bm{h}^a}{(2\pi)^{k(p+1)/2}}\frac{P_{\rm out}^a(\bm{y}|\bm{h}^a)}{\sqrt{\det \bm{Q}}}\exp\left(-\frac{1}{2}\sum_{a,b}\sum_{l,l'}h^a_l(Q^{-1})^{al}_{bl'}h^b_{l'}\right)
\end{align}
We observe that, upon exchanging the limits in $d$ and $p$, the high-dimensional limit of the free entropy can be computed via a saddle-point method:
\begin{align}
    \Phi=\lim_{d\rightarrow\infty}\mathbb{E}_{\bm{X},\bm{W}^*}\ln{Z}_d=\lim_{p\rightarrow 0^+}{\rm extr}_{\bm{Q},\bm{\hat Q}}\left[H(\bm{Q},\bm{\hat Q})\right].\label{eq:salle_point}
\end{align}
\subsection{Replica symmetric ansatz}
In order to progress in the calculation, we assume that the extremum in Eq. \eqref{eq:salle_point} is attained at $\{\bm{Q},\bm{\hat Q}\}$ described by a replica symmetric (RS) ansatz \cite{Nishimori2001StatisticalPO}. We distinguish between the Bayes-optimal and non Bayes-optimal cases. Note that in the Bayes-optimal case we can drop the $a-$index from the prior and model distributions. In the non Bayes-optimal case, we will denote the teacher distributions by $P^*_w,P^*_{\rm out}$ and the student ones simply by $P_w,P_{\rm out}$.
\paragraph{Bayes-optimal RS ansatz ---} In the Bayes-optimal setting we make the following ansatz:
\begin{align}
Q^{al}_{al'}=Q^*_{ll'}, && \hat Q^{al}_{al'}=\hat Q^*_{ll'}, && \forall a=0,..p, \,\forall l,l'\leq k\\
Q^{al}_{bl'}=q_{ll'},&& \hat Q^{al}_{bl'}=\hat q_{ll'},&& \forall a\neq b, \forall l,l'\leq k
\end{align}
The trace term is simplified as follows
\begin{align}
\label{eq:traceBO}
   \frac{1}{2}\sum_{a=0}^p\sum_{l,l'}Q^{al}_{al'}\hat Q^{al}_{al'}-\frac{1}{2}\sum_{a\neq b}\sum_{l,l'}Q^{al}_{bl'}\hat Q^{al}_{bl'}=\frac 12 (p+1)\sum_{l,l'}\hat Q^*_{ll'}Q^*_{ll'}-\frac 12 p(p+1)\sum_{ll'}\hat q_{ll'}q_{ll'}.
\end{align}
The prior and output terms can be simplified by performing a Hubbard-Stratonovich transformation that allows  to decouple the replica indices $a,b$. By indicating the standard Gaussian measure with $\mathcal{D}\bm{\xi}$: $\bm{\xi}\sim\mathcal{N}(\bm{0},\bm{I}_k)$, we obtain
\begin{align}
    I(\bm{\hat Q}^*,\bm{\hat q})&=\int_{\mathbb{R}^k}\mathcal{D}\bm{\xi} \left[\int_{\mathbb{R}^k}\dde \bm{w}\, P^*_w(\bm{w})\exp\left(-\frac{1}{2}\bm{w}^\top (\bm{\hat Q}^*+\bm{\hat q})\bm{w}+\bm{\xi}^\top \bm{\hat q}^{\frac 12}\bm{w}\right)\right]^{p+1},\\
    J( \bm{Q}^*, \bm{q})&=\int_{\mathbb{R}^k}\dde \bm{y}\int_{\mathbb{R}^k}\mathcal{D}\bm{\xi} \left[\int_{\mathbb{R}^k}\mathcal{D} \bm{h} \, P^*_{\rm out}\left(\bm{y}|(\bm{Q}^*-\bm{q})^{\frac{1}{2}}\bm{h}+\bm{q}^{\frac 12}\bm{\xi}\right)\right]^{p+1}.
\end{align}
Since we are interested in the $p\rightarrow 0^+$ limit, it is useful to rewrite
\begin{align}
\label{eq:logBO1}
\begin{split}
    \ln I(\bm{\hat Q}^*,\bm{\hat q})&=p\int_{\mathbb{R}^k}\mathcal{D}\bm{\xi} \int_{\mathbb{R}^k}\dde \bm{w}\, P^*_w(\bm{w})\exp\left(-\frac{1}{2}\bm{w}^\top (\bm{\hat Q}^*+\bm{\hat q})\bm{w}+\bm{\xi}^\top \bm{\hat q}^{\frac 12}\bm{w}\right)\\&\times\ln \int_{\mathbb{R}^k}\dde \bm{w}'\, P^*_w(\bm{w}')\exp\left(-\frac{1}{2}\bm{w}'^\top (\bm{\hat Q}^*+\bm{\hat q})\bm{w}'+\bm{\xi}^\top \bm{\hat q}^{\frac 12}\bm{w}'\right) + o(p),
    \end{split}
    \\
    \begin{split}
    \ln J( \bm{Q}^*, \bm{q})&=p\int_{\mathbb{R}^k}\dde \bm{y}\int_{\mathbb{R}^k}\mathcal{D}\bm{\xi} \int_{\mathbb{R}^k}\mathcal{D} \bm{h} \, P^*_{\rm out}\left(\bm{y}|(\bm{Q}^*-\bm{q})^{\frac{1}{2}}\bm{h}+\bm{q}^{\frac 12}\bm{\xi}\right)\\
    &\times \ln \int_{\mathbb{R}^k}\mathcal{D} \bm{h}' \, P^*_{\rm out}\left(\bm{y}|(\bm{Q}^*-\bm{q})^{\frac{1}{2}}\bm{h}'+\bm{q}^{\frac 12}\bm{\xi}\right) + o(p).
\end{split}
\label{eq:logBO2}
\end{align}
\paragraph{Non Bayes-optimal RS ansatz ---}  In the non Bayes-optimal setting we make the following ansatz:
\begin{align}
Q^{al}_{al'}=Q^0_{ll'}, && \hat Q^{al}_{al'}=\hat Q^0_{ll'}, && \forall a=1,..p, \,\forall l,l'\leq K\\
Q^{al}_{bl'}=q_{ll'},&& \hat Q^{al}_{bl'}=\hat q_{ll'},&& \forall a\neq b,\, a,b=1,...p,\, \forall l,l'\leq k\\
Q^{0l}_{al'}=m_{ll'},&& \hat Q^{0l}_{al'}=\hat m_{ll'},&& \forall a=1,...p,\, \forall l,l'\leq k\\
Q^{0l}_{0l'}=Q^*_{ll'}, && \hat Q^{0l}_{0l'}=\hat Q^*_{ll'}, &&\forall l,l'\leq k
\end{align}
The trace term is simplified as follows
\begin{align}
\label{eq:traceNBO}
\begin{split}
   \frac{1}{2}\sum_{a=0}^p\sum_{l,l'}Q^{al}_{al'}\hat Q^{al}_{al'}-\frac{1}{2}\sum_{a\neq b}\sum_{l,l'}Q^{al}_{bl'}\hat Q^{al}_{bl'}&=\frac 12 p\sum_{l,l'}\hat Q^0_{ll'}Q^0_{ll'}-\frac 12 p(p-1)\sum_{ll'}\hat q_{ll'}q_{ll'}\\&+\frac 12 \sum_{l,l'}\hat Q^*_{ll'}Q^*_{ll'}-p\sum_{l,l'}m_{ll'}\hat m_{ll'}.
   \end{split}
\end{align}
The prior term is
\begin{align}
\begin{split}
    I(\bm{\hat Q}^0,\bm{\hat q},\bm{\hat Q}^*,\bm{\hat m})=&\int_{\mathbb{R}^k}\mathcal{D}\bm{\xi}\int_{\mathbb{R}^k}\dde \bm{w}^* \,P^*_w(\bm{w}^*)\exp\left(-\frac{1}{2}{\bm{w}^*}^\top\bm{\hat Q}^* \bm{w}^*\right)\\&\times\left[\int_{\mathbb{R}^k}\dde \bm{w}\,P_w(\bm{w})\exp\left(-\frac{1}{2}\bm{w}^\top(\bm{\hat Q}^0+\bm{\hat q})\bm{w}+{\bm{w}^*}^\top \bm{\hat m} \bm{w}+\bm{\xi}^\top \bm{\hat q}^{\frac 12}\bm{w} \right)\right]^p.
\end{split}
\end{align}
In order to compute the output term we need to compute the inverse matrix
\begin{align}
    \bm{Q}^{-1}=\left[
    \begin{array}{ccccc}
        \bm{\tilde Q}^*  &\bm{\tilde m} & \ldots & \bm{\tilde m}  \\
       \bm{ \tilde m} & \bm{\tilde Q}^0 & \bm{\tilde q} 
         & \ldots\\
          \vdots & \bm{\tilde q} & \ddots& \bm{\tilde q}\\
          \bm{\tilde m} & \ldots & \bm{\tilde q} & \bm{\tilde Q}^0
    \end{array}
    \right]\in\mathbb{R}^{k(p+1)\times k(p+1)},
\end{align}
which has a similar block structure as $\bm{Q}$. The components of the inverse can be computed from the relation $\bm{Q}^{-1}\bm{Q}=\bm{I}_k$ and are given by
\begin{equation}
    \begin{split}
        \bm{\tilde Q}^*&=\left(\bm{Q}^*-p\,\bm{m}\left(\bm{Q}^0+(p-1)\bm{q}\right)^{-1}\bm{m}\right)^{-1},\\
        \bm{\tilde Q}^0&=(\bm{Q}^0-\bm{q})^{-1}+\left(\bm{Q}^0+(p-1)\bm{q}\right)^{-1}\left[\bm{m}\left(\bm{Q}^*-p\,\bm{m}\left(\bm{Q}^0+(p-1)\bm{q}\right)^{-1}\bm{m}\right)^{-1}\bm{m}\left(\bm{Q}^0+(p-1)\bm{q}\right)^{-1}-\bm{q}(\bm{Q}^0-\bm{q})^{-1}\right],\\
        \bm{\tilde q}&=\bm{\tilde Q}^0-(\bm{Q}^0-\bm{q})^{-1},\\
      \bm{m}&=-\left(\bm{Q}^*-p\,\bm{m} \left(\bm{Q}^0+(p-1)\bm{q}\right)^{-1}\bm{m}\right)^{-1}\bm{m}\left(\bm{Q}^0+(p-1)\bm{q}\right)^{-1}.
    \end{split}
\end{equation}
The determinant of $\bm{Q}$ is given by
\begin{equation}
   \ln \det \bm{Q}=(p-1)\ln\det(\bm{Q}^0-\bm{q})+\ln \det\left(\bm{Q}^0+(p-1)\bm{q}\right)+\ln \det\left(\bm{Q}^*-p\,\bm{m}\left(\bm{Q}^0+(p-1)\bm{q}\right)^{-1}\bm{m}\right).
\end{equation}
The above results allow us to rewrite
\begin{equation}
\begin{split}
    J\left(\bm{Q}^*,\bm{Q}^0,\bm{q},\bm{m}\right)=&\int_{\mathbb{R}^k}\dde \bm{y}\int_{\mathbb{R}^k}\mathcal{D}\bm{\xi}\exp\left(-\frac 12 \ln\det (2\pi\bm{Q})\right)\int_{\mathbb{R}^k}\mathcal{D}\bm{z}^*\,P^*_{\rm out}(\bm{y}|\bm{z}^*)\exp\left(-\frac 12 {\bm{z}^*}^\top\bm{\tilde Q}^*\bm{z}^*\right)\\
    &\;\times \left[\int_{\mathbb{R}^k}\mathcal{D}\bm{z}\,P_{\rm out}(\bm{y}|\bm{z})\exp\left(-\frac 12 {\bm{z}}^\top(\bm{\tilde Q}^0-\bm{\tilde q})\bm{z}- {\bm{z}^*}^\top\bm{\tilde m}\bm{z}-\bm{\xi}^\top\bm{\tilde q}^{1/2}\bm{z}\right)\right]^p.
\end{split}
\end{equation}
As in the Bayes-optimal case, in order to consider the $p\rightarrow 0^+$ limit, we can rewrite
\begin{align}
\label{eq:logNBO1}
\begin{split}
    \ln I(\bm{\hat Q}^0,\bm{\hat q},\bm{\hat Q}^*,\bm{\hat m})=&p\int_{\mathbb{R}^k}\mathcal{D}\bm{\xi}\int_{\mathbb{R}^k}\dde \bm{w}^* \,P^*_w(\bm{w}^*)\exp\left(-\frac{1}{2}{\bm{w}^*}^\top\bm{\hat Q}^* \bm{w}^*\right)\\&\times\ln\int_{\mathbb{R}^k}\dde \bm{w}\,P_w(\bm{w})\exp\left(-\frac{1}{2}\bm{w}^\top(\bm{\hat Q}^0+\bm{\hat q})\bm{w}+{\bm{w}^*}^\top \bm{\hat m} \bm{w}+\bm{\xi}^\top \bm{\hat q}^{\frac 12}\bm{w} \right)+o(p),
\end{split}
\\
\begin{split}
    \ln J\left(\bm{Q}^*,\bm{Q}^0,\bm{q},\bm{m}\right)=&p\int_{\mathbb{R}^k}\dde \bm{y}\int_{\mathbb{R}^k}\mathcal{D}\bm{\xi}\exp\left(-\frac 12 \ln\det (2\pi\bm{Q})\right)\int_{\mathbb{R}^k}\mathcal{D}\bm{z}^*\,P^*_{\rm out}(\bm{y}|\bm{z}^*)\exp\left(-\frac 12 {\bm{z}^*}^\top\bm{\tilde Q}^*\bm{z}^*\right)\\
    &\;\times\ln \int_{\mathbb{R}^k}\mathcal{D}\bm{z}\,P_{\rm out}(\bm{y}|\bm{z})\exp\left(-\frac 12 {\bm{z}}^\top(\bm{\tilde Q}^0-\bm{\tilde q})\bm{z}- {\bm{z}^*}^\top\bm{\tilde m}\bm{z}-\bm{\xi}^\top\bm{\tilde q}^{1/2}\bm{z}\right)+o(p).
\end{split}
\label{eq:logNBO2}
\end{align}
\subsection{Computing the free entropy}
At this point, it is straightforward to compute the free entropy from Eq. \eqref{eq:salle_point} by taking the limit $p\rightarrow 0^+$ of Eqs. \eqref{eq:traceBO}-\eqref{eq:logBO1}-\eqref{eq:logBO2} and \eqref{eq:traceNBO}-\eqref{eq:logNBO1}-\eqref{eq:logNBO2}. In the Bayes-optimal case, we obtain:
\begin{align}\label{eq:free_entropyBO}
    \Phi_{\rm BO}(\alpha)&={\rm extr}_{\bm{q},\bm{\hat q}}\left\{-\frac{1}{2}\Tr{\left[\bm{q}\bm{\hat q}\right]} + \Psi_w^*(\bm{\hat q})+\alpha \Psi^*_{\rm out}(\bm{q})\right\},\\
    \begin{split}
    \Psi_w^*(\bm{\hat q})&=\mathbb{E}_{\bm{\xi}}\left[\mathcal{Z}_{w}^*\left(\bm{\hat q}^{1/2}\bm{\xi},\bm{\hat q}\right)\ln\mathcal{Z}_{w}^*\left(\bm{\hat q}^{1/2}\bm{\xi},\bm{\hat q}\right)\right],\\
    \Psi_{\rm out}^*(\bm{q})&=\mathbb{E}_{\bm{y},\bm{\xi}}\left[\mathcal{Z}_{\rm out}^*\left(\bm{y},\bm{q}^{1/2}\bm{\xi},\bm{Q}^*-\bm{q}\right)\ln\mathcal{Z}_{\rm out}^*\left(\bm{y},\bm{q}^{1/2}\bm{\xi},\bm{Q}^*-\bm{q}\right)\right].
\end{split}
\end{align}
In the non Bayes-optimal case, we obtain:
\begin{align}
\label{eq:free_entropyNBO}
    \Phi_{\rm non-BO}(\alpha)={\rm extr}_{\bm{Q}^0,\bm{q},\bm{m},\bm{\hat Q}^0,\bm{\hat q},\bm{\hat m}}\left\{-{\rm Tr}\left[\bm{m}\bm{\hat m}\right]+\frac 12{\rm Tr}\left[\bm{Q}^0\bm{\hat Q}^0\right]+\frac 12{\rm Tr}\left[\bm{q}\bm{\hat q}\right]+\Psi_w(\bm{\hat Q}^0,\bm{\hat q},\bm{\hat m})+\alpha\Psi_{\rm out}(\bm{Q}^*,\bm{Q}^0,\bm{q},\bm{m})\right\},\end{align}
    \begin{align}
    \begin{split}
        \Psi_w(\bm{\hat Q}^0,\bm{\hat q},\bm{\hat m})&=\E_{\bm{\xi}}\left[\mathcal{Z}_{w}^*\left(\bm{\hat m}\bm{\hat q}^{-1/2}\bm{\xi},\bm{\hat m}\bm{\hat q}^{-1}\bm{\hat m}\right)\ln\mathcal{Z}_w\left(\bm{\hat q}^{1/2}\bm{\xi},\bm{\hat Q}^0+\bm{\hat q}\right)\right],\\
        \Psi_{\rm out}(\bm{Q}^*,\bm{Q}^0,\bm{q},\bm{m})&=\E_{\bm{y},\bm{\xi}}\left[\mathcal{Z}_{\rm out}^*\left(\bm{y};\bm{ m}\bm{ q}^{-1/2}\bm{\xi},-\bm{ m}\bm{ q}^{-1}\bm{ m}\right)\ln\mathcal{Z}_{\rm out}\left(\bm{y};\bm{ q}^{1/2}\bm{\xi},\bm{ Q}^0-\bm{ q}\right)\right],
    \end{split}
\end{align}
where we remind that in both cases $\bm{Q}^*$ is fixed given the teacher prior. The above equations make use of a series of auxiliary functions $\mathcal{Z}_{w}^*,\mathcal{Z}_{w},\mathcal{Z}_{\rm out}^*,\mathcal{Z}_{\rm out}$ that simply come from a more compact way of writing Eqs. \eqref{eq:logBO1}-\eqref{eq:logBO2} and \eqref{eq:logNBO1}-\eqref{eq:logNBO2}, i.e.,
\begin{subequations}
\label{eq:part_funcs}
\begin{equation}
     {\cal Z}_{w} ( \bm{\gamma}, \bm{\Lambda} ) = \int_{\mathbb{R}^k} \dde \bm{w}  P_{w} ( \bm{w} )  e^{- \frac{1}{2} \bm{w}^{\top} \bm{\Lambda} \bm{w} + \bm{\gamma}^{\top} \bm{w}   } \;,
\end{equation}
\begin{equation}
    {\cal Z}_{\text{out}} ( \bm{y}; \bm{\omega}, \bm{V}  ) 
      = \int_{\mathbb{R}^k} d \bm{z}   \frac{ e^{- \frac{1}{2} ( \bm{z} - \bm{\omega}  )^{\top} \bm{V}^{-1} ( \bm{z} - \bm{\omega}  )  }}{  \sqrt{  \text{det} (2 \pi \bm{V} )  }} P_{\rm out}(\bm{y}|\bm{z}) \;,
\end{equation}
\end{subequations}
and $\mathcal{Z}_{w}^*,\mathcal{Z}_{\rm out}^*$ are defined in the exact same way provided that the student distributions $P_w,P_{\rm out}$ are replaced by the teacher distributions $P_w^*,P^*_{\rm out}$.
\section{Update equations for the overlap parameters}
We can now compute the update equations for the overlap parameters both in the Bayes and non Bayes-optimal settings by taking the derivatives of Eq. \eqref{eq:free_entropyBO} with respect to $(\bm{q},\bm{\hat q})$ and of Eq. \eqref{eq:free_entropyNBO} with respect to $(\bm{Q}^0,\bm{q},\bm{m},\bm{\hat Q}^0,\bm{\hat q},\bm{\hat m})$, and setting them to zero. In the Bayes-optimal setting the update equations are therefore given by:
\begin{align}
\label{eq:update_eq_NBO}
    \bm{q}&=\E_{\bm{\xi}}\left[\mathcal{Z}^*_w (\bm{\hat q}^{1/2}\bm{\xi},\bm{\hat q})\,  \bm{f}^*_w(\bm{\hat q}^{1/2}\bm{\xi},\bm{\hat q})\,\bm{f}^*_w(\bm{\hat q}^{1/2}\bm{\xi},\bm{\hat q})^\top\right],\\
    \bm{\hat q}&=\alpha\E_{\bm{y},\bm{\xi}}\left[\mathcal{Z}^*_{\rm out} (\bm{y};\bm{ q}^{1/2}\bm{\xi},\bm{Q}^*-\bm{ q})\,  \bm{f}^*_{\rm out}(\bm{y};\bm{ q}^{1/2}\bm{\xi},\bm{Q}^*-\bm{ q})\,\bm{f}^*_{\rm out}(\bm{y};\bm{ q}^{1/2}\bm{\xi},\bm{Q}^*-\bm{ q})^\top\right].
\end{align}
In the non-Bayes optimal setting, we define for simplicity: $\bm{V}=\bm{Q}^0-\bm{q}$, $\bm{\hat V}=\bm{\hat Q}^0+\bm{\hat q}$, and we find
\begin{align}
    \bm{m}&=\mathbb{E}_{\bm \xi}\left[\mathcal{Z}^*_w \, \times  \bm{f}^*_w(\bm{\hat m}\bm{\hat q}^{-1/2}\bm{\xi},\bm{\hat m}^T\bm{\hat q}^{-1}\bm{\hat m})\,\bm{f}_w(\bm{\hat q}^{1/2}\bm{\xi},\bm{\hat V})^\top\right],\\
    \bm{q}&=\mathbb{E}_{\bm \xi}\left[\mathcal{Z}^*_w(\bm{\hat m}\bm{\hat q}^{-1/2}\bm{\xi},\bm{\hat m}^T\bm{\hat q}^{-1}\bm{\hat m})\,\bm{f}_w(\bm{\hat q}^{1/2}\bm{\xi},\bm{\hat V})\bm{f}_w(\bm{\hat q}^{1/2}\bm{\xi},\bm{\hat V})^\top\right],\\
    \bm{V}&=\mathbb{E}_{\bm \xi}\left[\mathcal{Z}^*_w(\bm{\hat m}\bm{\hat q}^{-1/2}\bm{\xi},\bm{\hat m}^T\bm{\hat q}^{-1}\bm{\hat m})\partial_{\bm{\gamma}}\bm{f}_w(\bm{\hat q}^{1/2}\bm{\xi},\bm{\hat V})\right],\\
   \bm{ \hat m} &=\alpha\,\mathbb{E}_{\bm{y},\bm{\xi}}\left[\mathcal{Z}_{\rm out}^* \,\bm{f}_{\rm out}^*(\bm{y},\bm{m}\bm{q}^{-1/2}\bm{\xi},\bm{Q}^*-\bm{m}^\top \bm{q}^{-1}\bm{m})\,\bm{f}_{\rm out}(\bm{y},\bm{q}^{1/2}\bm{\xi},\bm{V})^\top \right],\\
   \bm{ \hat q}&=\alpha\,\mathbb{E}_{\bm{y},\bm{\xi}}\left[\mathcal{Z}_{\rm out}^* (\bm{y},\bm{m}\bm{q}^{-1/2}\bm{\xi},\bm{Q}^*-\bm{m}^\top \bm{q}^{-1}\bm{m})\,\bm{f}_{\rm out}(\bm{y},\bm{q}^{1/2}\bm{\xi},\bm{V})\,\bm{f}_{\rm out}(\bm{y},\bm{q}^{1/2}\bm{\xi},\bm{V})^\top \right],\\
    \bm{\hat V} &=-\alpha\,\mathbb{E}_{\bm{y},\bm{\xi}}\left[\mathcal{Z}_{\rm out}^* (\bm{y},\bm{m}\bm{q}^{-1/2}\bm{\xi},\bm{Q}^*-\bm{m}^\top \bm{q}^{-1}\bm{m})\partial_w \bm{f}_{\rm out}(\bm{y},\bm{q}^{1/2}\bm{\xi},\bm{V})\right],
\end{align}
where in both settings we have made use of the following definitions.
\subsection{Definitions of the update functions}
\label{app:denoising_fcts}
For $\bm{w} \in \bR^k$, let
\begin{subequations}
\label{eq:den_func}
\begin{equation}
   Q_{w} ( \bm{w}; \bm{\gamma}, \bm{\Lambda}  ) \equiv \frac{ P_{w} ( \bm{w} )}{ {\cal Z}_{w} ( \bm{\gamma}, \bm{\Lambda} ) } e^{- \frac{1}{2} \bm{w}^{\top} \bm{\Lambda} \bm{w} + \bm{\gamma}^{\top} \bm{w}   }  \;,
\end{equation}
with
\begin{equation}
   \bm{f}_{w} ( \bm{\gamma}, \bm{\Lambda}  ) \equiv \partial_{\bm{\gamma}} \log  {\cal Z}_{w} ( \bm{\gamma}, \bm{\Lambda} )=\E_{Q_w}\left[\bm{w}\right]  \;,
\end{equation}
and for $\bm{z}\in\mathbb{R}^k$, let 
\begin{equation}
    Q_{\text{out}} ( \bm{z}; \bm{y}, \bm{\omega}, \bm{V}  ) \equiv \frac{ P_{\text{out}} ( \bm{y} | \bm{w} )}{ {\cal Z}_{\text{out}} ( \bm{y},  \bm{\omega}, \bm{V} ) } \frac{ e^{- \frac{1}{2} ( \bm{z} - \bm{\omega}  )^{\top} \bm{V}^{-1} ( \bm{z} - \bm{\omega}  )  }}{  \sqrt{  \text{det} (2 \pi \bm{V} )  }} \;,
\end{equation}
with
\begin{equation}
\label{eq:f_out}
    \bm{f}_{\text{out}} ( \bm{y}, \bm{\omega}, \bm{V}  ) \equiv \partial_{\bm{\omega}} \log 
{\cal Z}_{\text{out}} ( \bm{y}, \bm{\omega}, \bm{V}  ) = \bm{V}^{-1} \E_{Q_{\text{out}}} [ \bm{z} - \bm{\omega}  ]  \;,
\end{equation}
\end{subequations}
where the definitions of $\bm{f}^*_w,\bm{f}^*_{\rm out}$ are identical, provided that $P_w,P_{\rm out}$ are replaced by $P^*_w,P^*_{\rm out}$. The functions $ {\cal Z}_{w}$ and ${\cal Z}_{\text{out}}$ are given by Eqs.~\eqref{eq:part_funcs}.\\
The explicit expressions of the auxiliary functions depend on the choice of the teacher and student distributions. We evaluate these expressions for the special cases under consideration in the following sections.
\subsection{Bayes-optimal update functions}
In this section, we evaluate the Bayes-optimal update functions. We consider directly the expressions obtained after performing the mapping described in Appendix \ref{appendix:mapping}.
\subsubsection{Gaussian prior terms with the dimensional reduction \ref{appendix:mapping}}
In the case of Gaussian teacher prior, it is straightforward to notice that, after the application of the mapping \ref{appendix:mapping}, the prior over the weights is still Gaussian with covariance $\tilde{\bm{\Sigma}}= \bm{I}_{k-1}+\bm{1}_{k-1}\bm{1}_{k-1}^\top$:
\begin{subequations}
\begin{equation}
    \begin{split}
       {\cal Z}^*_{w} ( \bm{\gamma}, \bm{\Lambda} ) &=\int_{\mathbb{R}^{k-1}} \frac{d \bm{w}}{\sqrt{(2\pi)^{k-1}\det(\bm{\Tilde \Sigma})}} 
    \exp\left[  - \frac{1}{2} \bm{w}^{\top} \left(  \bm{\Tilde \Sigma}{}^{-1} +\bm{\Lambda}  \right) \bm{w} + \bm{\gamma}^{\top} \bm{w}   \right] \\&= \frac{1}{\sqrt{\det(\tilde{\bm{\Sigma}})
    \det (  \tilde{\bm{\Sigma}}{}^{-1} + \bm{\Lambda}  )}}\exp\left[\frac{1}{2}\bm{\gamma}^\top \left( \tilde{\bm{\Sigma}}{}^{-1}+\bm{\Lambda}  \right)^{-1}\bm{\gamma}\right] \;,
       \end{split}
\end{equation}
leading to
\begin{equation}
   \bm{f}_{w}^* ( \bm{\gamma}, \bm{\Lambda}  ) = \partial_{\bm{\gamma}} \log  {\cal Z}^*_{w} ( \bm{\gamma}, \bm{\Lambda} ) = \left( \tilde{\bm{\Sigma}}{}^{-1}+\bm{\Lambda} \right)^{-1} \bm{\gamma} \;,
\end{equation}
\begin{equation}
   \partial_{\bm \gamma}\bm{f}_{w}^* ( \bm{\gamma}, \bm{\Lambda}  ) = \left( \tilde{\bm{\Sigma}}{}^{-1} + \bm{\Lambda} \right)^{-1} \;.
\end{equation}
\end{subequations}
For $k=3$, the reduced covariance matrix is given by
\begin{equation}
 \tilde{\bm{\Sigma}} = 
 \begin{bmatrix}
2 & 1 \\
1 & 2 
\end{bmatrix} \;.
 \end{equation}
\subsubsection{Rademacher prior terms with the dimensional reduction \ref{appendix:mapping}}
Considering $k = 3$, the reduced prior given by Eq.~\eqref{eq:rad_prior_k1} becomes
\begin{equation} 
\label{eq:rad_prior_reduced__}
\begin{split}
      P_{\tilde{w}} ( \tilde{w}_1^*,  \tilde{w}_2^*) &=  \frac{1}{2^3} 
  \left[  2 \delta( \tilde{w}_1^*) \delta( \tilde{w}_2^*) +  \delta( \tilde{w}_1^*)\delta( \tilde{w}_2^* +2) +  \delta( \tilde{w}_1^* +2) \delta( \tilde{w}_2^* ) +\delta( \tilde{w}_1^* ) \delta( \tilde{w}_2^* -2) \right. \\
  &  \left.    +  \delta( \tilde{w}_1^*  -2) \delta( \tilde{w}_2^* )+   \delta( \tilde{w}_1^* +2 ) \delta( \tilde{w}_2^* +2 )  + \delta( \tilde{w}_1^* -2) \delta( \tilde{w}_2^* -2)     \right] \;.
\end{split}
\end{equation}

The denoising functions for this case are computed numerically, via Monte Carlo sampling of the distribution given Eq.~\eqref{eq:rad_prior_reduced__}.

\subsubsection{Output terms with the dimensional reduction \ref{appendix:mapping}}
Considering directly the mapping to dimension $k-1$, we can write the Bayes-optimal model distribution as
\begin{equation}
    P^*_{\rm out}(\bm{y}|\bm{z})=\sum_{l=1}^{k-1}\delta_{\bm{y},\bm{e}_l}\Theta(z_l)\prod_{h\neq l,h=1}^{k-1}\,\Theta(z_l-z_h)+\delta_{\bm{y},\bm{e}_k}\prod_{l=1}^{k-1}\Theta(-z_l).
\end{equation}
Therefore, the auxiliary functions $\mathcal{Z}^*_{\rm out}$, and so on, are composed by $k-1$ contributions according to the membership of the label $\bm{y}$ in the argument. For instance, in the case $k=3$, we have
\begin{equation}
\begin{split}
    {\cal Z}^*_{\text{out}} ( \bm{y}; \bm{\omega}, \bm{V}  ) 
      &= \delta_{\bm{y},\bm{e}_1}\int_0^{+\infty} \dde z_1\int_{-\infty}^{z_1}   \dde z_2\frac{ e^{- \frac{1}{2} ( \bm{z} - \bm{\omega}  )^{\top} \bm{V}^{-1} ( \bm{z} - \bm{\omega}  )  }}{  \sqrt{  \text{det} (2 \pi \bm{V} )  }} + \delta_{\bm{y},\bm{e}_2}\int_0^{+\infty} \dde z_2\int_{-\infty}^{z_2}   \dde z_1\frac{ e^{- \frac{1}{2} ( \bm{z} - \bm{\omega}  )^{\top} \bm{V}^{-1} ( \bm{z} - \bm{\omega}  )  }}{  \sqrt{  \text{det} (2 \pi \bm{V} )  }}\\
      &+\delta_{\bm{y},\bm{e}_3}\int_{-\infty}^0 \dde z_1\int_{-\infty}^{0}   \dde z_2\frac{ e^{- \frac{1}{2} ( \bm{z} - \bm{\omega}  )^{\top} \bm{V}^{-1} ( \bm{z} - \bm{\omega}  )  }}{  \sqrt{  \text{det} (2 \pi \bm{V} )  }}.
      \end{split}
\end{equation}
Similarly, for $\bm{f}^*_{\rm out}$ we need to change the integration bounds in order to take into account all the possibilities for the label. For each term, we can only compute analytically the inner integral, while we have to estimate the outer ones via Monte Carlo sampling. Therefore, applying the mapping in \ref{appendix:mapping} is useful in order to reduce the number of integrals to be performed numerically and speed up the whole procedure.
\subsection{ERM update functions}
The update equations for ERM can be derived as a special case of the non Bayes-optimal Eqs. \eqref{eq:update_eq_NBO} and so on. In particular, this can be seen by rewriting the solution of the optimization problem as the \emph{ground state} of the following measure
\begin{equation}
\begin{split}
    P_{\beta}(\bm{W}|\bm{X},\bm{Y})&=\frac{1}{Z_{d}(\beta)}\exp\left(-\beta \,r_\lambda(\bm{W})\right)\exp\left(-\beta \mathcal{L}(\bm{W};\bm{X},\bm{Y})\right)\\
    &=\frac{1}{Z_{d}(\beta)}\prod_{l=1}^k\exp\left(-\frac{\beta\lambda}{2}\|\bm{w}_l\|_2^2\right)\prod_{\mu=1}^n\exp\left(-\beta \ell\left(\bm{W}^\top\bm{x}_\mu,\bm{y}_\mu\right)\right)
\end{split}
\end{equation}
i.e., the solution in the limit $\beta \rightarrow\infty$. Therefore, we can express the prior and model distributions of a student learning via ERM as
\begin{align}
        P_w(\bm{w})\propto \exp\left(-\frac{\beta\lambda}{2}\|\bm{w}\|_2^2\right)\;, && P_{\rm out}(\bm{y}|\bm{W}^\top\bm{x})\propto \exp\left(-\beta \ell\left(\bm{W}^\top\bm{x},\bm{y}\right)\right).
\end{align}
\subsubsection{Prior terms with the dimensional reduction \ref{appendix:mapping}}
The ERM ridge-regularization prior can therefore be seen as i.i.d. Gaussian-distributed with variance $1/\beta\lambda$. This means that, appling the mapping \ref{appendix:mapping}, we have
\begin{align}
P_w(\bm{w}) = \frac{1}{(2 \pi)^{(k-1)/2} \sqrt{\det(\bm{C}/(\beta \lambda)}}\exp\left(-\frac{\beta\lambda}{2}  \bm{w}^\top \bm{C}^{-1}\bm{w}\right),
\end{align}
where again $\bm{C}$ is the prior covariance in the reduced setting, i.e. $\bm{C}= [2,1; 1,2]$ for $k=3$.
Let we rescale $\bm{\gamma}\leftarrow\beta\bm{\gamma}$ and $\bm{\Lambda}\leftarrow\beta\bm{\Lambda}$. We will see that this would correspond to the rescaling: $\bm{\hat{q}}\leftarrow\beta^2\bm{\hat q}$ and $\bm{\hat V}\leftarrow\beta \bm{\hat V}$. We obtain
\begin{align}
{\cal Z}_{w} ( \bm{\gamma}, \bm{\Lambda} ) &= \int_{\mathbb{R}^{k-1}}\dde\bm{w}\, \frac{e^{-\frac{\beta}{2} \bm{w}^{\top}(\lambda \bm{C^{-1}}+ \bm{\Lambda}) \bm{w} + \beta\bm{\gamma}^\top \bm{w}   } }{(2 \pi)^{(k-1)/2} \sqrt{\det (\bm{C}/\beta \lambda)}} =\frac{1}{\sqrt{\det(\bm{C}/\beta \lambda) \det(\beta\lambda\bm{C^{-1}}+\beta\bm{\Lambda})}}\exp\left(\frac{\beta}{2}\bm{\gamma}^\top (\lambda\bm{C^{-1}}+\bm{\Lambda})^{-1}\bm{\gamma}\right)\;,\\
\bm{f}_{w} ( \bm{\gamma}, \bm{\Lambda}  )&=\beta\,(\lambda \bm{C^{-1}}+\bm{\Lambda})^{-1} \bm{\gamma},\\
\partial_{\bm{\gamma}}\bm{f}_{w} ( \bm{\gamma}, \bm{\Lambda}  )&=\beta (\lambda\bm{C^{-1} }+\bm{\Lambda})^{-1}.
\end{align}
Substituting the expressions above in Eqs. \eqref{eq:SE}, we find
\begin{align}
\bm{m}&=\frac{1}{\sqrt{\det(C)\det(C^{-1}+\bm{\hat m}\bm{ \hat q}^{-1}\bm{\hat m})}\sqrt{\det(I-\bm{\hat q}^{-1/2}\bm{\hat m}(C^{-1}+\bm{\hat m}\bm{\hat q}^{-1}\bm{\hat m})^{-1}\bm{\hat m}\bm{\hat q}^{-1/2})}}\\
&\times( C^{-1} +\bm{\hat m}\bm{\hat q}^{-1}\bm{\hat m})^{-1} \hat m \hat q^{-1/2}\left(\bm{I}-\bm{\hat q}^{-1/2}\bm{\hat m}(C^{-1}+\bm{\hat m}\bm{\hat q}^{-1}\bm{\hat m})^{-1} \hat m \hat q^{-1/2}\right)^{-1}\bm{\hat q^{1/2}}(\lambda C^{-1}+\bm{\hat V})^{-1}\;,\\
\bm{q}&=\frac{1}{\sqrt{\det(C)\det(C^{-1}+\bm{\hat m}\bm{ \hat q}^{-1}\bm{\hat m})}\sqrt{\det(I-\bm{\hat q}^{-1/2}\bm{\hat m}(C^{-1}+\bm{\hat m}\bm{\hat q}^{-1}\bm{\hat m})^{-1}\bm{\hat m}\bm{\hat q}^{-1/2})}}\\
&\times (\lambda C^{-1}+\bm{\hat V})^{-1} \hat q^{1/2}\left(\bm{I}-\bm{\hat q}^{-1/2}\bm{\hat m}(C^{-1}+\bm{\hat m}\bm{\hat q}^{-1}\bm{\hat m})^{-1} \hat m \hat q^{-1/2}\right)^{-1}\hat q^{1/2}(\lambda C^{-1}+\bm{\hat V})^{-1}
\;,\\
\bm{V}&=\frac{(\lambda C^{-1}+\hat V)^{-1}}{\sqrt{ \det(C)\det(C^{-1}+\bm{\hat m}\bm{ \hat q}^{-1}\bm{\hat m})}\sqrt{\det(I-\bm{\hat q}^{-1/2}\bm{\hat m}(C^{-1}+\bm{\hat m}\bm{\hat q}^{-1}\bm{\hat m})^{-1}\bm{\hat m}\bm{\hat q}^{-1/2})}},
\end{align}
where additionally we have rescaled: $\bm{m}\leftarrow \beta \bm{m}$, $\bm{q}\leftarrow\beta^2 \bm{q}$, $\bm{V}\leftarrow\beta\bm{V}$. The equations are now independent of the parameter $\beta$. We will see that the above rescaling is consistent and leads to a set of well-defined equations in the $\beta\rightarrow\infty$ limit.
\subsubsection{Output terms with the dimensional reduction \ref{appendix:mapping}}
We remind that we have performed the rescaling $ \bm{V} \rightarrow \beta^{-1} \bm{V}$. In the $\beta \rightarrow\infty$ imit, the ERM output term $\mathcal{Z}_{\rm out}$ becomes
\begin{equation}
       {\cal Z}_{\text{out}} ( \bm{y}; \bm{\omega}, \bm{V}  ) 
      \propto \sqrt{\beta^{k-1}} \int_{\mathbb{R}^{k-1}} d \bm{z}   \frac{ e^{ -\beta \left[ - \frac{1}{2} ( \bm{z} - \bm{\omega}  )^{\top} \bm{V}^{-1} ( \bm{z} - \bm{\omega}  ) + {\cal L } (\bm{y} , \bm{z}) \right] }}{  \sqrt{  \text{det} (2 \pi \bm{V} )  }}  \; \; \underset{\beta \rightarrow \infty }{\longrightarrow} \; \; 
      \sqrt{\frac{\beta^{k-1}}{\det(2\pi\bm{V})}}\,e^{- \beta {\cal M}_{ \bm{V}  {\cal L}(\bm{y} , \cdot)} ( \bm{\omega}  )} \;. 
\end{equation}
where $ {\cal M }$ is a Moreau envelope associated with the loss ${\cal L}$,
\begin{equation}
    {\cal M}_{ \bm{V}  {\cal L}(\bm{y} , \cdot)} ( \bm{\omega}  ) =  \inf _{ \bm{z} \in \mathbb{R}^{k-1} } \left[   \frac{1}{2} ( \bm{z} - \bm{\omega}  )^{\top} \bm{V}^{-1} ( \bm{z} - \bm{\omega}  ) +  {\cal L } (\bm{y} , \bm{z}) \right]
\end{equation}
From Eq.(\ref{eq:f_out}), we have
\begin{equation}
    \bm{f}_{\text{out}} ( \bm{y}, \bm{\omega}, \bm{V}  ) = -\beta \partial_{\bm{\omega}}  {\cal M}_{ \bm{V}  {\cal L}(\bm{y} , \cdot)} ( \bm{\omega}  )  \;,
\end{equation}
which can be obtained using the proximal operator
\begin{equation}
    \text{prox}_{ \bm{V}  {\cal L}(\bm{y} , \cdot)} ( \bm{\omega}  ) = \argmin_{\bm{z} \in \mathbb{R}^{k-1}} \left[  \frac{1}{2} ( \bm{z} - \bm{\omega}  )^{\top} \bm{V}^{-1} ( \bm{z} - \bm{\omega}  ) +  {\cal L } (\bm{y} , \bm{z})  \right] \;,
\end{equation}
The envelope theorem, $  {\cal M}^{'}_{ \bm{\Sigma}  f } ( \bm{x}  )   =  \bm{\Sigma}^{-1} \left( \bm{x} -  \text{prox}_{ \bm{\Sigma}  f } ( \bm{x}  ) \right)$ leads to 
\begin{align}
     \bm{f}_{\text{out}} ( \bm{y}, \bm{\omega}, \bm{V}  )  &= - \beta  \bm{V}^{-1} \left( \bm{\omega} -  \text{prox}_{ \bm{V}   {\cal L}(\bm{y} , \cdot) } ( \bm{\omega}  ) \right) \;,\\
   \partial_{\bm{\omega}}  \bm{f}_{\text{out}} ( \bm{y}, \bm{\omega}, \bm{V}  ) &=  - \beta  \bm{V}^{-1} \left( \bm{I} - \partial_{\bm{\omega}}   \text{prox}_{ \bm{V}   {\cal L}(\bm{y} , \cdot) } ( \bm{\omega}  ) \right),\\
     \partial_{\bm{\omega}}   \text{prox}_{ \bm{V}   {\cal L}(\bm{y} , \cdot) } ( \bm{\omega}  )
     &=  \partial_{\bm{\omega}}  \bm{z}^* ( \bm{\omega}  ) = \left(\bm{V}^{-1}+\partial^2_{\bm{z}}\mathcal{L}\right)^{-1},
\end{align}
consistently with the rescaling previously adopted, which leads to the final equations Eqs. \eqref{eq:SE} in the main text, holding in the limit $\beta\rightarrow\infty$.
\paragraph{Special case: square loss ---} The proximal operator for the square loss can be computed analytically:
\begin{align}
    \text{prox}^{\rm SL}_{ \bm{V}  {\cal L}(\bm{y} , \cdot)} ( \bm{\omega}  ) &=(\bm{I}+\bm{V})^{-1}(\bm{w}+\bm{V}\bm{y}),\\
    \partial_{\bm{w}}\text{prox}^{\rm SL}_{ \bm{V}  {\cal L}(\bm{y} , \cdot)} ( \bm{\omega}  )&=(\bm{I}+\bm{V})^{-1}.
\end{align}

\section{Proof of the main theorem}
\label{app:proof}
In this section we prove the main theorem in a slightly more general setup than what is presented in the main part of the paper. We start by reminding the learning problem defining the ensemble of estimators with a few auxiliary notations, so that this part is self contained. The exact match with the replica prediction will be given at the end of the proof.
\subsection{The learning problem}
We start by reminding the definition of the problem. Consider the following generative model
\begin{equation}
\bm{Y} = \phi_{\rm out}\left(\frac{1}{\sqrt{d}}\bm{X}\bm{W}^{*}\right)
\end{equation}
where $\bm{Y} \in \mathbb{R}^{n \times k},\bX \sim \mathcal{N}(0,1) \in \mathbb{R}^{n \times d}$ and $\bm{W}^{*} \in \mathbb{R}^{d\times k}$.
The goal is to try to learn an estimator of $\bm{W}^{*}$ using a generalised linear model defined by the optimisation problem
\begin{equation}
    \label{eq:student}
    \hat{\bm{W}} \in \underset{{\bm{W}\in \mathbb{R}^{d \times k}}}{\arg\min}\, \mathcal{L}\left(\bm{Y},\frac{1}{\sqrt{d}}\bm{X}\bm{W}\right)+r(\bm{W})
\end{equation}
where $\mathcal{L}, r$ are convex functions, and we omit the dependence of the regularisation $r$ on the parameter $\lambda$ for simplicity.
We wish to determine the asymptotic properties of the estimator $\hat{\bm{W}}$ in the limit where $n,d \to \infty$ with fixed ratios $\alpha = n/d$. We now list the necessary assumptions for our main theorem to hold.
\paragraph{Assumptions --}
\begin{itemize}
    \item the functions $\mathcal{L},r$ are proper, closed, lower-semicontinuous, convex functions. The loss function $\mathcal{L}$ is differentiable and pseudo-Lipschitz of order 2 in both its arguments. We assume additionally that the regularisation $r$ is strongly convex, differentiable and pseudo-Lipschitz of order 2.
    \item the dimensions $n,d$ grow linearly with finite ratios $\alpha = n/d$, and the number of classes $k$ is kept constant.
    \item the lines of the ground truth matrix $\bm{W}^{*} \in \mathbb{R}^{d \times k }$ are sampled i.i.d. from a sub-Gaussian probability distribution in $\mathbb{R}^{k}$.
\end{itemize}
\subsection{Reduction to an AMP iteration}
We start by reformulating the optimisation problem \eqref{eq:student} in order to be able to solve it with an AMP iteration. In particular
it is useful to separate the design matrix $\bm{X}$ in two contributions : one aligned with the ground truth $\bm{W}^{*}$ and one independent on the teacher $\bm{Y}$. To do so we condition $\bm{X}$ on the teacher input $\bm{X}\bm{W}^{*}$ such that 
\begin{align}
    \bm{X} &= \mathbb{E}\left[\bm{X} \vert \bm{Y}\right]+\bm{X}-\mathbb{E}\left[\bm{X} \vert \bm{Y}\right] \\
    &= \mathbb{E}\left[\bm{X} \vert \bm{X}\bm{W}^{*}\right]+\bm{X}-\mathbb{E}\left[\bm{X} \vert \bm{X}\bm{W}^{*}\right] \\
    &= \bm{X}\mathbf{P}_{\bm{W}^{*}}+\tilde{\bm{X}}\mathbf{P}^{\perp}_{\bm{W}^{*}}
\end{align}
where $\tilde{\bm{X}}$ is an independent copy of the design matrix $\bm{X}$, $\mathbf{P}_{\bm{W}^{*}}$ denotes the orthogonal projection on the subspace spanned by the columns of $\bm{W}^{*}$ and $\mathbf{P}^{\perp}_{\bm{W}^{*}} = \bm{I}_{d}-\mathbf{P}_{\bm{W}^{*}}$. Furthermore, since we assume that 
$n,d$ are arbitrarily large and that $k$ remains finite for each instance of the problem, the matrix $\bm{W}^{*}$ has full column rank and the projector
$\mathbf{P}_{\bm{W}^{*}} = \bm{W}^{*}\left((\bm{W}^{*})^{\top}\bm{W}^{*}\right)^{-1}(\bm{W}^{*})^{\top}$ is always well-defined. We can then rewrite the original problem as
\begin{align}
\hat{\bm{W}} &\in \underset{{\bm{W}\in \mathbb{R}^{d \times k}}}{\arg\min} \, \mathcal{L}\left(\bm{Y},\frac{1}{\sqrt{d}}\left(\bm{X}\mathbf{P}_{\bm{W}^{*}}+\tilde{\bm{X}}\mathbf{P}^{\perp}_{\bm{W}^{*}}\right)\bm{W}\right)+r(\bm{W})
\end{align}
The quantity $\bm{X}\bm{W}^{*}$ is a $\mathbb{R}^{n \times k}$ Gaussian matrix with covariance $(\mathbf{W}^{*})^{\top}\mathbf{W}^{*}$, and can be represented as $\bm{X}\bm{W}^{*} = \bm{S}((\bm{W}^{*})^{\top}\bm{W}^{*})^{1/2}$ where 
$\bm{S}$ is an $n \times k$ random matrix with i.i.d. standard normal elements. We then have 
\begin{align}
    \frac{1}{\sqrt{d}}\bm{X}\mathbf{P}_{\bm{W}^{*}} &= \frac{1}{\sqrt{d}}\bm{W}^{*}\left((\bm{W}^{*})^{\top}\bm{W}^{*}\right)^{-1}(\bm{W}^{*})^{\top}\bm{W} \\
    &=\frac{1}{\sqrt{d}}\bm{S}\sqrt{d}\rho^{1/2}\frac{1}{d}\rho^{-1}d\bm{m}^{\top} \\
    &= \bm{S}\rho^{-1/2}\bm{m}^{\top}
\end{align}
where we introduced the order parameter $\bm{m} = \frac{1}{d}\hat{\bm{W}}^{\top}\bm{W}^{*} \in \mathbb{R}^{k \times k}$ and the quantity $\rho = \frac{1}{d}(\bm{W}^{*})^{\top}\mathbf{W}^{*} \in \mathbb{R}^{k \times k}$. Note that $\bm{Y} = \bm{S}\rho^{1/2}$, and 
\begin{align}
\label{eq:inter_hm}
\hat{\bm{W}} &\in \underset{{\bm{W}\in \mathbb{R}^{d \times k}}}{\arg\min} \, \mathcal{L}\left(\bm{Y},\bm{S}\rho^{-1/2}\bm{m}^{\top}+\frac{1}{\sqrt{d}}\tilde{\bm{X}}\mathbf{P}^{\perp}_{\bm{W}^{*}}\bm{W}\right)+r(\bm{W})
\end{align}
We may then 
rewrite the optimisation problem Eq.~\eqref{student1} as an equivalent problem under constraint on the definition of $\mathbf{m}$ leading to the Lagrangian formulation
\begin{align}
    \inf_{\mathbf{m},\mathbf{W}}\sup_{\hat{\bm{m}}} \mathcal{L}\left(\bm{Y},\bm{S}\rho^{-1/2}\bm{m}^{\top}+\frac{1}{\sqrt{d}}\tilde{\bm{X}}\mathbf{P}^{\perp}_{\bm{W}^{*}}\bm{W}\right)+r(\bm{W}^{*}\rho^{-1}\bm{m}^{\top}+\mathbf{P}^{\perp}_{\bm{W}^{*}}\bm{W})+\mbox{Tr}\left(\hat{\bm{m}}^\top\left(d\bm{m}-\hat{\bm{W}}^{\top}\bm{W}^{*}\right)\right)
\end{align}
Letting $\bm{U} = \mathbf{P}^{\perp}_{\bm{W}^{*}}\bm{W}$ such that $\bm{W} = \bm{W}^{*}\rho^{-1}\bm{m}^{\top}+\bm{U}$, the problem becomes
\begin{align}
    \label{eq:AMP_target}
    \inf_{\bm{m},\bm{U}}\sup_{\hat{\bm{m}}} \mathcal{L}\left(\bm{Y},\bm{S}\rho^{-1/2}\bm{m}^{\top}+\frac{1}{\sqrt{d}}\tilde{\bm{X}}\bm{U}\right)+r(\bm{W}^{*}\rho^{-1}\bm{m}^{\top}+\bm{U})-\mbox{Tr}\left(\hat{\bm{m}}^\top\bm{U}^{\top}\bm{W}^{*}\right)
\end{align}
where the initial constraint on $\bm{m}$ automatically enforces the orthogonality constraint on $\bm{U}$ w.r.t. $\bm{W}^{*}$. The following lemma then characterises the feasibility sets of $\bm{m}, \hat{\bm{m}}, \bm{U}$.
\begin{lemma}
\label{lemma_compact}
Consider the optimisation problem Eq.~\eqref{eq:AMP_target}. Then there exist constants $C_{\bm{U}}, C_{\bm{m}}, C_{\hat{\bm{m}}}$ such that
\begin{equation}
    \frac{1}{\sqrt{d}}\norm{\bm{U}}_{F}\leqslant C_{\bm{U}}, \quad \norm{\bm{m}}_{F} \leqslant C_{\bm{m}}, \quad \norm{\hat{\bm{m}}}_{F}\leqslant C_{\hat{\bm{m}}}
\end{equation}
with high probability as $n,d\to \infty$. 
\end{lemma}
\begin{proof}
Consider the optimisation problem defining $\hat{\bm{W}}$
\begin{align}
    \hat{\bm{W}} \in \underset{{\bm{W}\in \mathbb{R}^{d \times k}}}{\arg\min} \thickspace \mathcal{L}(\mathbf{Y},\bm{X}\bm{W})+r(\bm{W})
\end{align}
From the strong convexity assumption on $r$, there exists a strictly positive constant $\lambda_{2}$ such that the function $\tilde{r}(\bm{W}) := r(\bm{W})-\frac{\lambda_{2}}{2}\norm{\bm{W}}_{F}^{2}$ is convex (and proper, closed, lower semi-continuous). We can then rewrite the optimisation problem as
\begin{align}
    \hat{\bm{W}} \in \underset{{\bm{W}\in \mathbb{R}^{d \times k}}}{\arg\min} \mathcal{L}(\bm{Y},\bm{X}\bm{W})+\tilde{r}(\bm{W})+\frac{\lambda_{2}}{2}\norm{\bm{W}}_{F}^{2}
\end{align}
which, owing to the convexity of the cost function, verifies
\begin{equation}
    \frac{1}{d}\left(\mathcal{L}(\bm{Y},\bm{X}\hat{\bm{W}})+\tilde{r}(\hat{\bm{W}})+\frac{\lambda_{2}}{2}\norm{\hat{\bm{W}}}_{F}^{2}\right) \leqslant \frac{1}{d}\left(\mathcal{L}(\bm{Y})+\tilde{r}(\bm{0})\right)
\end{equation}
The functions $\mathcal{L}$ and $\tilde{r}$ are proper, thus their sum is bounded below for any value of their arguments and we may write 
\begin{equation}
    \frac{1}{d}\frac{\lambda_{2}}{2}\norm{\hat{\bm{W}}}_{F}^{2} \leqslant \frac{1}{d}\left(\mathcal{L}(\bm{Y})+\tilde{r}(\bm{0})\right)
\end{equation}
The pseudo-Lipschitz assumption on $\mathcal{L}$ and $r$ then implies that there exist positive constants $C_{\mathcal{L}}$ and $C_{\tilde{r}}$ such that 
\begin{align}
    \frac{1}{d}\frac{\lambda_{2}}{2}\norm{\hat{\bm{W}}}_{F}^{2} &\leqslant \frac{1}{d}\left(C_{\mathcal{L}}\left(1+\norm{\bm{Y}}^{2}_{2}\right)\right)+C_{\tilde{r}}
\end{align}
where, on the right hand side, the term $\norm{\bm{Y}}_{F}^{2} /d= {\alpha}\norm{\bm{Y}}_{F}^{2}/n$ is bounded since the labels are in $\left\{-1,+1\right\}$ and $\alpha$ is finite.
 Now using the definition of $\bm{U}$
\begin{align}
    \frac{1}{d}\norm{\bm{U}}_{F}^{2} &= \frac{1}{d}\norm{\mathbf{P}_{\tilde{\bm{W}^{*}}}^{\perp}\hat{\bm{W}}}_{F}^{2} \\
    & \leqslant \norm{\mathbf{P}_{\tilde{\bm{W}^{*}}}^{\perp}}^{2}_{op}\frac{1}{d}\norm{\hat{\bm{W}}}_{F}
\end{align} 
where the singular values of $\mathbf{P}_{\tilde{\bm{W}^{*}}}^{\perp}$ are bounded with probability one. Therefore there exists a constant $C_{\bm{U}}$ such that $\norm{\bm{U}} /\sqrt{d}\leqslant C_{\bm{U}}$. Then, by definition of $\bm{m}$ and the Cauchy-Schwarz inequality
\begin{align}
    \norm{\bm{m}}_{F}^{2} &\leqslant \frac{1}{d}\norm{\bm{c}}_{2}^{2}\frac{1}{d}\norm{\bm{W}}_{F}^{2} \\
    &\leqslant \frac{1}{d}\norm{\bm{W}^{*}}_{2}^{2}\frac{1}{d}\norm{\hat{\bm{W}}}_{F}^{2} 
\end{align}
 By assumption the columns of $\bm{W}^{*}$ are sampled from sub-Gaussian distributions, thus, using Bernstein's inequality for sub-exponential random variables there exists a positive constant $C_{\bm{W}^{*}}$ such that, with high probability as $n,d\to +\infty$, $\norm{\bm{W}^{*}}_{F}^{2} \leqslant C_{\bm{W}^{*}}$. Combining this with the result on $\hat{\bm{W}}$, there exists a positive constant $C_{\bm{m}}$ such that $\norm{\bm{m}}_{F} \leqslant C_{\bm{m}}$ with high probability as $n,d \to +\infty$. We finally turn to $\hat{\bm{m}}$. The optimality condition for $\bm{m}$ in problem Eq.\eqref{eq:inter_hm} gives
\begin{equation}
    \hat{\bm{m}} = -\frac{1}{\sqrt{d}}\rho^{-1/2}\mathbf{S}^\top\partial \mathcal{L}\left(\bm{Y},\frac{\bm{S}\bm{m}^\top}{\sqrt{\rho}}+\frac{1}{\sqrt{d}}\tilde{\bm{X}}\bm{C}^{1/2}\bm{W}^{*}\right)
\end{equation}
The pseudo-Lipschtiz assumption on $\mathcal{L}$ implies that we can find a constant $C_{\partial \mathcal{L}}$ such that 
\begin{equation}
    \norm{\hat{\bm{m}}}_{2}^{2} \leqslant \frac{1}{d}\norm{\rho^{-1}}_{F}\norm{\bm{S}}_{F}^{2}C_{\partial \mathcal{L}}\left(1+\frac{1}{d}\norm{\bm{Y}}_{F}^{2}+\frac{1}{d}\norm{\frac{\bm{S}\bm{m}^\top}{\sqrt{\rho}}+\frac{1}{\sqrt{d}}\tilde{\bm{X}}\bm{W}^{*}}_{F}^{2}\right)
\end{equation}
All quantities in the right hand side of the last inequality have bounded scaled norm with high probability, except the operator norm of the random matrix $\tilde{\bm{X}}$ which has i.i.d. $\mathcal{N}(0,1/d)$ elements. Existing results in random matrix theory \cite{vershynin2018high} ensure this operator norm is bounded with high as $n,d \to +\infty$, which concludes the proof of this lemma.
\end{proof}
The optimisation problem Eq.~\eqref{eq:AMP_target} is convex and feasible. Furthermore, we may reduce the feasibility sets of $\bm{m},\hat{\bm{m}}$ to compact spaces, and the function of $\bm{U}$ is coercive and thus has bounded lower level sets. Strong duality then implies we can invert the order of minimisation to obtain the equivalent problem
\begin{align}
    \label{eq:AMP_target_inv}
   \inf_{\bm{m}} \sup_{\hat{\bm{m}}} \inf_{\bm{U}} \mathcal{L}\left(\bm{Y},\bm{S}\rho^{-1/2}\bm{m}^{\top}+\frac{1}{\sqrt{d}}\tilde{\bm{X}}\bm{U}\right)+r(\bm{W}^{*}\rho^{-1}\bm{m}^{\top}+\bm{U})-\mbox{Tr}\left(\hat{\bm{m}}^\top\bm{U}^{\top}\bm{W}^{*}\right)
\end{align}
and study the optimisation problem in $\bm{U}$ at fixed $\bm{m},\hat{\bm{m}}$:
\begin{align}
    \label{student1}
    \inf_{\bm{U} \in \mathbb{R}^{d \times k}} \tilde{\mathcal{L}}(\frac{1}{\sqrt{d}}\tilde{\bm{X}}\bm{U})+\tilde{r}(\bm{U})
\end{align}
where we defined the functions
\begin{align}
    \tilde{\mathcal{L}} : \mathbb{R}^{n \times k} &\to \mathbb{R} \\
    \frac{1}{\sqrt{d}}\tilde{\bm{X}}\bm{U} &\to \mathcal{L}\left(\bm{Y},\bm{S}\rho^{-1/2}\bm{m}^{\top}+\frac{1}{\sqrt{d}}\tilde{\bm{X}}\bm{U}\right) \\
    \tilde{r} : \mathbb{R}^{d \times k} &\to \mathbb{R} \\
    \bm{U} &\to r(\bm{W}^{*}\rho^{-1}\bm{m}^{\top}+\bm{U})-\mbox{Tr}\left(\hat{\bm{m}}^\top\bm{U}^{\top}\bm{W}^{*}\right)
\end{align}
and the random matrix $\tilde{\bm{X}}$ with i.i.d. $\mathcal{N}(0,1)$ elements is independent from all other random quantities in the problem. The asymptotic properties of the unique solution to this optimisation problem can now be studied with a non-separable, matrix-valued approximate message passing iteration. The AMP iteration solving problem Eq.~\eqref{student1} is given in the following lemma
\begin{lemma}
Consider the following AMP iteration
\begin{align}
    \label{eq:AMP1}
	&\hspace{1cm} \bu^{t+1} = \tilde{\bX}^{\top}\bh_{t}(\bv^{t})-\be_{t}(\bu^{t})\langle \bh_{t}'\rangle^\top \\
	&\hspace{1cm} \bv^{t} = \tilde{\bX}\be_{t}(\bu^{t})-\bh_{t-1}(\bv^{t-1})\langle \be_{t}'\rangle^\top
\end{align}
where for any $t \in \mathbb{N}$
\begin{align}
    &\bh_{t}(\mathbf{v}^{t}) = \left(\bbR_{\mathcal{L}(\bm{Y},.),\bm{S}^{t}}(\bm{S}\rho^{-1/2}\bm{m}^{\top}+\bm{v}^{t})-\left(\bm{S}\rho^{-1/2}\bm{m}^{\top}+\bm{v}^{t}\right)\right)(\bm{S}^{t})^{-1} \\
    &\be_{t}(\bm{u}^{t}) = \bbR_{r(.),\hat{\bm{S}}^{t}}\left(\bm{u}^{t}\hat{\bm{S}}^{t}+\bm{W}^{*}\hat{\bm{m}}^{\top}\hat{\bm{S}}^{t}+\bm{W}^{*}\rho^{-1}\bm{m}^{\top}\right)-\bm{W}^{*}\rho^{-1}\bm{m}^{\top} \\
    &\mbox{and} \quad \bm{S}^{t} = \langle (\bm{e}^{t})'\rangle^{\top}, \quad \hat{\bm{S}}^{t} = -\left(\langle (\bm{h}^{t})'\rangle^{\top} \right)^{-1}
    \label{eq:AMP2}
\end{align}
Then the fixed point $(\bm{u}^{\infty}, \bm{v}^{\infty})$ of this iteration verifies 
\begin{align}
    &\bbR_{r(.),\hat{\bm{S}}^{\infty}}\left(\bm{u}^{\infty}\hat{\bm{S}}^{\infty}+\bm{W}^{*}\hat{\bm{m}}^{\top}\hat{\bm{S}}^{\infty}+\bm{W}^{*}\rho^{-1}\bm{m}^{\top}\right)-\bm{W}^{*}\rho^{-1}\bm{m}^{\top} = \bm{U}^{*}\\
    &\bbR_{\mathcal{L}(\bm{Y},.),\bm{S}^{\infty}}(\bm{S}\rho^{-1/2}\bm{m}^{\top}+\bm{v}^{\infty})-\bm{S}\rho^{-1/2}\bm{m}^{\top} = \tilde{\bm{X}}\bm{U}^{*}
\end{align}
where $\mathbf{U}^{*}$ is the unique solution to the optimisation problem Eq.~\eqref{student1}.
\end{lemma}

\begin{proof}
To find the correct form of the non-linearities in the AMP iteration, we match the optimality condition of problem Eq.~\eqref{student1} with the generic form of the fixed point of the AMP iteration Eq.~\eqref{canon_AMP}.
In the subsequent derivation, we absorb the scaling $1/{\sqrt{d}}$ in the matrix $\tilde{\bX}$, such that its elements are i.i.d. $\mathcal{N}(0,1/d)$, and omit time indices for simplicity.
Going back to problem Eq.~\eqref{student1}, its optimality condition reads :
\begin{align}
	&\tilde{\bX}^\top \partial \tilde{\mathcal{L}}(\tilde{\bX}\bU)+\partial\tilde{r}(\bU) = 0 
\end{align}

For any pair of $k \times k$ symmetric positive definite matrices $\bm{S}, \hat{\bm{S}}$, this optimality condition is equivalent to 
\begin{equation}
\label{eq:inter_opt}
    \tilde{\bX}^\top\left(\partial \tilde{\mathcal{L}}(\tilde{\bX}\bU)\bm{S}+\tilde{\bm{X}}\bm{U}\right)\bm{S}^{-1} +\left(\partial\tilde{r}(\bU)\hat{\bm{S}}+\bm{U}\right)\hat{\bm{S}}^{-1}= \tilde{\bm{X}}^{\top}\tilde{\bm{X}}\bm{U}\bm{S}^{-1}+\bm{U}\hat{\bm{S}}^{-1}
\end{equation}
where we added the same quantity on both sides of the equality.
For the loss function, we can then introduce the resolvent, formally D-resolvent: 
\begin{equation}
	\hat{\bv} = \partial \tilde{\mathcal{L}}(\tilde{\bX}\bU)\bS+\tilde{\bX}\bU \iff \tilde{\bX}\bU = \bbR_{\tilde{\mathcal{L}},\bS}(\hat{\bv})
\end{equation}
such that
\begin{equation}
\label{block_res_loss}
\begin{split}
	\bbR_{\tilde{\mathcal{L}},\bS}(\hat{\bv}) &= (\mathrm{Id}+\partial \tilde{\mathcal{L}}(\bullet)\bS)^{-1}(\hat{\bv}) = \underset{\bT \in \mathbb{R}^{n \times K}}{\arg\min}\left\{\tilde{\mathcal{L}}(\bT)+\frac{1}{2}\mathrm{tr}\left((
        \bT-\hat{\bv})\bS^{-1}(\bT-\hat{\bv})^\top \right)\right\} \\
\end{split}
\end{equation}
Similarly for the regularisation, introduce
\begin{equation}
\label{eq:res_reg}
\hat{\bm{u}} \equiv \left(\bm{I}+\partial \tilde{r}(\bullet)\bhS\right)(\bU) \qquad \bU = \bbR_{\tilde{r},\bhS}(\hat{\bm{u}})
\end{equation}
where $\bm{S} \in \mathbb{R}^{k \times k}$ is a positive definite matrix, and 
\begin{equation}
\bbR_{\tilde{r},\bhS}(\hat{\bm{v}}) = \left(\bm{I}+\partial \tilde{r}(\bullet)\bhS\right)^{-1}(\hat{\bm{v}})= \underset{\bT \in \mathbb{R}^{d \times k}}{\arg\min}\left\{\tilde{r}(\bT)+\frac{1}{2}\mathrm{tr}\left((
	\bT-\hat{\mathbf{v}})\bhS^{-1}(\bT-\hat{\mathbf{v}})^\top \right)\right\}
\end{equation}
where $\bhS \in \mathbb{R}^{k\times k}$ is a positive definite matrix, and $\hat{\bm{v}}\in \mathbb{R}^{d\times k}$.
The optimality condition Eq.~\eqref{eq:inter_opt} may then be rewritten as:
\begin{align}
\label{prox-opti}
	\tilde{\bX}^\top \left(\bbR_{\tilde{\mathcal{L}},\bS}(\hat{\bm{v}})-\hat{\bm{v}}\right)\bS^{-1} &= (\hat{\mathbf{u}}-\bbR_{\tilde{r},\bhS}(\hat{\mathbf{u}}))\bhS^{-1} \\
	\tilde{\bX}\bbR_{\tilde{r},\bhS}(\hat{\mathbf{u}}) &= \bbR_{\tilde{\mathcal{L}},\bS}(\hat{\mathbf{v}})
\end{align}
where both equations should be satisfied. We can now define update functions based on the previously obtained block decomposition. The fixed point of the matrix-valued AMP Eq.~(\ref{canon_AMP}), omitting the time indices for simplicity, reads:
\begin{align}
	\bu+\be(\bu)\langle \bh' \rangle^\top  &= \tilde{\bX}^\top \bh(\bv) \\
	\bv+\bh(\bv)\langle \be' \rangle^\top  &= \tilde{\bX}\be(\bu) 
\end{align}
Matching this fixed point with the optimality condition Eq.~(\ref{prox-opti}) suggests the following mapping:
\begin{equation}
\begin{split}
\bh(\bm{v}) &= \left(\bbR_{\tilde{\mathcal{L}},\bS}(\bm{v})-\bm{v}\right)\bS^{-1}, \\
\be(\bm{u}) &= \bbR_{\tilde{r},\bhS}(\bm{u}\bhS),
\end{split}\qquad
\begin{split}
\bS &= \langle \be' \rangle^{\top},\\
\bhS &= -(\langle \bh '\rangle^{\top})^{-1},
\end{split}
\end{equation}
where we redefined $\hat{\bm{u}}\equiv \hat{\bm{u}}\bhS$ in \eqref{eq:res_reg}. We are now left with the task of evaluating the resolvents of $\tilde{\mathcal{L}}, \tilde{r}$ as expressions of the original functions $\mathcal{L},r$. Starting with the loss function, we get 
\begin{align}
    &\bbR_{\tilde{\mathcal{L}},\bS}(\bm{v}) = \underset{\bm{x} \in \mathbb{R}^{n \times k}}{\arg\min} \left\{\mathcal{L}\left(\phi_{\rm out}\left(\sqrt{\rho}\bm{s}\right), \bm{S}\rho^{-1/2}\bm{m}^{\top}+\bm{x}\right)+\frac{1}{2}\mbox{tr}\left((\bm{x}-\bm{v})\bm{S}^{-1}(\bm{x}-\bm{v})\right)^{\top}\right\}
\end{align}
letting $\tilde{\bm{x}} = \bm{S}\rho^{-1/2}\bm{m}^{\top}+\bm{x}$, the problem is equivalent to 
\begin{align}
    &\bbR_{\tilde{\mathcal{L}},\bS}(\bm{v}) = \underset{\tilde{\bm{x}} \in \mathbb{R}^{n \times k}}{\arg\min} \bigg\{\mathcal{L}\left(\phi_{\rm out}\left(\sqrt{\rho}\mathbf{s}\right), \tilde{\bm{x}}\right) \notag \\
    &+\frac{1}{2}\mbox{tr}\left((\tilde{\bm{x}}-(\bm{S}\rho^{-1/2}\bm{m}^{\top}+\bm{v}))\bm{S}^{-1}(\tilde{\bm{x}}-(\bm{S}\rho^{-1/2}\bm{m}^{\top}+\bm{v}))^{\top}\right)\bigg\}-\bm{S}\rho^{-1}\bm{m}^{\top} \\
    & = \bbR_{\mathcal{L}(\bm{Y},.),\bm{S}}(\bm{S}\rho^{-1/2}\bm{m}^{\top}+\bm{v})-\bm{S}\rho^{-1}\bm{m}^{\top}
\end{align}
and the corresponding non-linearity will then be 
\begin{equation}
    \bh(\bm{v}) = \left(R_{\mathcal{L}(\bm{Y},.),\bm{S}}(\bm{S}\rho^{-1/2}\bm{m}^{\top}+\bm{v})-\left(\bm{S}\rho^{-1/2}\bm{m}^{\top}+\bm{v}\right)\right)\bm{S}^{-1}
\end{equation}
Moving to the regularisation, the resolvent reads 
\begin{align}
    &\bbR_{\tilde{r},\bhS}(\bm{u}) = \underset{\bm{x} \in \mathbb{R}^{d \times k}}{\arg\min} \bigg\{r\left(\bm{W}^{*}\rho^{-1}\bm{m}^{\top}+\bm{x}\right)-\mbox{Tr}\left(\hat{\bm{m}}^{\top}\bm{x}^{\top}\bm{W}^{*}\right)+\frac{1}{2}\mbox{tr}\left((\bm{x}-\bm{u})\hat{\bm{S}}^{-1}(\bm{x}-\bm{u})^{\top}\right)\bigg\}
\end{align}
letting $\tilde{\bm{x}} = \bm{W}^{*}\rho^{-1}\bm{m}^{\top}+\bm{x}$, we obtain
\begin{align}
    \bbR_{\tilde{r},\bhS}(\bm{u}) &= \underset{\tilde{\bm{x}} \in \mathbb{R}^{d \times k}}{\arg\min} \bigg\{r\left(\tilde{\bm{x}}\right)-\hat{\bm{m}}^{\top}\tilde{\bm{x}}^{\top}\bm{W}^{*}\\
    &\hspace{1cm}+\frac{1}{2}\mbox{tr}\left((\tilde{\bm{x}}-\left(\bm{u}+\bm{W}^{*}\rho^{-1}\bm{m}^{\top}\right))\hat{\bm{S}}^{-1}(\tilde{\bm{x}}-\left(\bm{u}+\bm{W}^{*}\rho^{-1}\bm{m}^{\top}\right))^{\top}\right)\bigg\}-\bm{W}^{*}\rho^{-1}\bm{m}^{\top} \\
    &=\underset{\tilde{\bm{x}} \in \mathbb{R}^{d \times k}}{\arg\min} \bigg\{r\left(\tilde{\bm{x}}\right)\\
    &+\frac{1}{2}\mbox{tr}\left((\tilde{\bm{x}}-\left(\bm{u}+\bm{W}^*\hat{\bm{m}}^{\top}\hat{\bm{S}}+\bm{W}^{*}\rho^{-1}\bm{m}^{\top}\right))\hat{\bm{S}}^{-1}(\tilde{\bm{x}}-\left(\bm{u}+\bm{W}^{*}\hat{\bm{m}}^{\top}\hat{\bm{S}}+\bm{W}^{*}\rho^{-1}\bm{m}^{\top}\right))^{\top}\right)\bigg\}\\
    &\hspace{5cm}-\bm{W}^{*}\rho^{-1}\bm{m}^{\top} \\
    &\bbR_{r(.),\hat{\bm{S}}}\left(\bm{u}+\bm{W}^{*}\hat{\bm{m}}^{\top}\hat{\bm{S}}+\bm{W}^{*}\rho^{-1}\bm{m}^{\top}\right)-\bm{W}^{*}\rho^{-1}\bm{m}^{\top}
\end{align}
Which gives the following non-linearity for the AMP iteration
\begin{equation}
    e(\bm{u}) = \bbR_{r(.),\hat{\bm{S}}}\left(\bm{u}\hat{\bm{S}}+\bm{W}^{*}\hat{\bm{m}}^{\top}\hat{\bm{S}}+\bm{W}^{*}\rho^{-1}\bm{m}^{\top}\right)-\bm{W}^{*}\rho^{-1}\bm{m}^{\top}
\end{equation}
\end{proof}
The following lemma then gives the exact asymptotics at each time step of the AMP iteration solving problem Eq.\eqref{student1} : its \emph{state evolution equations}.
\begin{lemma}
    \label{lemma:SE_inter}
Consider the AMP iteration Eq.(\ref{eq:AMP1}-\ref{eq:AMP2}). Assume it is initialised with $\mathbf{u}^{0}$ such that \\
 $\lim_{d \to \infty}\frac{1}{d}\norm{\be_{0}(\bu^{0})^\top\be_{0}(\bu^{0})}_{\rm F}$ exists, a positive 
definite matrix $\hat{\bm{S}}_{0}$, and $\bh_{-1}\equiv 0$. Then for any $t \in \mathbb{N}$, and any pair of sequences of uniformly pseudo-Lipschitz functions $\phi_{1,n} : \mathbb{R}^{d \times k}$ and $\phi_{2,n} : \mathbb{R}^{n \times k}$, the following holds
\begin{align}
    &\phi_{1,n}\left(\bm{u}^{t}\right) \stackrel{P}{\simeq}\mathbb{E}\left[\phi_{1,n}\left(\bm{G}(\hat{\bm{Q}}^{t})^{1/2}\right)\right]\\
    &\phi_{2,n}\left(\bm{v}^{t}\right) \stackrel{P}{\simeq}\mathbb{E}\left[\phi_{2,n}\left(\bm{H}(\bm{Q}^{t})^{1/2}\right)\right]
\end{align}
where $\bm{G} \in \mathbb{R}^{d \times k}$ and $\bm{H} \in \mathbb{R}^{n \times K}$ are independent random matrices with i.i.d. standard normal elements, and $\bm{Q}^{t}, \hat{\bm{Q}}^{t}, \bm{V}^{t}, \hat{\bm{V}}^{t}$ are given by the equations
\begin{align}
&\bm{Q}^{t} = \frac{1}{d}\mathbb{E}\bigg[\left(\bbR_{r(.),(\hat{\bm{V}}^{t})^{-1}}\left(\bm{G}(\hat{\bm{Q}}^{t})^{1/2}(\hat{\bm{V}}^{t})^{-1}+\bm{W}^{*}\hat{\bm{m}}^{\top}(\hat{\bm{V}}^{t})^{-1}+\bm{W}^{*}\rho^{-1}\bm{m}^{\top}\right)-\bm{W}^{*}\rho^{-1}\bm{m}^{\top}\right)^{\top} \notag \\
&\hspace{3cm}\left(\bbR_{r(.),(\hat{\bm{V}}^{t})^{-1}}\left(\bm{G}(\hat{\bm{Q}}^{t})^{1/2}(\hat{\bm{V}}^{t})^{-1}+\bm{W}^{*}\hat{\bm{m}}^{\top}(\hat{\bm{V}}^{t})^{-1}+\bm{W}^{*}\rho^{-1}\bm{m}^{\top}\right)-\mathbf{W}^{*}\rho^{-1}\mathbf{m}^{\top}\right)\bigg] \\
&\hat{\bm{Q}}^{t} = \frac{1}{d}\mathbb{E}\bigg[\left(\left(\bbR_{\mathcal{L}(\mathbf{Y},.),\mathbf{V}^{t-1}}(.)-Id\right)\left(\mathbf{S}\rho^{-1/2}\mathbf{m}^{\top}+\mathbf{H}(\mathbf{Q}^{t-1})^{1/2}\right)(\mathbf{V}^{t-1})^{-1}\right)^{\top}\notag \\
&\hspace{3cm}\left(\left(\bbR_{\mathcal{L}(\mathbf{Y},.),\mathbf{V}^{t-1}}(.)-Id\right)\left(\mathbf{S}\rho^{-1/2}\mathbf{m}^{\top}+\mathbf{H}(\mathbf{Q}^{t-1})^{1/2}\right)(\mathbf{V}^{t-1})^{-1}\right)\bigg] \\
&\bm{V}^{t} = \frac{1}{d}\mathbb{E}\left[(\hat{\mathbf{Q}}^{t})^{-1/2}\mathbf{G}^{\top}R_{r(.),(\hat{\mathbf{V}}^{t})^{-1}}\left(\mathbf{G}(\hat{\mathbf{Q}}^{t})^{1/2}(\hat{\mathbf{V}}^{t})^{-1}+\mathbf{W}^{*}\hat{\mathbf{m}}^{\top}(\hat{\mathbf{V}}^{t})^{-1}+\mathbf{W}^{*}\rho^{-1}\mathbf{m}^{\top}\right)\right] \\
&\hat{\bm{V}}^{t} =- \frac{1}{d}\mathbb{E}\left[(\mathbf{Q}^{t-1})^{-1/2}\mathbf{H}^{\top}\left(\left(\bbR_{\mathcal{L}(\mathbf{Y},.),\mathbf{V}^{t-1}}(.)-Id\right)\left(\mathbf{S}\rho^{-1/2}\mathbf{m}^{\top}+\mathbf{H}(\mathbf{Q}^{t-1})^{1/2}\right)\right)(\mathbf{V}^{t-1})^{-1}\right]
\end{align}
\end{lemma}
\begin{proof}
    Owing to the properties of Bregman proximity operators \cite{bauschke2003bregman,bauschke2006joint}, the update functions in the AMP iteration 
    Eq.(\ref{eq:AMP1}-\ref{eq:AMP2}) are Lipschitz continuous. Thus under the assumptions made on the initialisation, the assumptions of Theorem \ref{th:SE} are verified, 
    which gives the desired result.
\end{proof}
\begin{lemma}
\label{lemma:conv_traj}
Consider iteration Eq.~(\ref{eq:AMP1}-\ref{eq:AMP2}), where the parameters $\bQ,\hat{\bQ},\bV,\hat{\bV}$ are initialised at any fixed point of the state evolution equations of Lemma \ref{lemma:SE_inter}. For any sequence initialised with $\bhV_{0} = \bhV$ and  $\bu_{0}$ such that
\begin{equation}
\lim_{d \to \infty}\frac{1}{d}\be_{0}({\bu_{0}})^{\top}\be_{0}(\bu_{0}) = \bQ
\end{equation}
the following holds
\begin{equation}
\lim_{t\to \infty}\lim_{d \to \infty}\frac{1}{\sqrt{d}}\norm{\bu^{t}-\bu^{\star}}_{\rm F} = 0 \quad \lim_{t\to \infty}\lim_{d \to \infty}\frac{1}{\sqrt{d}}\norm{\bv^{t}-\bv^{\star}}_{\rm F} = 0
\end{equation}
\end{lemma}
\begin{proof}
    The proof of this lemma is identical to that of Lemma 7 from \cite{loureiro2021learning}.
\end{proof}
Combining these results, we obtain the following asymptotic characterisation of $\bm{U}^{*}$.
\begin{lemma}
\label{lemma:asy_inter}
For any fixed $\bm{m}$ and $\hat{\bm{m}}$ in their feasibility sets, let $\bm{U}^{*}$ be the unique solution to the optimisation problem Eq.\eqref{student1}. Then, for any sequences (in the problem dimension) of pseudo-Lipschitz functions of order 2 $\phi_{1,n} : \mathbb{R}^{n \times k} \to \mathbb{R}$ and $\phi_{2,n} : \mathbb{R}^{d \times k} \to \mathbb{R}$, the following holds 
\begin{align}
        &\phi_{1,n}\left(\mathbf{U}^{*}\right) \stackrel{P}{\simeq}\mathbb{E}\left[\phi_{1,n}\left( \bbR_{r(.),\hat{\mathbf{V}}^{-1}}\left(\mathbf{G}\hat{\mathbf{Q}}^{1/2}\hat{\mathbf{V}}^{-1}+\mathbf{W}^{*}\hat{\mathbf{m}}^{\top}\hat{\mathbf{V}}^{-1}+\mathbf{W}^{*}\rho^{-1}\mathbf{m}^{\top}\right)-\mathbf{W}^{*}\rho^{-1}\mathbf{m}^{\top}\right)\right]\\
        &\phi_{2,n}\left(\frac{1}{\sqrt{d}}\tilde{\mathbf{X}}\mathbf{U}^{*}\right) \stackrel{P}{\simeq}\mathbb{E}\left[\phi_{2,n}\left( \bbR_{\mathcal{L}(\mathbf{Y},.),\mathbf{V}}(\mathbf{S}\rho^{-1/2}\mathbf{m}^{\top}+\mathbf{H}\hat{\mathbf{Q}}^{1/2})-\mathbf{S}\rho^{-1}\mathbf{m}^{\top}\right)\right]
\end{align}
where $\mathbf{G} \in \mathbb{R}^{d \times k}$ and $\mathbf{H} \in \mathbb{R}^{n \times k}$ are independent random matrices with i.i.d. standard normal elements, and $\mathbf{Q}, \hat{\mathbf{Q}}, \mathbf{V}, \hat{\mathbf{V}}$ are given by the fixed point of the following set of self consistent equations
\begin{align}
    &\mathbf{Q} = \frac{1}{d}\mathbb{E}\bigg[\left(\bbR_{r(.),\hat{\mathbf{V}}^{-1}}\left(\mathbf{G}\hat{\mathbf{Q}}^{1/2}\hat{\mathbf{V}}^{-1}+\mathbf{W}^{*}\hat{\mathbf{m}}^{\top}\hat{\mathbf{V}}^{-1}+\mathbf{W}^{*}\rho^{-1}\mathbf{m}^{\top}\right)-\mathbf{W}^{*}\rho^{-1}\mathbf{m}^{\top}\right)^{\top} \notag \\
    &\hspace{3cm}\left(\bbR_{r(.),\hat{\mathbf{V}}^{-1}}\left(\mathbf{G}\hat{\mathbf{Q}}^{1/2}\hat{\mathbf{V}}^{-1}+\mathbf{W}^{*}\hat{\mathbf{m}}^{\top}\hat{\mathbf{V}}^{-1}+\mathbf{W}^{*}\rho^{-1}\mathbf{m}^{\top}\right)-\mathbf{W}^{*}\rho^{-1}\mathbf{m}^{\top}\right)
    \bigg] \\
    &\hat{\mathbf{Q}} = \frac{1}{d}\mathbb{E}\bigg[\left(\left(\bbR_{\mathcal{L}(\mathbf{Y},.),\mathbf{V}}(.)-Id\right)\left(\mathbf{S}\rho^{-1/2}\mathbf{m}^{\top}+\mathbf{H}\mathbf{Q}^{1/2}\right)\mathbf{V}^{-1}\right)^{\top} \notag \\
    &\hspace{3cm}\left(\left(\bbR_{\mathcal{L}(\mathbf{Y},.),\mathbf{V}}(.)-Id\right)\left(\mathbf{S}\rho^{-1/2}\mathbf{m}^{\top}+\mathbf{H}\mathbf{Q}^{1/2}\right)\mathbf{V}^{-1}\right)\bigg] \\
    &\mathbf{V} = \frac{1}{d}\mathbb{E}\left[\hat{\mathbf{Q}}^{-1/2}\mathbf{G}^{\top}\bbR_{r(.),\hat{\mathbf{V}}^{-1}}\left(\mathbf{G}\hat{\mathbf{Q}}^{1/2}\hat{\mathbf{V}}^{-1}+\mathbf{W}^{*}\hat{\mathbf{m}}^{\top}\hat{\mathbf{V}}^{-1}+\mathbf{W}^{*}\rho^{-1}\mathbf{m}^{\top}\right)\right] \\
    &\hat{\mathbf{V}} =- \frac{1}{d}\mathbb{E}\left[\mathbf{Q}^{-1/2}\mathbf{H}^{\top}\left(\left(\bbR_{\mathcal{L}(\mathbf{Y},.),\mathbf{V}}(.)-Id\right)\left(\mathbf{S}\rho^{-1/2}\mathbf{m}^{\top}+\mathbf{H}\mathbf{Q}^{1/2}\right)\mathbf{V}^{-1}\right)\right]
\end{align}
\end{lemma}
\begin{proof}
    Combining the results of the previous lemmas, this proof is close to that of Theorem 1.5 in \cite{bayati2011lasso}.
\end{proof}
Returning to the optimisation problem on $\mathbf{m}, \hat{\mathbf{m}}$ in Eq.~\eqref{eq:AMP_target_inv}, the solution $\bm{U}^{*}$, at any dimension, verifies the zero gradient conditions on $\bm{m}, \hat{\bm{m}}$:
\begin{align}
    \label{eq:subgrad_Lip}
    &\partial \hat{\bm{m}} = 0 \iff (\bm{U^{*}})^{\top}\bm{W}^{*} = 0 \\
    &\partial \bm{m} = 0 \iff \bm{m}\rho^{-1/2}\bm{S}^{\top}\partial\mathcal{L}\left(\phi_{\rm out}\left(\sqrt{\rho}\bm{s}\right), \bm{S}\rho^{-1/2}\bm{m}^{\top}+\frac{1}{\sqrt{d}}\tilde{\bm{X}}\bm{U}\right) \notag \\
    &\hspace{2.5cm}+\rho^{-1}(\bm{W}^{*})^{\top}\partial r(\left(\bm{W}^{*}\rho^{-1}\bm{m}^{\top}+\bm{U}\right)) = 0
\end{align}
Using Lemma \ref{lemma:asy_inter} with the assumption that the gradients of $\mathcal{L}, r$ are pseudo-Lipschitz, we obtain for $\mathbf{m}$
\begin{align}
     &\frac{1}{d}\mathbb{E}\left[\left( \bbR_{r(.),\hat{\bm{V}}^{-1}}\left(\bm{G}\hat{\bm{Q}}^{1/2}\hat{\bm{V}}^{-1}+\bm{W}^{*}\hat{\bm{m}}^{\top}\hat{\bm{V}}^{-1}+\bm{W}^{*}\rho^{-1}\bm{m}^{\top}\right)-\bm{W}^{*}\rho^{-1}\bm{m}^{\top}\right)^{\top}\bm{W}^{*} \right]= 0 \\
     &\iff \bm{m} = \frac{1}{d}\mathbb{E}\left[(\bm{W}^{*})^{\top}\bbR_{r(.),\hat{\bm{V}}^{-1}}\left(\bm{G}\hat{\bm{Q}}^{1/2}\hat{\bm{V}}^{-1}+\bm{W}^{*}\hat{\bm{m}}^{\top}\hat{\bm{V}}^{-1}+\bm{W}^{*}\rho^{-1}\bm{m}^{\top}\right)\right]
\end{align}
 and for $\hat{\mathbf{m}}$
 \begin{align}
     &\frac{1}{d}\mathbb{E}\bigg[\bm{m}\rho^{-1/2}\bm{S}^{\top}\partial\mathcal{L}\left(\phi_{\rm out}\left(\sqrt{\rho}\mathbf{s}\right), \bbR_{\mathcal{L}(\bm{Y},.),\bm{V}}(\bm{S}\rho^{-1/2}\bm{m}^{\top}+\bm{H}\hat{\bm{Q}}^{1/2})\right)\\
    &\hspace{1cm}+\rho^{-1}(\bm{W}^{*})^{\top}\partial r\left(\left(\bbR_{r(.),\hat{\bm{V}}^{-1}}\left(\bm{G}\hat{\bm{Q}}^{1/2}\hat{\bm{V}}^{-1}+\bm{W}^{*}\hat{\bm{m}}^{\top}\hat{\bm{V}}^{-1}+\bm{W}^{*}\rho^{-1}\bm{m}^{\top}\right)\right)\right)\bigg] = 0
 \end{align}
Using the definition of D-resolvents, this is equivalent to 
 \begin{align}
     &\frac{1}{d}\mathbb{E}\bigg[\mathbf{m}\rho^{-1/2}\bm{S}^{\top}\left(Id-\bbR_{\mathcal{L}(\bm{Y},.),\bm{V}}\left(.\right)\right)\left(\bm{S}\rho^{-1/2}\bm{m}^{\top}+\bm{H}\hat{\bm{Q}}^{1/2}\right)\bm{V}^{-1}\\
    &+\rho^{-1}(\bm{W}^{*})^{\top}\bigg(Id-\bbR_{r(.),\hat{\bm{V}}^{-1}}\left(.\right)\bigg)\left(\bm{G}\hat{\bm{Q}}^{1/2}\hat{\bm{V}}^{-1}+\bm{W}^{*}\hat{\bm{m}}^{\top}\hat{\bm{V}}^{-1}+\bm{W}^{*}\rho^{-1}\bm{m}^{\top}\right)\hat{\bm{V}}\bigg] = 0
 \end{align}
 which simplifies to 
 \begin{align}
     \hat{\bm{m}}^{\top} = -\frac{1}{d}\mathbb{E}\bigg[\bm{m}\rho^{-1/2}\bm{S}^{\top}\left(Id-\bbR_{\mathcal{L}(\bm{Y},.),\bm{V}}\left(.\right)\right)\left(\bm{S}\rho^{-1/2}\bm{m}^{\top}+\bm{H}\hat{\bm{Q}}^{1/2}\right)\bm{V}^{-1}\bigg]
 \end{align}
 which brings us to the following set of six self consistent equations
 \begin{align}
    \label{eq:SE_final}
    &\mathbf{Q} = \frac{1}{d}\mathbb{E}\bigg[\left(\bbR_{r(.),\hat{\bm{V}}^{-1}}\left(\bm{G}\hat{\bm{Q}}^{1/2}\hat{\bm{V}}^{-1}+\bm{W}^{*}\hat{\bm{m}}^{\top}\hat{\bm{V}}^{-1}+\bm{W}^{*}\rho^{-1}\bm{m}^{\top}\right)-\bm{W}^{*}\rho^{-1}\bm{m}^{\top}\right)^{\top} \notag \\
    &\hspace{3cm}\left(\bbR_{r(.),\hat{\bm{V}}^{-1}}\left(\bm{G}\hat{\bm{Q}}^{1/2}\hat{\bm{V}}^{-1}+\bm{W}^{*}\hat{\bm{m}}^{\top}\hat{\bm{V}}^{-1}+\bm{W}^{*}\rho^{-1}\bm{m}^{\top}\right)-\bm{W}^{*}\rho^{-1}\bm{m}^{\top}\right)
    \bigg] \\
    &\hat{\bm{Q}} = \frac{1}{d}\mathbb{E}\bigg[\left(\left(\bbR_{\mathcal{L}(\bm{Y},.),\bm{V}}(.)-Id\right)\left(\bm{S}\rho^{-1/2}\bm{m}^{\top}+\bm{H}\mathbf{Q}^{1/2}\right)\bm{V}^{-1}\right)^{\top} \notag \\
    &\hspace{3cm}\left(\left(\bbR_{\mathcal{L}(\bm{Y},.),\bm{V}}(.)-Id\right)\left(\bm{S}\rho^{-1/2}\bm{m}^{\top}+\bm{H}\bm{Q}^{1/2}\right)\bm{V}^{-1}\right)\bigg] \\
    &\bm{V} = \frac{1}{d}\mathbb{E}\left[\hat{\bm{Q}}^{-1/2}\bm{G}^{\top}\bbR_{r(.),\hat{\bm{V}}^{-1}}\left(\bm{G}\hat{\bm{Q}}^{1/2}\hat{\bm{V}}^{-1}+\bm{W}^{*}\hat{\bm{m}}^{\top}\hat{\bm{V}}^{-1}+\bm{W}^{*}\rho^{-1}\bm{m}^{\top}\right)\right] \\
    &\hat{\bm{V}} =- \frac{1}{d}\mathbb{E}\left[\bm{Q}^{-1/2}\bm{H}^{\top}\left(\left(\bbR_{\mathcal{L}(\bm{Y},.),\bm{V}}(.)-Id\right)\left(\bm{S}\rho^{-1/2}\bm{m}^{\top}+\bm{H}\bm{Q}^{1/2}\right)\bm{V}^{-1}\right)\right] \\
    &\bm{m} = \frac{1}{d}\mathbb{E}\left[(\bm{W}^{*})^{\top}\bbR_{r(.),\hat{\bm{V}}^{-1}}\left(\bm{G}\hat{\bm{Q}}^{1/2}\hat{\bm{V}}^{-1}+\bm{W}^{*}\hat{\bm{m}}^{\top}\hat{\bm{V}}^{-1}+\bm{W}^{*}\rho^{-1}\bm{m}^{\top}\right)\right] \\
    &\hat{\bm{m}}^{\top} = -\frac{1}{d}\mathbb{E}\bigg[\bm{m}\rho^{-1/2}\bm{S}^{\top}\left(\bm{I}-\bbR_{\mathcal{L}(\bm{Y},.),\bm{V}}\left(.\right)\right)\left(\bm{S}\rho^{-1/2}\bm{m}^{\top}+\bm{H}\hat{\bm{Q}}^{1/2}\right)\bm{V}^{-1}\bigg]
\end{align}
These equations then characterise the asymptotic properties of the quantities $\hat{\bm{U}}$ and $\tilde{\bm{X}}\hat{\bm{U}}/{\sqrt{d}}$. The properties of $\hat{\bm{W}}$ and $\bm{X}\hat{\bm{W}}/{\sqrt{d}}$ are then obtained by using the definition of $\bm{U}$ in terms of orthogonal decompositions. Note that 
To match these equations with the replica ones, we first need to assume the loss and cost functions are separable. The proximal operators are then separable as well across lines of the input matrices. All arguments then have i.i.d. lines (Gaussian matrices with $k \times k$ covariances, or lines of the teacher matrix, which are i.i.d., multiplied with $k \times k$ matrices), and the $1/d$ averages simplify, leaving the aspect ratio in the quantities defined over arguments in $\mathbb{R}^{n \times k}$. The rest of the matching then boils down to identifying the proximal operators with the replica notations, done in Section \ref{appendix:mapping}, and standard Gaussian integration, as done for instance in \cite{aubin2020generalization}, Appendix III.3.
\subsection{Toolbox}
\label{sec:req_back}
In this section, we reproduce part of the appendix of \cite{loureiro2021learning} for completeness, in order to give an overview of the main concepts and tools on approximate message passing algorithms which will be required for the proof.\vspace{0.1cm}
\paragraph{Notations ---}
For a given function $\bphi\colon \mathbb{R}^{d \times k}\to \mathbb{R}^{n \times k}$, we write :
\begin{equation}
    \bphi(\bX) = \begin{bmatrix}
    \bphi^{1}(\bX) \\
    \vdots \\
    \bphi^{d}(\bX)\end{bmatrix} \in \mathbb{R}^{d \times k}
\end{equation}
where each $\bphi^{i} \colon \mathbb{R}^{d \times k} \to \mathbb{R}^{k}$. We then write the $k \times k$ Jacobian 
\begin{equation}
\label{eq:jacob}
    \frac{\partial \bphi^{i}}{\partial \bX_{j}}(\bX) = \begin{bmatrix}\frac{\partial \phi^{i}_1(\bX)}{\partial X_{j1}} & \cdots & \frac{\partial \phi^{i}_1(\bX)}{\partial X_{jk}} \\
    \vdots&\ddots&\vdots \\
    \frac{\partial \phi^{i}_k(\bX)}{\partial X_{j1}} & \cdots & \frac{\partial \phi^{i}_k(\bX)}{\partial X_{jk}}
    \end{bmatrix} \in \mathbb{R}^{k \times k}
\end{equation}
For a given matrix $\bQ\in \mathbb{R}^{k \times k}$, we write $\bZ \in \mathbb{R}^{n \times k} \sim \mathcal{N}(\boldsymbol 0,\bQ\otimes \bI_{n})$ to denote that the lines of $\bZ$ are sampled i.i.d.~from $\mathcal{N}(\boldsymbol 0,\bQ)$. Note that this is equivalent to saying that $\bZ = \tilde{\bZ}\bQ^{1/2}$ where $\tilde{\bZ} \in \mathbb{R}^{n \times k}$ is an i.i.d.~standard normal random matrix. The notation $\stackrel{\rm P}\simeq$ denotes convergence in probability.
We start with some definitions that commonly appear in the approximate message-passing literature, see e.g. \cite{bayati2011dynamics,javanmard2013state}.
The main regularity class of functions we will use is that of pseudo-Lipschitz functions, which roughly amounts to functions with polynomially bounded first derivatives. We include the required scaling w.r.t.~the dimensions in the definition for convenience.
\begin{definition}[Pseudo-Lipschitz function]
For $K,k \in \mathbb{N}^{*}$ and any $n,d \in \mathbb{N}^{*}$, a function $\bphi \colon \mathbb{R}^{d \times k} \to \mathbb{R}^{n \times k}$ is called a \emph{pseudo-Lipschitz of order $K$} if there exists a constant $L(K,k)$ such that for any $\bX,\bY \in \mathbb{R}^{d \times k}$, 
\begin{equation}
    \frac{\norm{\bphi(\bX)-\bphi(\bY)}_{\rm F}}{\sqrt{n}} \leqslant L \left(1+\left(\frac{\norm{\bX}_{\rm F}}{\sqrt{d}}\right)^{K-1}+\left(\frac{\norm{\bY}_{\rm F}}{\sqrt{d}}\right)^{K-1}\right)\frac{\norm{\bX-\bY}_{\rm F}}{\sqrt{d}}
\end{equation}
where $\norm{\bullet }_{\rm F}$ denotes the Frobenius norm.
Since $k$ will be kept finite, it can be absorbed in any of the constants.
\end{definition}
For example, the function $f:\mathbb{R}^{d \times k}\to \mathbb{R}, \bX \mapsto \norm{\bX}_{F}^{2}/d$ is pseudo-Lipshitz of order 2.

\paragraph{Moreau envelopes and Bregman proximal operators ---} In our proof, we will also frequently use the notions of Moreau envelopes and proximal operators, see e.g. \cite{parikh2014proximal,bauschke2011convex}. These elements of convex analysis are often encountered in recent works on high-dimensional asymptotics of convex problems, and more detailed analysis of their properties can be found for example in \cite{thrampoulidis2018precise,loureiro2021learning}. For the sake of brevity, we will only sketch the main properties of such mathematical objects, referring to the cited literature for further details. In this proof, we will mainly use proximal operators acting on sets of real matrices endowed with their canonical scalar product. Furthermore, proximals will be defined with matrix valued parameters in the following way: for a given convex function $f \colon\mathbb{R}^{d \times k}\to\mathbb R$, a given matrix $\bX \in \mathbb{R}^{d \times k}$ and a given symmetric positive definite matrix $\bV \in \mathbb{R}^{k \times k}$ with bounded spectral norm, we will consider operators of the type
\begin{equation}
    \underset{\bT \in \mathbb{R}^{d \times k}}{\arg\min} \left\{f(\bT)+\frac{1}{2}\mathrm{tr}\left((\bT-\bX)\bV^{-1}(\bT-\bX)^{\top}\right)\right\}
\end{equation}
This operator can either be written as a standard proximal operator by factoring the matrix $\bV^{-1}$ in the arguments of the trace:
\begin{equation}
    \Prox_{f(\bullet\bV^{1/2})}(\bX\bV^{-1/2})\bV^{1/2} \in \mathbb{R}^{d \times k}
\end{equation}
or as a Bregman proximal operator \cite{bauschke2003bregman} defined with the Bregman distance induced by the strictly convex, coercive function (for positive definite $\bV$)
\begin{equation}
\label{breg_dist}
    \bX \mapsto \frac{1}{2}\mathrm{tr}(\bX\bV^{-1}\bX^\top)
\end{equation}
which justifies the use of the Bregman resolvent 
\begin{equation}
\label{bregman_res}
   \underset{\bT \in \mathbb{R}^{d \times k}}{\arg\min} \left\{f(\bT)+\frac{1}{2}\mathrm{tr}\left((\bT-\bX)\bV^{-1}(\bT-\bX)^{\top}\right)\right\}=\left(\bm{I}+\partial f(\bullet)\bV\right)^{-1}(\bX)
\end{equation}
Many of the usual or similar properties to that of standard proximal operators (i.e.~firm non-expansiveness, link with Moreau/Bregman envelopes,\dots) hold for Bregman proximal operators defined with the function \eqref{breg_dist}, see e.g. \cite{bauschke2003bregman,bauschke2006joint}. In particular, we will be using the equivalent notion to firmly nonexpansive operators for Bregman proximity operators, called $\emph{D-firm}$ operators. Consider the Bregman proximal defined with a differentiable, strictly convex, coercive function $g : \mathcal{X} \to \mathbb{R}$, where $\mathcal{X}$ is a given input Hilbert space. Let $T$ be the associated Bregman proximal of a given convex function $f : \mathcal{X} \to \mathbb{R}$, i.e., for any $\mathbf{x} \in \mathcal{X}$
\begin{equation}
    T(\mathbf{x}) = \underset{\mathbf{y} \in \mathcal{X}}{\arg\min} \left\{f(\mathbf{x})+D_{g}(\mathbf{x},\mathbf{y})\right\}
\end{equation}
Then $T$ is \emph{D-firm}, meaning it verifies
\begin{equation}
\label{eq:D-firm}
    \langle T\bx-T\by, \nabla g(T\bx)-\nabla g(T\by) \rangle \leqslant  \langle T\bx-T\by, \nabla g(\bx)-\nabla g(\by) \rangle
\end{equation}
for any $\mathbf{x},\mathbf{y}$ in $\mathcal{X}$.

\paragraph{Approximate message-passing ---}
Approximate message-passing algorithms are a statistical physics inspired family of iterations which can be used to solve high dimensional inference problems \cite{zdeborova2016statistical}. One of the central objects in such algorithms are the so called \emph{state evolution equations}, a low-dimensional recursion equations which allow to exactly compute the high dimensional distribution of the iterates of the sequence.
In this proof we will use a specific form of matrix-valued approximate message-passing iteration with non-separable non-linearities. In its full generality, the validity of the state evolution equations in this case is an extension of the works of \cite{javanmard2013state} included in \cite{gerbelot2021graph}. Consider a sequence Gaussian matrices $\bA(n) \in \mathbb{R}^{n\times d}$ with i.i.d.~Gaussian entries, $A_{ij}(n)\sim\mathcal{N}(0,1/d)$. For each $n,d \in \mathbb{N}$, consider two sequences of pseudo-Lipschitz functions
\begin{equation}
\{\bh_{t} : \mathbb{R}^{n \times k}\to \mathbb{R}^{n \times k}\}_{t \in \mathbb{N}}\qquad \{\be_{t} : \mathbb{R}^{d \times k}\to \mathbb{R}^{d \times k}\}_{t \in \mathbb{N}}
\end{equation}
initialised on $\bu^{0} \in \mathbb{R}^{d\times k}$ in such a way that the limit
\begin{equation}
    \lim_{d \to \infty}\frac{1}{d}\norm{\be_{0}(\bu^{0})^\top\be_{0}(\bu^{0})}_{\rm F}
\end{equation}
exists, and recursively define:
\begin{align}
\label{canon_AMP}
	&\hspace{1cm} \bu^{t+1} = \bA^{\top}\bh_{t}(\bv^{t})-\be_{t}(\bu^{t})\langle \bh_{t}'\rangle^\top \\
	&\hspace{1cm} \bv^{t} = \bA\be_{t}(\bu^{t})-\bh_{t-1}(\bv^{t-1})\langle \be_{t}'\rangle^\top
\end{align}
where the dimension of the iterates are $\bu^{t} \in \mathbb{R}^{d \times k}$ and $\bv^{t} \in \mathbb{R}^{n \times k}$. The terms in brackets are defined as:
\begin{equation}
\label{eq:mat_ons}
 \langle \bh_{t}' \rangle = \frac{1}{d}\sum_{i=1}^{n} \frac{\partial \bh_{t}^{i}}{\partial \bv_{i}}(\bv^{t})\in \mathbb{R}^{k \times k} \quad \langle \be_{t}' \rangle = \frac{1}{d}\sum_{i=1}^{d} \frac{\partial \be_{t}^{i}}{\partial \bu_{i}}(\bu^{t}) \in \mathbb{R}^{k \times k}
\end{equation}
We define now the \emph{state evolution recursion} on two sequences of matrices $\{\bQ_{r,s}\}_{s,r\geqslant0}$ and $\{\hat{\bQ}_{r,s}\}_{s,r\geqslant1}$ initialised with $\bQ_{0,0} = \lim_{d \to \infty}\frac{1}{d}\be_{0}(\bu^{0})^\top \be_{0}(\bu^{0})$:
\begin{align}
   &\bQ_{t+1,s} = \bQ_{s,t+1} = \lim_{d \to \infty} \frac{1}{d}\mathbb{E}\left[\be_{s}(\hat{\bZ}^{s})^{\top}\be_{t+1}(\hat{\bZ}^{t+1})\right] \in \mathbb{R}^{k \times k} \\
    &\hat{\bQ}_{t+1,s+1} = \hat{\bQ}_{s+1,t+1} = \lim_{d \to \infty} \frac{1}{d}\mathbb{E}\left[\bh_{s}(\bZ^{s})^{\top}\bh_{t}(\bZ^{t})\right] \in \mathbb{R}^{k \times k}
\end{align}
where $(\bZ^{0},\dots,\bZ^{t-1}) \sim \mathcal{N}(\boldsymbol 0,\{\bQ_{r,s}\}_{0\leqslant r,s \leqslant t-1} \otimes \bI_{n}),(\hat{\bZ}^{1},\dots,\hat{\bZ}^{t}) \sim \mathcal{N}(\boldsymbol 0,\{\bhQ_{r,s}\}_{1\leqslant r,s \leqslant t} \otimes \bI_{d})$. Then the following holds
\begin{theorem}
\label{th:SE}
In the setting of the previous paragraph, for any sequence of pseudo-Lipschitz functions $\phi_{n}:(\mathbb{R}^{n \times K}\times \mathbb{R}^{d \times k})^{t} \to \mathbb{R}$, for $n,d \to +\infty$:
\begin{equation}
\phi_{n}(\bu^{0},\bv^{0},\bu^{1},\bv^{1},\dots,\bv^{t-1},\bu^{t}) \stackrel{\rm P}\simeq \mathbb{E}\left[\phi_{n}\left(\bu^{0},\bZ^{0},\hat{\bZ}^{1},\bZ^{1},\dots,\bZ^{t-1},\hat{\bZ}^{t}\right)\right]
\end{equation}
where $(\bZ^{0},\dots,\bZ^{t-1}) \sim \mathcal{N}(\boldsymbol 0,\{\bQ_{r,s}\}_{0\leqslant r,s \leqslant t-1} \otimes \bI_{n}),(\hat{\bZ}^{1},\dots,\hat{\bZ}^{t}) \sim \mathcal{N}(\boldsymbol 0,\{\bhQ_{r,s}\}_{1\leqslant r,s \leqslant t} \otimes \bI_{n})$.
\end{theorem}
\section{AMP implementation: channel and prior updates for \texorpdfstring{$k=3$}{k=3}}
\label{app:amp}

In this appendix we present the expression of the integrals numerically computed for the implementation of Algorithm \ref{alg:amp} in the present case. For a general and detailed derivation of the algorithm see \cite{zdeborova2016statistical, aubin2020}. 

Apart from the update functions $\bm{f}_w ( \gamm ,  \Lamb   )$ and $ \bm{f}_{\text{out}} (\y, \om , \V  )$ in Appendix \ref{app:denoising_fcts}, the approximate message passing algorithm also requires the variance updates. Using the mapping from Appendix \ref{appendix:mapping}, the updates are computed though $(\nc - 1) \times (\nc - 1)$ matrices constructed from the derivatives of the denoising functions:
\begin{subequations}
\begin{equation}
\partial_{\gamm}  \bm{f}_w ( \gamma , \Lamb  ) = \E_{Q_0} [ \w \w^\top   ] -
\bm{f}_w ( \gamm , \Lamb   ) \bm{f}_{w}^{\top} ( \gamm , \Lamb  )    \;,
\end{equation}
\begin{equation}
\label{eq:fout_dev}
\partial_{\om}  \bm{f}_{\text{out}} ( \bm{y}, \bm{\omega}, \bm{V}  )  = \bm{V}^{-1 } \E_{ Q_{\text{out} }} \left[ ( \bm{z} - \bm{\omega}   ) ( \bm{z} - \bm{\omega}   )^{\top}  \right] - \bm{V}^{-1} - \bm{f}_{\text{out}} ( \bm{\gamma}, \bm{\Lambda}  ) \bm{f}_{\text{out}}^{\top} ( \bm{\gamma}, \bm{\Lambda}  ) \;.
\end{equation}
\end{subequations}

\subsection{Channel updates}

Under the mapping from Appendix \ref{appendix:mapping}, for $k = 3$ we have two sets of integrals related to the channel: one when $\y = [0,0]^\top$ and other when $\y = [1,0]^\top$ or $\y = [0,1]^\top$. A simple flip on the variables distinguishes these two latter cases. 

We introduce the notation,
\begin{subequations}
\begin{equation}
 \V^{-1} 
 \equiv \begin{bmatrix}
    \Vi_{11} & \Vi_{12} \\
     \Vi_{21} & \Vi_{22} 
\end{bmatrix} \;,
\end{equation}
as well as the quantities
\begin{subequations}
\begin{equation}
    \Vib \equiv \frac{\Vi_{12} + \Vi_{21}}{2} \;,
\end{equation}
\begin{equation}
   \nu \equiv \frac{\Vi_{11} \Vi_{22} }{ \Vib^2} \;.
\end{equation}
\end{subequations}
Additionally, the standard Gaussian measure is denoted as 
\begin{equation}
    {\cal D} z \equiv \frac{d z }{\sqrt{2 \pi}} e^{ - \frac{1}{2} z^2 }\;,
\end{equation}
and the cumulative distribution function of the standard Gaussian distribution as
\begin{equation}
   \Phi(x) \equiv \int_{-\infty}^{x} {\cal D} z \;.
\end{equation}
\end{subequations}

\paragraph{Case $\y = [0,0]^\top$ ---}

The channel function $ \bm{f}_{\text{out}} (\y, \om , \V  ) $ is obtained through the numerical computation of the following quantities:
\begin{subequations}
\begin{equation}
       \Zout^{(00)} =  \sqrt{ \frac{\pi^2}{ \Vi_{22} } }  \frac{ \omega_1}{ \alpha_{12}} {\cal J}_{0}^{(00)} \;,
\end{equation}
\begin{equation}
    \frac{\partial}{\partial \omega_1 }  \Zout^{(00)} = -     \sqrt{ 2 \pi} \frac{\Vib}{\sqrt{\Vi_{11} \Vi_{22}^2}}  e^{-\frac{1}{2} \alpha_{21}^2}     \Phi (-  \gamma_{12}   )  +  \sqrt{ \frac{\pi^2}{ \Vi_{22} } }  {\cal J}_{1}^{(00)}  \;,
\end{equation}
\begin{equation}
     \frac{\partial}{\partial \omega_2 }  \Zout^{(00)} = -   \sqrt{\frac{2 \pi }{\Vi_{11}}}
     e^{-\frac{1}{2} \alpha_{21}^2}     \Phi (-  \gamma_{12}   ) \;,
\end{equation}
and its derivative $ \partial_{\bm \omega} \bm{f}_{\text{out}} (\y, \om , \V  ) $ by computing 
\begin{equation}
\begin{split}
       \frac{\partial^2}{\partial \omega_{1}^2 }  \Zout^{(00)} =&   \sqrt{ 2 \pi } \frac{\Vib^3}{\Vi_{11} \Vi_{22}^2 }  e^{-\frac{1}{2} \alpha_{21}^2} \left[ \frac{ e^{-\frac{1}{2} \gamma_{12}^2 }}{\sqrt{2\pi}} ( 2 \nu - 1 ) - \frac{\Vib \omega_2}{\sqrt{\Vi_{11}}} ( \nu - 1 )  \Phi (-  \gamma_{12}   ) \right]    \\
       & - \frac{\alpha_{12}^2}{\omega_{1}^2} \Zout^{(00)}  + \sqrt{\frac{\pi^2}{\Vi_{22}}}  \frac{\alpha_{12}}{\omega_1}  {\cal J}_{2}^{(00)} \;,
\end{split}
\end{equation}
\begin{equation}
    \frac{\partial^2}{\partial \omega_{2}^2 }  \Zout^{(00)} =   \sqrt{2 \pi} \frac{\Vib}{\Vi_{11}}  e^{-\frac{1}{2} \alpha_{21}^2} \left[   \frac{e^{-\frac{1}{2} \gamma_{12}^2 }}{\sqrt{2\pi}}    + \frac{\sqrt{\Vi_{11}}}{\omega_2 \Vib} \alpha_{21}^2 \Phi (-  \gamma_{12}   ) \right]     \;,
\end{equation}
\begin{equation}
    \frac{\partial^2}{\partial \omega_1 \omega_2  }  \Zout^{(00)} =   e^{-\frac{1}{2} \left( \alpha_{21}^2 + \gamma_{12}^2 \right)}   \;,
\end{equation}
\end{subequations}
with
\begin{equation}
\label{eq:amp_int_00}
    {\cal J}_{l}^{(00)} \equiv   \int_{-\infty}^{-\alpha_{12}} {\cal D} z \; z^l     \erfc \left( \frac{ - \frac{z \Vib \omega_1  }{ \alpha_{12}} + \Vi_{22} \omega_2 }{\sqrt{2 \Vi_{22}    } }  \right) \;,
\end{equation}
for $l=0,1,2$ and
\begin{subequations}
\begin{equation}
\label{eq:alpha12}
    \alpha_{12}  \equiv \omega_1 \sqrt{\Vi_{11} - \frac{\Vib^2}{\Vi_{22}}   }  \;.
\end{equation}
\begin{equation}
    \gamma_{12} \equiv \frac{\Vi_{11} \omega_1 + \Vib \omega_2}{\sqrt{\Vi_{11}}}  \;.
\end{equation}
\end{subequations}

\paragraph{Case $\y = [1,0]^\top$ ---}

The channel function $ \bm{f}_{\text{out}} (\y, \om , \V  ) $ is obtained through the numerical computation of the following quantities:
\begin{subequations}
\begin{equation}
  \Zout^{(10)} =  \sqrt{\frac{\pi^2}{\Vi_{22}}} \frac{\omega_1}{\alpha_{12}} {\cal J}_{0}^{(10)}\;,
\end{equation}
\begin{equation}
  \frac{\partial}{\partial \omega_1 }\Zout^{(10)} = -  \sqrt{ 2 \pi } \frac{\Vib}{\sqrt{ {\cal S}_\Vi \Vi_{22}^2 }}  
 e^{- \frac{1}{2} (   \Omega^2 - \beta^2 ) } \tilde{\Phi} (-\beta)  +  \sqrt{\frac{\pi^2}{\Vi_{22}}}  {\cal J}_{1}^{(10)}   \;,
\end{equation}
\begin{equation}
        \frac{\partial}{\partial \omega_2 } \Zout^{(10)} = -  \sqrt{\frac{2 \pi }{{\cal S}_\Vi} }  e^{- \frac{1}{2} (   \Omega^2 - \beta^2 ) } \tilde{\Phi} (-\beta) \;,
\end{equation}
and its derivative $ \partial_{\bm \omega} \bm{f}_{\text{out}} (\y, \om , \V  ) $ by computing 
\begin{equation}
\begin{split}
        \frac{\partial^2}{\partial \omega_{1}^2 } \Zout^{(10)} =&  -  \left( \frac{\alpha_{12}}{\omega_1} \right)^2  \Zout^{(10)} +  \sqrt{\frac{\pi^2}{\Vi_{22}}}    {\cal J}_{2}^{(10)}    \\
        & +  \sqrt{ 2 \pi}  \frac{\Vib^3}{{\cal S}_\Vi \Vi_{22}^{2} } \sigma_{12}
         e^{- \frac{1}{2} (   \Omega^2 - \beta^2 ) } \left[
        \frac{e^{- \frac{1}{2}\beta^2  }}{\sqrt{2\pi}}  + \left( \beta -  \sqrt{{\cal S}_\Vi}   
        \left( \left( 1 + \frac{\Vi_{22}}{ \Vib \sigma_{12} } \right) \omega_1       - \frac{\Vi_{22}}{ \Vib \sigma_{12} } \omega_2 \right)\right) \tilde{\Phi}(-\beta) \right] \;,
\end{split}
\end{equation}
\begin{equation}
         \frac{\partial^2}{\partial \omega_{2}^2 } \Zout^{(10)} = -   \sqrt{2 \pi} 
        \frac{  \Vib + \Vi_{22} }{{\cal S}_{\Vi}} 
         e^{- \frac{1}{2} (   \Omega^2 - \beta^2 ) }  \left [
         \frac{e^{- \frac{1}{2}\beta^2  }}{\sqrt{2\pi}} 
         +  \left( \beta -  \sqrt{{\cal S}_\Vi}    \left( \frac{\Vib \omega_1 + \Vi_{22} \omega_2}{\Vib + \Vi_{22}} \right)  \right)   \tilde{\Phi} (-\beta)   \right] \;,
\end{equation}
\begin{equation}
 \frac{\partial^2}{\partial \omega_1 \omega_2 } \Zout^{(10)} =
 -  \sqrt{2 \pi}  \frac{\Vib + \Vi_{11}}{ {\cal S}_\Vi }  e^{- \frac{1}{2} (   \Omega^2 - \beta^2 ) } \left[  \frac{e^{- \frac{1}{2}\beta^2  }}{\sqrt{2\pi}}  - \frac{\Vi_{11} \Vi_{22} - \Vib^2}{{\cal S}_\Vi ( \Vib + \Vi_{11} ) } \sqrt{{\cal S}_\Vi}   (\omega_1 -\omega_2) \tilde{\Phi} (-\beta) \right] \;,
\end{equation}
\end{subequations}
where
\begin{equation}
\label{eq:amp_int_10}
   {\cal J}_{l}^{(10)} \equiv   \int_{- \alpha_{12}}^{\infty} {\cal D}z  \; z^l  
 \erfc\left( \frac{ - \frac{z  \omega_1 ( \Vib + \Vi_{22}  )  }{ \alpha_{12}} + \Vi_{22} (\omega_2 - \omega_1 ) }{\sqrt{2 \Vi_{22}    } }  \right) \;,
\end{equation}
for $l=0,1,2$ with $\alpha_{12}$ given by Eq.~\eqref{eq:alpha12} and
\begin{subequations}
\begin{equation}
\tilde{\Phi}(x) \equiv 1 -  \Phi(x) \;,
\end{equation}
\begin{equation}
    \beta \equiv \frac{ \omega_1 (\Vi_{11}+\Vib) + \omega_2 (\Vi_{22}+\Vib)}{\sqrt{\Vi_{11}+\Vi_{22}+2\Vib}} \;, 
\end{equation}
\begin{equation}
 \Omega^2 \equiv \Vi_{11} \omega_1 + \Vi_{22} \omega_2 + 2 \Vib \omega_1 \omega_2   \;,
\end{equation}
\begin{equation}
   {\cal S}_\Vi   \equiv \Vi_{11} +  \Vi_{22} + 2 \Vib   \;,
\end{equation}
\begin{equation}
    \sigma_{12} \equiv \frac{\Vib^2 - 2 \Vib_{11}  \Vib_{22} -  \Vib_{22} \Vib }{\Vib^2} \;.
\end{equation}
\end{subequations}

If the label vector is $\y = [0,1]^\top$, one just needs to perform the following trivial changes in the equations above for $\y = [1,0]^\top$:
\begin{subequations}
\begin{equation}
    \Vi_{11} \to \Vi_{22}   \;,    
\end{equation}
\begin{equation}
    \Vi_{22} \to \Vi_{11}     \;,    
\end{equation}
\begin{equation}
    \omega_1  \to \omega_2  \;,    
\end{equation}
\begin{equation}
     \omega_2  \to \omega_1  \;,    
\end{equation}
\end{subequations}

Observe that the mapping from Appendix \ref{appendix:mapping} has allowed us to reduce the number of integrals to be numerically computed at each AMP iteration to three, given by Eq.~\eqref{eq:amp_int_00} or Eq.~\eqref{eq:amp_int_10}, depending on the one-hot output representation $\y$. These integrals were solved through the \texttt{integrate.quad} module from \texttt{SciPy} \cite{scipy}. To speed up the integration, we have also used \texttt{Numba} \cite{numba} decorators.

\subsection{Prior updates}
\paragraph{Gaussian prior ---}

Under the mapping of Appendix \ref{appendix:mapping}, the prior partition function is written as 
\begin{equation}
    {\cal Z}_{w} ( \bm{\gamma}, \bm{\Lambda} ) = \int_{\mathbb{R}^{k-1}} \frac{d \bm{w}}{\sqrt{(2\pi)^{k-1}\det(\bm{\Tilde \Sigma})}} 
    \exp\left[  - \frac{1}{2} \bm{w}^{\top} \left(  \bm{\Tilde \Sigma}{}^{-1} +\bm{\Lambda}  \right) \bm{w} + \bm{\gamma}^{\top} \bm{w}   \right]   \;,
\end{equation}
and can be analytically computed, 
\begin{subequations}
\begin{equation}
    {\cal Z}_{w} ( \bm{\gamma}, \bm{\Lambda} ) = \frac{1}{\sqrt{\det(\tilde{\bm{\Sigma}})
    \det (  \tilde{\bm{\Sigma}}{}^{-1} + \bm{\Lambda}  )}}\exp\left[\frac{1}{2}\bm{\gamma}^\top \left( \tilde{\bm{\Sigma}}{}^{-1}+\bm{\Lambda}  \right)^{-1}\bm{\gamma}\right] \;,
\end{equation}
as well as the denoising functions:
\begin{equation}
   \bm{f}_{w} ( \bm{\gamma}, \bm{\Lambda}  ) = \partial_{\bm{\gamma}} \log  {\cal Z}_{w} ( \bm{\gamma}, \bm{\Lambda} ) = \left( \tilde{\bm{\Sigma}}{}^{-1}+\bm{\Lambda} \right)^{-1} \bm{\gamma} \;,
\end{equation}
\begin{equation}
   \partial_{\bm \gamma}\bm{f}_{w} ( \bm{\gamma}, \bm{\Lambda}  ) = \left( \tilde{\bm{\Sigma}}{}^{-1} + \bm{\Lambda} \right)^{-1} \;.
\end{equation}
\end{subequations}
For $k=3$, the reduced covariance matrix is given by
 \begin{equation}
 \tilde{\bm{\Sigma}} = 
 \begin{bmatrix}
2 & 1 \\
1 & 2 
\end{bmatrix} \;.
 \end{equation}

\paragraph{Rademacher prior ---}
Considering $k = 3$, the reduced prior given by Eq.~\eqref{eq:rad_prior_k1} becomes
\begin{equation} 
\label{eq:rad_prior_reduced_}
\begin{split}
      P_{\tilde{w}} ( \tilde{w}_1,  \tilde{w}_2) &=  \frac{1}{2^3} 
  \left[  2 \delta( \tilde{w}_1 ) \delta( \tilde{w}_2 ) +  \delta( \tilde{w}_1 ) \delta( \tilde{w}_2  +2) +  \delta( \tilde{w}_1  +2) \delta( \tilde{w}_2  ) +\delta( \tilde{w}_1  ) \delta( \tilde{w}_2  -2) \right. \\
  &  \left.    +  \delta( \tilde{w}_1   -2) \delta( \tilde{w}_2  )+   \delta( \tilde{w}_1  +2 ) \delta( \tilde{w}_2  +2 )  + \delta( \tilde{w}_1  -2) \delta( \tilde{w}_2  -2)     \right] \;.
\end{split}
\end{equation}
The denoising functions for this case are computed numerically, via Monte Carlo sampling of the distribution given Eq.~\eqref{eq:rad_prior_reduced_}.

For more details, see the \texttt{amp} folder on \href{https://github.com/rodsveiga/mc_perceptron}{GitHub}.

\section{Details on the numerical simulations}
\label{app:numerics}
In this section we provide some more details on the numerical simulations implemented to test our theory for the learning curves of ERM (Figure~\ref{fig:gaussian_prior}). The solution of the convex optimization problem defined in Eq. \eqref{eq:optimization_ERM} can be computed by a standard gradient descent algorithm. We ran simulations using the squared loss and the cross-entropy loss. The simulations for the cross-entropy loss have been implemented using the \texttt{LogisticRegression} module of the \texttt{scikit-learn} package \cite{scikit}. The solution for the square loss is analytical. 
The results from numerical simulations that we show in the figures are averaged over $250$ instances of the problem at dimension $d=1000$.
\bibliographystyle{IEEEtran}
\bibliography{bibtex}

\begin{thebibliography}{10}
\providecommand{\url}[1]{#1}
\csname url@samestyle\endcsname
\providecommand{\newblock}{\relax}
\providecommand{\bibinfo}[2]{#2}
\providecommand{\BIBentrySTDinterwordspacing}{\spaceskip=0pt\relax}
\providecommand{\BIBentryALTinterwordstretchfactor}{4}
\providecommand{\BIBentryALTinterwordspacing}{\spaceskip=\fontdimen2\font plus
\BIBentryALTinterwordstretchfactor\fontdimen3\font minus
  \fontdimen4\font\relax}
\providecommand{\BIBforeignlanguage}[2]{{%
\expandafter\ifx\csname l@#1\endcsname\relax
\typeout{** WARNING: IEEEtran.bst: No hyphenation pattern has been}%
\typeout{** loaded for the language `#1'. Using the pattern for}%
\typeout{** the default language instead.}%
\else
\language=\csname l@#1\endcsname
\fi
#2}}
\providecommand{\BIBdecl}{\relax}
\BIBdecl

\bibitem{gardner1989three}
E.~Gardner and B.~Derrida, ``Three unfinished works on the optimal storage
  capacity of networks,'' \emph{Journal of Physics A: Mathematical and
  General}, vol.~22, no.~12, p. 1983, 1989.

\bibitem{seung1992statistical}
H.~S. Seung, H.~Sompolinsky, and N.~Tishby, ``Statistical mechanics of learning
  from examples,'' \emph{Physical review A}, vol.~45, no.~8, p. 6056, 1992.

\bibitem{watkin1993statistical}
T.~L. Watkin, A.~Rau, and M.~Biehl, ``The statistical mechanics of learning a
  rule,'' \emph{Reviews of Modern Physics}, vol.~65, no.~2, p. 499, 1993.

\bibitem{engel2001statistical}
A.~Engel and C.~Van~den Broeck, \emph{Statistical mechanics of learning}.\hskip
  1em plus 0.5em minus 0.4em\relax Cambridge University Press, 2001.

\bibitem{gyorgyi1990first}
G.~Gy{\"o}rgyi, ``First-order transition to perfect generalization in a neural
  network with binary synapses,'' \emph{Physical Review A}, vol.~41, no.~12, p.
  7097, 1990.

\bibitem{sompolinsky1990learning}
H.~Sompolinsky, N.~Tishby, and H.~S. Seung, ``Learning from examples in large
  neural networks,'' \emph{Physical Review Letters}, vol.~65, no.~13, p. 1683,
  1990.

\bibitem{barbier2019optimal}
J.~Barbier, F.~Krzakala, N.~Macris, L.~Miolane, and L.~Zdeborov{\'a}, ``Optimal
  errors and phase transitions in high-dimensional generalized linear models,''
  \emph{Proceedings of the National Academy of Sciences}, vol. 116, no.~12, pp.
  5451--5460, 2019.

\bibitem{aubin2020generalization}
B.~Aubin, F.~Krzakala, Y.~Lu, and L.~Zdeborov\'{a}, ``Generalization error in
  high-dimensional perceptrons: Approaching bayes error with convex
  optimization,'' in \emph{Advances in Neural Information Processing Systems},
  H.~Larochelle, M.~Ranzato, R.~Hadsell, M.~F. Balcan, and H.~Lin, Eds.,
  vol.~33.\hskip 1em plus 0.5em minus 0.4em\relax Curran Associates, Inc.,
  2020, pp. 12\,199--12\,210.

\bibitem{loureiro2021learning}
B.~Loureiro, G.~Sicuro, C.~Gerbelot, A.~Pacco, F.~Krzakala, and
  L.~Zdeborov{\'a}, ``Learning gaussian mixtures with generalised linear
  models: Precise asymptotics in high-dimensions,'' \emph{arXiv preprint
  arXiv:2106.03791}, 2021.

\bibitem{wang2021benign}
K.~Wang, V.~Muthukumar, and C.~Thrampoulidis, ``Benign overfitting in
  multiclass classification: All roads lead to interpolation,'' \emph{arXiv
  preprint arXiv:2106.10865}, 2021.

\bibitem{9414099}
G.~R. Kini and C.~Thrampoulidis, ``Phase transitions for one-vs-one and
  one-vs-all linear separability in multiclass gaussian mixtures,'' in
  \emph{ICASSP 2021 - 2021 IEEE International Conference on Acoustics, Speech
  and Signal Processing (ICASSP)}, 2021, pp. 4020--4024.

\bibitem{thrampoulidis2020theoretical}
C.~Thrampoulidis, ``Theoretical insights into multiclass classification: A
  high-dimensional asymptotic view,'' \emph{Neural Information Processing
  Systems (NeuRIPS 2020)}, 2020.

\bibitem{aubin2019committee}
B.~Aubin, A.~Maillard, J.~Barbier, F.~Krzakala, N.~Macris, and
  L.~Zdeborov{\'a}, ``The committee machine: Computational to statistical gaps
  in learning a two-layers neural network,'' \emph{Journal of Statistical
  Mechanics: Theory and Experiment}, vol. 2019, no.~12, p. 124023, 2019.

\bibitem{barbier2021overlap}
J.~Barbier, ``Overlap matrix concentration in optimal bayesian inference,''
  \emph{Information and Inference: A Journal of the IMA}, vol.~10, no.~2, pp.
  597--623, 2021.

\bibitem{javanmard2013state}
A.~Javanmard and A.~Montanari, ``State evolution for general approximate
  message passing algorithms, with applications to spatial coupling,''
  \emph{Information and Inference: A Journal of the IMA}, vol.~2, no.~2, pp.
  115--144, 2013.

\bibitem{bayati2011lasso}
M.~Bayati and A.~Montanari, ``{The LASSO risk for Gaussian matrices},''
  \emph{IEEE Transactions on Information Theory}, vol.~58, no.~4, pp.
  1997--2017, 2011.

\bibitem{donoho2009message}
D.~L. Donoho, A.~Maleki, and A.~Montanari, ``Message-passing algorithms for
  compressed sensing,'' \emph{Proceedings of the National Academy of Sciences},
  vol. 106, no.~45, pp. 18\,914--18\,919, 2009.

\bibitem{zdeborova2016statistical}
L.~Zdeborová and F.~Krzakala, ``Statistical physics of inference: thresholds
  and algorithms,'' \emph{Advances in Physics}, vol.~65, no.~5, pp. 453--552,
  2016.

\bibitem{gerbelot2021graph}
C.~Gerbelot and R.~Berthier, ``Graph-based approximate message passing
  iterations,'' \emph{arXiv preprint arXiv:2109.11905}, 2021.

\bibitem{Nishimori2001StatisticalPO}
H.~Nishimori and O.~U. Press, \emph{Statistical Physics of Spin Glasses and
  Information Processing: An Introduction}, ser. International series of
  monographs on physics.\hskip 1em plus 0.5em minus 0.4em\relax Oxford
  University Press, 2001.

\bibitem{scikit}
F.~Pedregosa, G.~Varoquaux, A.~Gramfort, V.~Michel, B.~Thirion, O.~Grisel,
  M.~Blondel, P.~Prettenhofer, R.~Weiss, V.~Dubourg, J.~Vanderplas, A.~Passos,
  D.~Cournapeau, M.~Brucher, M.~Perrot, and E.~Duchesnay, ``Scikit-learn:
  Machine learning in {P}ython,'' \emph{Journal of Machine Learning Research},
  vol.~12, pp. 2825--2830, 2011.

\bibitem{celentano2020estimation}
M.~Celentano, A.~Montanari, and Y.~Wu, ``The estimation error of general first
  order methods,'' in \emph{Conference on Learning Theory}.\hskip 1em plus
  0.5em minus 0.4em\relax PMLR, 2020, pp. 1078--1141.

\bibitem{rangan2011generalized}
S.~Rangan, ``Generalized approximate message passing for estimation with random
  linear mixing,'' in \emph{2011 IEEE International Symposium on Information
  Theory Proceedings}.\hskip 1em plus 0.5em minus 0.4em\relax IEEE, 2011, pp.
  2168--2172.

\bibitem{aubin2020}
B.~Aubin, ``Mean-field methods and algorithmic perspectives for
  high-dimensional machine learning,'' Ph.D. dissertation, Université
  Paris-Saclay, 2020, thèse de doctorat dirigée par Zdeborová, Lenka
  Physique université Paris-Saclay 2020.

\bibitem{vershynin2018high}
R.~Vershynin, \emph{High-dimensional probability: An introduction with
  applications in data science}.\hskip 1em plus 0.5em minus 0.4em\relax
  Cambridge university press, 2018, vol.~47.

\bibitem{bauschke2003bregman}
H.~H. Bauschke, J.~M. Borwein, and P.~L. Combettes, ``Bregman monotone
  optimization algorithms,'' \emph{SIAM Journal on control and optimization},
  vol.~42, no.~2, pp. 596--636, 2003.

\bibitem{bauschke2006joint}
H.~Bauschke, P.~Combettes, and D.~Noll, ``Joint minimization with alternating
  bregman proximity operators,'' \emph{Pacific Journal of Optimization}, 2006.

\bibitem{bayati2011dynamics}
M.~Bayati and A.~Montanari, ``The dynamics of message passing on dense graphs,
  with applications to compressed sensing,'' \emph{IEEE Transactions on
  Information Theory}, vol.~57, no.~2, p. 764–785, Feb 2011.

\bibitem{parikh2014proximal}
N.~Parikh and S.~Boyd, ``Proximal algorithms,'' \emph{Foundations and Trends in
  optimization}, vol.~1, no.~3, pp. 127--239, 2014.

\bibitem{bauschke2011convex}
H.~H. Bauschke, P.~L. Combettes \emph{et~al.}, \emph{Convex analysis and
  monotone operator theory in Hilbert spaces}.\hskip 1em plus 0.5em minus
  0.4em\relax Springer, 2011, vol. 408.

\bibitem{thrampoulidis2018precise}
C.~Thrampoulidis, E.~Abbasi, and B.~Hassibi, ``Precise error analysis of
  regularized $ m $-estimators in high dimensions,'' \emph{IEEE Transactions on
  Information Theory}, vol.~64, no.~8, pp. 5592--5628, 2018.

\bibitem{scipy}
P.~Virtanen, R.~Gommers, T.~E. Oliphant, M.~Haberland, T.~Reddy, D.~Cournapeau,
  E.~Burovski, P.~Peterson, W.~Weckesser, J.~Bright, S.~J. {van der Walt},
  M.~Brett, J.~Wilson, K.~J. Millman, N.~Mayorov, A.~R.~J. Nelson, E.~Jones,
  R.~Kern, E.~Larson, C.~J. Carey, {\.I}.~Polat, Y.~Feng, E.~W. Moore,
  J.~{VanderPlas}, D.~Laxalde, J.~Perktold, R.~Cimrman, I.~Henriksen, E.~A.
  Quintero, C.~R. Harris, A.~M. Archibald, A.~H. Ribeiro, F.~Pedregosa, P.~{van
  Mulbregt}, and {SciPy 1.0 Contributors}, ``{{SciPy} 1.0: Fundamental
  Algorithms for Scientific Computing in Python},'' \emph{Nature Methods},
  vol.~17, pp. 261--272, 2020.

\bibitem{numba}
S.~K. Lam, A.~Pitrou, and S.~Seibert, ``Numba: A llvm-based python jit
  compiler,'' in \emph{Proceedings of the Second Workshop on the LLVM Compiler
  Infrastructure in HPC}, ser. LLVM '15.\hskip 1em plus 0.5em minus 0.4em\relax
  New York, NY, USA: Association for Computing Machinery, 2015.

\end{thebibliography}
\end{document}